\documentclass[a4paper,UKenglish,numberwithinsect,autoref,final]{lipics-v2021}

\nolinenumbers
\usepackage[utf8]{inputenc}
\usepackage[T1]{fontenc}
\usepackage{stmaryrd}
\usepackage{amssymb}
\usepackage{amsthm}
\usepackage{amsmath}
\usepackage{thmtools}
\usepackage{mathtools}
\usepackage{tikz-cd}
\tikzset{shiftarr/.style={
        rounded corners,%
        to path={--([#1]\tikztostart.center)
                     -- ([#1]\tikztotarget.center) \tikztonodes
                     -- (\tikztotarget)},
}}
\usepackage[nomargin, marginclue, inline]{fixme}
\fxuselayouts{inline,nomargin}
\fxusetheme{color}
\FXRegisterAuthor{fb}{afb}{FB}
\FXRegisterAuthor{sm}{asm}{SM}
\FXRegisterAuthor{hu}{ahu}{HU}
\usepackage[notref,notcite, color]{showkeys}
\usepackage{xspace}
\usepackage{awesomebox}
\usepackage{cite}
\usepackage{microtype}
\def\o{\cdot}
\usepackage{anyfontsize}

\declaretheorem[name=Definition,style=definition,numberwithin=section, sibling=definition]{defn}
\declaretheorem[name=Example,style=definition,sibling=defn]{expl}

\declaretheorem[name=Remark,style=definition,sibling=defn]{rem}

\declaretheorem[name=Construction,style=definition,sibling=defn]{construction}

\addto\extrasbritish{
}

\newcounter{mycounter}

\newcommand{\iref}[2]{\autoref{#1}.\ref{#1:#2}}

\newcommand{\msr}{MSR\xspace}

\newcommand{\sbnd}[1][s]{\ensuremath{#1}-bounded\xspace}
\newcommand{\kbnd}[1][k]{\( #1 \)-bounded\xspace}
\newcommand{\Top}{\mathbf{Top}}

\newcommand{\pow}{\ensuremath{\mathcal{P}}\xspace}

\newcommand{\powfs}{\ensuremath{\mathcal{P}_{\text{fs}}}\xspace}
\newcommand{\N}{\ensuremath{\mathbb{N}}\xspace}
\newcommand{\R}{\ensuremath{\mathbb{R}}\xspace}
\newcommand{\A}{\ensuremath{\mathbb{A}}\xspace}
\newcommand{\V}{\ensuremath{\mathcal{V}}\xspace}
\newcommand{\T}{\ensuremath{\mathcal{T}}\xspace}

\newcommand{\pfm}[2][\Sigma]{\ensuremath{\myhat{0.95}{0pt}{#1^{*}}_{#2}}\xspace}
\newcommand{\spfm}[1][\Sigma]{\ensuremath{\myhat{0.95}{0pt}{{#1}^*_{s}}}\xspace}

\newcommand{\epi}{\twoheadrightarrow}
\newcommand{\mono}{\rightarrowtail}
\newcommand{\id}{\mathrm{id}}
\newcommand{\darr}{\downarrow}
\newcommand{\depi}{\mathbin{\rotatebox[origin=c]{-90}{$\epi$}}}
\newcommand{\dhook}{\mathbin{\rotatebox[origin=c]{-90}{$\hookrightarrow$}}}

\DeclareMathOperator{\cod}{cod}

\DeclareMathOperator{\pro}{Pro}
\DeclareMathOperator{\ind}{Ind}

\newcommand{\cat}[1]{\ensuremath{\mathbf{#1}}\xspace}
\newcommand{\xcat}[2]{\ensuremath{\mathbf{#1}_{#2}}\xspace}
\newcommand{\set}{\cat{Set}}
\newcommand{\setf}{\xcat{Set}{\mathrm f}}

\newcommand{\nom}{\cat{Nom}}
\newcommand{\nomk}{\xcat{Nom}{k}}
\newcommand{\nomof}{\xcat{Nom}{\mathrm{of}}}
\newcommand{\nomofk}{\xcat{Nom}{\mathrm{of}, k}}
\newcommand{\pronomof}{\ensuremath{\pro(\nomof})\xspace}

\newcommand{\To}{\Longrightarrow}

\newcommand{\ncba}{\cat{nCBA}}
\newcommand{\ncaba}{\cat{nCABA}}
\newcommand{\ncofba}{\cat{nC_{\mathrm{of}}BA}}
\newcommand{\ncakba}[1][k]{\cat{nCA_{\mathit{#1}}BA}}
\newcommand{\ncofalkba}[1][k]{\cat{nC_{\mathrm{of}}A_{\mathrm{l}\mathit{#1}}BA}}

\newcommand{\mon}{\cat{Mon}}
\newcommand{\monf}{\xcat{Mon}{\mathrm f}\xspace}
\newcommand{\nmon}{\cat{nMon}}
\newcommand{\nmonof}{\xcat{nMon}{\mathrm{of}}}
\newcommand{\nmonofk}{\xcat{nMon}{\mathrm{of}, k}}

\newcommand{\npro}[0]{\cat{nStone}}

\newcommand{\stone}{\cat{Stone}}

\newcommand{\topo}[0]{\cat{Top}}
\newcommand{\ntop}[0]{\cat{nTop}}

\newcommand{\epito}{\twoheadrightarrow}
\newcommand{\hookto}{\hookrightarrow}

\newcommand{\slice}[2]{\ensuremath{{#1}\mathord{\darr}{#2}}\xspace}
\newcommand{\sslice}[3][s]{\ensuremath{{#2}\mathord{\darr_{#1}} {#3}}\xspace}
\newcommand{\epislice}[2]{\ensuremath{{#1}\mathord{\depi}{#2}}\xspace}
\newcommand{\sepislice}[3][s]{\ensuremath{{#2}\mathord{\depi_{#1}} {#3}}\xspace}
\newcommand{\can}{\ensuremath{\fm\mathord{\depi_{s}}{\nmonof}}\xspace}
\DeclareMathOperator{\pr}{pr}
\DeclareMathOperator*{\colim}{colim}

\newcommand{\fs}{finitely supported\xspace}
\newcommand{\ufs}{uniformly finitely supported\xspace}
\newcommand{\pof}{pro-orbit-finite\xspace}

\DeclareMathOperator{\supp}{supp}
\DeclareMathOperator{\perm}{Perm}
\DeclareMathOperator{\orb}{orb}
\DeclareMathOperator{\hull}{hull} 
\DeclareMathOperator{\at}{At}

\DeclareMathOperator{\fp}{\mathcal{F}_{np}}

\DeclareMathOperator{\clo}{Clo}
\DeclareMathOperator{\rec}{Rec}

\newcommand{\forget}[2][-]{\ensuremath{|{#1}|_{#2}}\xspace}
\newcommand{\tr}[2]{({#1}\,\,{#2})}
\newcommand{\fm}[1][\Sigma]{\ensuremath{{#1}^{*}}\xspace}
\newcommand{\sa}[1][k]{\leq_{\mathrm{of}, {#1}}}

\newcommand{\op}{{\mathrm{op}}}
\newcommand{\seq}{\subseteq}

\newcommand{\new}{\reflectbox{$\mathsf{N}$}}

\newcommand{\mybar}[3]{%
  \mathrlap{\hspace{#2}\overline{\scalebox{#1}[1]{\phantom{\ensuremath{#3}}}}}\ensuremath{#3}
}

\newcommand{\myhat}[3]{%
  \mathrlap{\hspace{#2}\widehat{\scalebox{#1}[1]{\phantom{\ensuremath{#3}}}}}\ensuremath{#3}
}

\newcommand{\hatSigmas}{\myhat{1.0}{0pt}{\Sigma^*}}
\newcommand{\barF}{\mybar{0.6}{2.5pt}{F}}

\title{Nominal Topology for Data Languages}
\titlerunning{Nominal Topology for Data Languages}

\author{Fabian Birkmann}{Friedrich-Alexander-Universität Erlangen-Nürnberg, Germany}{henning.urbat@fau.de}{https://orcid.org/0000-0001-5890-9485}{}
\authorrunning{F.~Birkmann}
\author{Stefan Milius}{Friedrich-Alexander-Universität Erlangen-Nürnberg, Germany}{henning.urbat@fau.de}{https://orcid.org/0000-0002-2021-1644}{}
\authorrunning{S.~Milius}
\author{Henning Urbat}{Friedrich-Alexander-Universität Erlangen-Nürnberg, Germany}{henning.urbat@fau.de}{https://orcid.org/0000-0002-3265-7168}{}
\authorrunning{H.~Urbat}

\authorrunning{F.\ Birkmann, S.\ Milius and H.\ Urbat}

\Copyright{Fabian Birkmann, Stefan Milius and Henning Urbat}

\relatedversion{A full version is available at \url{https://arxiv.org/abs/2304.13337}}
\funding{Deutsche Forschungsgemeinschaft (DFG, German Research Foundation) – project number 470467389}
\pdfoutput=1
\numberwithin{equation}{section}

\begin{document}

\maketitle

\begin{abstract}
  We propose a novel topological perspective on data languages recognizable by orbit-finite nominal monoids.
  For this purpose, we introduce pro-orbit-finite nominal topological spaces.
  Assuming globally bounded support sizes, they coincide with nominal Stone spaces and are shown to be dually equivalent to a subcategory of nominal boolean algebras. Recognizable data languages are characterized as topologically clopen sets of pro-orbit-finite words. In addition, we explore the expressive power of pro-orbit-finite equations by establishing a nominal version of Reiterman's pseudovariety theorem.
\end{abstract}

\section[introduction]{Introduction}
\label{sec:introduction}

While automata theory is largely concerned with formal languages over finite alphabets, the extension to \emph{infinite alphabets} has been identified as a natural approach to modelling structures involving \emph{data}, such as nonces~\cite{KurtzEA07}, channel names~\cite{Hennessy02}, object identities~\cite{GrigoreEA13}, process
identifiers~\cite{BolligEA14}, URLs~\cite{BieleckiEA02}, or
values in XML documents~\cite{NevenEA04}. For example, if $\A$ is a (countably infinite) set of data values, typical languages to consider might be
\begin{align*}
 L_{0} &= \{\,vaaw \mid a\in \A,\, v,w\in \A^*\,\} &&  \text{(``some data value occurs twice in a row'')}, \text{ or} \\
 L_{1} &= \{\,avaw \mid a\in \A,\, v,w\in \A^*\,\} && \text{(``the first data value occurs again'')}.
\end{align*}
Automata for data languages enrich finite automata with register mechanisms that allow to store data and test data values for equality (or more complex relations, e.g.\ order)~\cite{KaminskiFrancez94, NevenEA04}. In a modern perspective first advocated by Boja\'nczyk, Klin, and Lasota~\cite{BojanczykEA14}, a convenient abstract framework for studying data languages is provided by the theory of \emph{nominal sets}~\cite{pit-13}.

Despite extensive research in the past three decades, no universally
acknowledged notion of \emph{regular} data language has emerged so
far. One reason is that automata models with data notoriously lack
robustness, in that any alteration of their modus operandi (e.g.\
deterministic vs.\ nondeterministic, one-way vs.\ two-way) usually
affects their expressive power. Moreover, machine-independent
descriptions of classes of data languages in terms of algebra or model
theory are hard to come by. However, there is one remarkable
class of data languages that closely mirrors classical regular
languages: data languages recognizable by orbit-finite nominal
monoids~\cite{boj-13}. Originally introduced from a purely algebraic angle, recognizable data languages have subsequently been
characterized in terms of \emph{rigidly guarded $\text{MSO}^{\sim}$},
a fragment of monadic second-order logic with equality
tests~\cite{col-15}, \emph{single-use register automata}~\cite{boj-20}
(both one-way and two-way), and \emph{orbit-finite regular list
  functions}~\cite{boj-20}. In addition, several landmark results from
the algebraic theory of regular languages, namely the
McNaughton-Papert-Schützenberger theorem~\cite{sch-65,mp71}, the
Krohn-Rhodes theorem~\cite{kr65}, and Eilenberg's variety
theorem~\cite{eil-74} have been extended to recognizable data
languages~\cite{boj-13,col-15,mil-urb-19-06,boj-20}.

In the present paper, we investigate recognizable data languages
through the lens of \emph{topology}, thereby providing a further
bridge to classical regular languages. The topological approach to the
latter is closely tied to the algebraic one, which regards
regular languages as the languages recognizable by finite monoids. Its
starting point is the construction of the topological space
$\hatSigmas$ of \emph{profinite words}. Informally, this space
casts all information represented by regular languages over~$\Sigma$
and their recognizing monoids into a single mathematical
object. Regular languages can then be characterized by purely
topological means: they may be interpreted as precisely the clopen
subsets of $\hatSigmas$, in such way that algebraic recognition
by finite monoids becomes a continuous process. Properties of
regular languages are often most conveniently classified in terms of the topological concept of
\emph{profinite equations}, that is, equations between profinite
words; see~\cite{alm-05,pin-09,ac21} for a survey of profinite methods
in automata theory. Moreover, since~$\hatSigmas$ forms a Stone
space, the power of \emph{Stone duality} -- the dual equivalence
between Stone spaces and boolean algebras -- becomes available. This
allows for the use of duality-theoretic methods for the study of
regular languages and their connection to logic and model theory, which in part even extend to \emph{non-regular}
languages~\cite{pip-97,ggp08,ggp10,gpr16,gpr17}.

On a conceptual level, the topological view of regular languages rests on a single category-theoretic fact: Stone spaces admit a universal property. In fact, they arise from the category of finite sets as the free completion under codirected limits, a.k.a.\ its \emph{Pro-completion}:
\begin{equation}\label{eq:stone-pro} \stone \simeq \pro(\setf).\end{equation}
In the world of data languages, the role of finite sets is taken over by orbit-finite nominal sets. This strongly suggests to base a topological approach on their free completion $\pro(\nomof)$. However, this turns out to be infeasible: the category $\pro(\nomof)$ is not concrete over nominal sets (\autoref{prop:not-conc}), hence it cannot be described via any kind of nominal topological spaces. This is ultimately unsurprising given that the description \eqref{eq:stone-pro} of Stone spaces as a free completion depends on the axiom of choice, which is well-known to fail in the topos of nominal sets. As a remedy, we impose \emph{global bounds} on the support sizes of nominal sets, that is, we consider the categories $\nomk$ and $\nomofk$ of (orbit-finite) nominal sets where every element has a support of size $k$, for some fixed natural number $k$. This restriction is natural from an automata-theoretic perspective, as it corresponds to imposing a bound $k$ on the number of registers of automata, and it fixes exactly the issue making unrestricted nominal sets non-amenable (\autoref{lem:cod-lim-set}). Let us emphasize, however, that the category $\nomk$ is not proposed as a new foundation for names and variable binding; for instance, it generally fails to be a topos.

The first main contribution of our paper is a generalization of \eqref{eq:stone-pro} to \kbnd nominal sets. For this purpose we introduce \emph{nominal Stone spaces}, a suitable nominalization of the classical concept, and prove that $k$-bounded nominal Stone spaces form the Pro-completion of the category of $k$-bounded orbit-finite sets. We also derive a nominal version of Stone duality, which relates $k$-bounded nominal Stone spaces to \emph{locally \( k \)-atomic orbit-finitely complete nominal boolean algebras}. Hence we establish the following equivalences of categories:
\[\ncofalkba\simeq^\op \npro_k \simeq \pro(\nomofk).\]
The above equivalences are somewhat remarkable since even the category of $k$-bounded nominal sets does not feature choice. They hold because the presence of bounds allows us to reduce topological properties of nominal Stone spaces, most notably compactness, to their classical counterparts.

Building on the above topological foundations, which we regard to be of independent interest, we subsequently develop first
steps of a topological theory of data languages. Specifically, we
introduce nominal Stone spaces of (bounded) \emph{pro-orbit-finite
  words} and prove their clopen subsets to correspond to data
languages recognizable by bounded equivariant monoid
morphisms, generalizing the topological characterization of classical
regular languages (\autoref{thm:rec-rep-equiv}). Moreover, we investigate
the expressivity of \emph{pro-orbit-finite equations} and show that
they model precisely classes of orbit-finite monoids closed under
finite products, submonoids, and \emph{multiplicatively
  support-reflecting quotients}
(\autoref{thm:exp-nominal-reiterman}). This provides a nominal version
of Reiterman's celebrated pseudovariety theorem~\cite{rei-82} for
finite monoids.

\vspace*{-.75\baselineskip}
\subparagraph*{Related work.}
The perspective taken in our paper draws much of its inspiration from the recent categorical approach to algebraic recognition based on monads~\cite{boj-15, urb-ada-che-mil-17-proc, sal-17}. The importance of Pro-completions in algebraic language theory has been isolated in the work of Chen et al.~\cite{che-ada-mil-urb-16} and Urbat et al.~\cite{urb-ada-che-mil-17-proc}. In the latter work the authors introduce \emph{profinite monads} and present a general version of Eilenberg's variety theorem parametric in a given Stone-type duality. The theory developed there applies to algebraic base categories, but not to the category of nominal sets. 

Our version of nominal Stone duality builds on the orbit-finite restriction of the duality between nominal sets and complete atomic nominal boolean algebras due to Petri\c{s}an~\cite{gab-lit-pet-11}. It is fundamentally different from the nominal Stone duality proposed by Gabbay, Litak, and Petri\c{s}an~\cite{gab-08}, which relates \emph{nominal Stone spaces with $\new$} to \emph{nominal boolean algebras with $\new$}. The latter duality is not amenable for the theory of data languages; see \autoref{rem:nom-stone-prof}.

Reiterman's pseudovariety theorem has recently been generalized to the
level of finite algebras for a
monad~\cite{che-ada-mil-urb-16,ada-che-mil-urb-21} and, in a more
abstract disguise, finite objects in a category~\cite{mil-urb-19}. For
nominal sets, varieties of algebras over binding signatures have been
studied by Gabbay~\cite{gab-08} and by Kurz and
Petri\c{s}an~\cite{kur-pet-10}, resulting in nominal Birkhoff-type
theorems~\cite{bir-35}. Urbat and Milius~\cite{mil-urb-19-06}
characterize classes of orbit-finite monoids called \emph{weak}
pseudovarieties by sequences of nominal word equations. This gives a
nominal generalization of the classical Eilenberg-Schützenberger
theorem~\cite{es76}, which in fact is a special case of the general
HSP theorem in~\cite{mil-urb-19}. Nominal pro-orbit-finite equations as
introduced in the present paper are strictly more expressive than
sequences of nominal word equations (\autoref{ex:compare}), hence our
nominal Reiterman theorem is not equivalent to the nominal
Eilenberg-Schützenberger theorem. Moreover, we note that the nominal
Reiterman theorem does not appear to be an instance of any of the
abstract categorical frameworks mentioned above.

\section[Preliminaries]{Preliminaries}
\label{sec:preliminaries}

We assume that readers are familiar with basic notions from category theory, e.g.~functors, natural transformations, and (co)limits, and from point-set topology, e.g.~metric and topological spaces, continuous maps, and compactness. In the following we recall some facts about Pro-completions, the key categorical concept underlying our topological approach to data languages. Moreover, we give a brief introduction to the theory of nominal sets~\cite{pit-13}.

\medskip\noindent\textbf{Pro-completions.}  A small category $I$ is
\emph{cofiltered} if (i)~$I$ is non-empty, (ii)~for every pair of
objects $i,j\in I$ there exists a span $i\leftarrow k\to j$, and (iii)
for every pair of parallel arrows $f,g\colon j\to k$, there exists a
morphism $h\colon i\to j$ such that $f\cdot h= g\cdot h$.  Cofiltered
preorders are called \emph{codirected}; thus a preorder $I$ is codirected
if $I\neq\emptyset$ and every pair \( i, j \in I \) has a lower bound
\( k \le i, j \). For instance, every meet-semilattice with bottom is
codirected.  A diagram \( D\colon I \rightarrow \cat C \) in a
category $\cat C$ is \emph{cofiltered} if its index category $I$ is
{cofiltered}. A \emph{cofiltered limit} is a limit of a cofiltered
diagram. \emph{Codirected limits} are defined analogously. The two
concepts are closely related: a category has cofiltered limits iff it
has codirected limits, and a functor preserves cofiltered limits iff
it preserves codirected
limits~\cite[Cor.~1.5]{adamek_rosicky_1994}. The dual concept is that
of a \emph{filtered colimit} or a \emph{directed colimit},
respectively.

\begin{example}\label{ex:cod-lim}
\begin{enumerate}
\item In the category \set of sets and functions, every filtered
  diagram $D\colon I\to \set$ has a colimit cocone
  $c_i\colon D_i\to \colim D$ ($i\in I$) given by
  $\colim D = \big(\coprod_{i\in I} D_i\big) / {\sim}$ and
  $c_i(x) = [x]_\sim$, where the equivalence relation $\sim$ on the
  coproduct (i.e.\ disjoint union) $\coprod_{i\in I} D_i$ relates $x\in D_i$ and $y\in D_j$
  iff there exist morphisms $f\colon i\to k$ and $g\colon j\to k$ in
  $I$ such that $Df(x)=Dg(y)$.

\item Every cofiltered diagram \( D \colon I \rightarrow \set \) has a
  limit whose cone $p_i\colon \lim D\to D_i$ ($i\in I$) is given by the \emph{compatible families} of
  $D$ and projection maps:
  \[
    \lim D = \{(x_{i})_{i \in I} \mid x_i\in D_i \text{ and }
    Df(x_{i}) = x_{j} \text{ for all $f\colon i\to j$ in $I$} \}
    \quad\text{and}\quad
    p_j((x_{i})_{i \in I})=x_{j}.
  \]

\item In the category $\Top$ of topological spaces and continuous maps, the limit cone of a cofiltered diagram \( D \colon I \rightarrow \Top \) is formed by taking the limit  in $\set$ and equipping $\lim D$ with the \emph{initial topology}, viz.~the topology generated by the basic open sets \( p_{i}^{-1}[U_{i}] \) for $i\in I$ and \( U_{i} \subseteq D_{i} \) open.
\end{enumerate}
\end{example}

An object \( C \) of a category $\cat C$ is \emph{finitely copresentable} if the contravariant hom-functor $\cat C(-,C)\colon \cat C^\op\to \set$ preserves directed colimits. In more elementary terms, this means that for every codirected diagram $D\colon I\to \cat C$ with limit cone \( p_{i} \colon L \rightarrow D_{i}\) ($i\in I$), 
\begin{enumerate}
\item every morphism \( f \colon L \rightarrow C \) factorizes as $f=g\circ p_i$ for some $i\in I$ and $g\colon D_i\to C$, and
\item the factorization is {essentially unique}: given another factorization $f=h\cdot p_i$, there exists \( j \le i \) such that \( g \cdot D_{j,i} = h \cdot D_{j,i} \).
\end{enumerate} 
A \emph{Pro-completion} of a small category \( \cat C \) is a free completion under codirected (equivalently cofiltered) limits. It is given by a category \( \pro(\cat C) \) with codirected limits together with a full embedding \( E \colon \cat C \hookrightarrow \pro(\cat C) \) satisfying the following universal property:
\begin{enumerate}
\item every functor \( F \colon \cat C \rightarrow \cat D \), where $\cat D$ has codirected limits, extends to a functor \( \barF \colon \pro(\cat C) \rightarrow \cat D \) that preserves codirected limits and satisfies $F=\barF\circ E$;
\item 
$\barF$ is {essentially unique}:
For every functor \( G \) that preserves codirected limits and satisfies $F=G\circ E$, there exists a natural isomorphism \( \alpha \colon \barF \cong G \) such that \( \alpha E = \id_{F} \).
\end{enumerate}
\[
  \begin{tikzcd}
    \cat{C}
    \rar[hook]{E}
    \drar[swap]{F}
    &
    \pro(\cat C)
    \dar[dashed]{\mybar{0.6}{2pt}{F}}
    \\
    &
    \cat{D}
  \end{tikzcd}
\qquad
  \begin{tikzcd}
    \cat{C}
    \rar[hook]{E}
    \drar[swap]{F}
    &
    \pro(\cat C)
    \dar[dashed]{G}
    \\
    &
    \cat{D}
  \end{tikzcd}
\]
The universal property determines $\pro(\cat C)$ uniquely up to equivalence of categories.
We note that every object $EC$ ($C\in \cat C$) is finitely copresentable in $\pro(\cat C)$, see e.g.~\cite[Thm~A.4]{ada-che-mil-urb-21}.
The dual of Pro-completions are \emph{Ind-completions}:
free completions under \emph{directed colimits}.

\begin{example}\label{ex:pro-completion}
  The Pro-completion $\pro(\setf)$ of the category of finite sets is
  the full subcategory of $\Top$ given by \emph{profinite spaces}
  (topological spaces that are codirected limits of finite discrete
  spaces). Profinite spaces are also known as \emph{Stone spaces} or
  \emph{boolean spaces} and can be characterized by topological
  properties: they are precisely compact Hausdorff spaces with a basis
  of clopen sets. This equivalent characterization depends on the
  axiom of choice (or rather the ultrafilter theorem, a weak form of
  choice), as does \emph{Stone duality}, the dual equivalence between
  the categories of Stone spaces and boolean algebras. The duality
  maps a Stone space to its boolean algebra of clopen sets, equipped
  with the set-theoretic boolean operations. Its inverse maps a
  boolean algebra the set of ultrafilters (equivalently, prime
  filters) on it, equipped with a suitable profinite topology.
\end{example}

\subparagraph*{Profinite words.} The topological approach to classical
regular languages is based on the space~$\hatSigmas$ of
\emph{profinite words} over the alphabet $\Sigma$. This space is constructed
as the codirected limit of all finite quotient monoids of $\Sigma^*$,
the free monoid of finite words generated by $\Sigma$. Formally, let
$\epislice{\Sigma^*}{\monf}$ be the codirected poset of all surjective
monoid morphisms $e\colon \Sigma^*\twoheadrightarrow M$, where $M$ is
a finite monoid; the order on $\epislice{\Sigma^*}{\monf}$ is defined
by $e\leq e'$ if $e'=e\cdot h$ for some $h$. Then $\hatSigmas$ is the
limit of the diagram
$D\colon \epislice{\Sigma^*}{\monf}\to \pro(\setf)$ sending
$e\colon \Sigma^*\twoheadrightarrow M$ to the underlying set of $M$,
regarded as a finite discrete topological space. The space
$\hatSigmas$ is completely metrizable; in fact, it is the Cauchy
completion of the metric space $(\Sigma^*,d)$ where
$d(v,w)=\sup \{\, 2^{-|M|} \mid \text{$M$ is a finite monoid
  separating $v,w$} \,\}$.  Here a monoid $M$ \emph{separates}
$v,w\in \Sigma^*$ if there exists a morphism $h\colon \Sigma^*\to M$
such that $h(v)\neq h(w)$. Regular languages over $\Sigma$ correspond
to clopen subsets of $\hatSigmas$, or equivalently to continuous maps
$L\colon \hatSigmas\to 2$ into the discrete two-element space.

\subparagraph*{Nominal Sets.} Fix a countable set \A of \emph{names}, and
denote by \( \perm \A \) the group of finite permutations,
i.e.~bijections $\pi\colon \A\to\A$ fixing all but finitely many
names. Given
$S\seq \A$ write
\[
  \perm_{S} \A = \{ \pi \in \perm \A \mid \pi(a) = a \text{ for all
    $a\in S$}\}
\]
for the the subgroup of permutations fixing $S$. A
\emph{\( \perm \A \)-set} is a set \( X \) with a group
action, that is, an operation 
\( \cdot \colon \perm \A \times X \rightarrow X \) such that 
\( \id \cdot x = x \) and
\(\pi \cdot (\sigma \cdot x) = (\pi \circ \sigma) \cdot x\) for every
\( x \in X \) and \(\pi, \sigma \in \perm \A \).  The \emph{trivial}
group action on \( X \) is given by \( \pi \cdot x = x \) for all
$x\in X$ and $\pi\in \perm\A$.

A subset \( S \subseteq \A \) is a \emph{support} of \( x \in X \) if
every permutation $\pi\in \perm_S\A$ acts trivially on \( x \), that
is, $\pi\cdot x = x$.  The idea is that \( x \) is some syntactic
object (e.g.\ a word, a tree, or a $\lambda$-term) whose free
variables are contained in \( S \).  A \( \perm \A \)-set $X$ is a
\emph{nominal set} if every element $x\in X$ has a finite support.
This implies that every $x\in X$ has a \emph{least} finite support,
denoted by \( \supp x \subseteq \A\).

For a nominal set \( X \) its \emph{nominal powerset}
\( \powfs X \subseteq \pow X\) consists of all subsets of
\( U \subseteq X \) which are \fs under the action
\( \pi \cdot U \vcentcolon = \{ \pi \cdot x \mid x \in U\} \).  For
example, for the nominal set \( \A \) of names with the action
$\pi\cdot a = \pi(a)$, its nominal powerset \( \powfs \A \) consists
of all finite and cofinite subsets of~$\A$.  A subset
\( U \subseteq X \) is \emph{equivariant} if it has empty support.  If
there exists a finite subset \( S \subseteq \A \) supporting every
\( x \in U \) then \( U \) is \emph{\ufs}, and \( S \) also supports
\( U \). Given a finite set \( S \subseteq \A \) of names and a subset
$U\seq X$, we define the \emph{\( S \)-hull} of \( U \) by
\( \hull_{S}U = \{\pi \cdot x \mid x \in U, \pi \in \perm_{S} \A\} \).
This is the smallest \( S \)-supported subset of \( X \) containing
\( U \).

For finite \( S \subseteq \A \) the \emph{\( S \)-orbit} of an element
\( x \in X \) is the set
\( \orb_{S} x = \{ \pi \cdot x \mid \pi \in \perm_{S} \A \} \). The
$\emptyset$-orbit of $x$ is called its \emph{orbit}, denoted $\orb x$.
We write \( \orb_{S} X = \{\orb_S x \mid x\in X\}\) for the set of all
$S$-orbits of $X$, and $\orb X$ for the set of all orbits. The $S$-orbits form a
partition of \( X \).  A finitely supported subset $Y\seq X$ is
\emph{orbit-finite} if it intersects only finitely many orbits of
$X$. In particular, the nominal set \( X \) is \emph{orbit-finite} if
$\orb X$ is a finite set. This implies that for every finite subset
\( S \subseteq \A \) the set $\orb_S X$ is finite. Moreover, $X$
contains only finitely many elements with support \( S \).

\begin{example}
  The set $\A^*$ of finite words over $\A$ forms a nominal set with
  the group action
  $\pi\cdot (a_1\cdots a_n) = \pi(a_1)\cdots \pi(a_n)$. The languages
  $L_0,L_1\seq \A^*$ from the Introduction are equivariant
  subsets. Given a fixed name $a\in \A$, the subset
  $L_2=\{awa\mid w\in \A^*\}$ is finitely supported with
  $\supp L_2 = \{a\}$. All the above sets have an infinite number of
  orbits. An example of an orbit-finite set is given by
  $\A^2=\A\times \A \seq \A^*$; its two orbits are
  $\{aa \mid a\in \A \}$ and $\{ab\mid a\neq b\in \A \}$.
\end{example}
A map \( f \colon X \rightarrow Y \) between nominal sets is
\emph{finitely supported} if there exists a finite set $S\seq \A$ such
that $f(\pi \cdot x)=\pi\cdot f(x)$ for all $x\in X$ and
$\pi\in \perm_S\A$, and \emph{equivariant} if it is supported by
$S=\emptyset$. Equivariant maps satisfy $\supp f(x)\seq \supp x$ for
all $x\in X$.  Nominal sets and equivariant maps form a category \nom,
with the full subcategory \nomof of {orbit-finite} nominal sets. The
category \nom is complete and cocomplete. Colimits and finite limits
are formed like in \set; general limits are formed by taking the limit in
$\set$ and restricting to finitely supported elements. The category
$\nomof$ is closed under finite limits and finite colimits in
$\nom$. \emph{Quotients} and \emph{subobjects} in \nom are represented
by surjective and injective equivariant maps. Every equivariant map
\( f \) has an image factorization \( f = m \cdot e \) with \( m \)
injective and \( e \) surjective; we call \( e \) the \emph{coimage}
of \( f \).

A nominal set is \emph{strong} if for all $x\in X$ and
$\pi\in \perm\A$ one has $\pi\cdot x= x$ iff $\pi \in \perm_{S} \A$,
where $S=\supp x$. (Note that the ``if'' direction holds in every
nominal set.)  For example, the nominal set
\( \A^{\#n} = \{f \colon n \rightarrow \A \mid \text{$f$ injective} \}
\) with pointwise action is strong and has a single orbit. Up to
isomorphism, (orbit-finite) strong nominal sets are precisely (finite)
coproducts of such sets.

\section{Nominal Stone Spaces}
\label{sec:nom-prof-spaces}

In this section, we establish the topological foundations for our pro-orbit-finite approach to data languages. We start by recalling the basic definitions of nominal topology~\cite{gab-lit-pet-11,pet-12}.

\begin{defn}\label{def:nom-top}
  \begin{enumerate}
  \item\label{def:nom-top:top} A \emph{nominal topology} on a nominal
    set \( X \) is an equivariant subset
    \( \mathcal{O}_{X} \subseteq \powfs X \) closed under finitely
    supported union (that is, if $\mathcal{U}\subseteq \mathcal{O}_X$
    is finitely supported then
    $\bigcup \mathcal{U} \in \mathcal{O}_X$) and finite intersection.
    Sets \( U \in \mathcal{O}_{X} \) are called \emph{open} and their
    complements \emph{closed}; sets that are both open and closed are
    \emph{clopen}.  A nominal set \( X \) together with a nominal
    topology \( \mathcal{O}_{X} \) is a \emph{nominal topological
      space}.  An equivariant map \( f \colon X \rightarrow Y \)
    between nominal topological spaces is \emph{continuous} if for every open set $U$ of $Y$ its
    preimage $f^{-1}[U]$ is an open set of \( X \).  Nominal
    topological spaces and continuous maps form the category
    \( \ntop \).

\item\label{def:nom-top:basis} A \emph{subbasis} of a nominal
  topological space $(X,\mathcal{O}_X)$ is an
  equivariant subset \( \mathcal{B} \subseteq \mathcal{O}_X \) such that
  every open set of $X$ is a finitely supported union of
  finite intersections of sets in \( \mathcal{B} \). If additionally every finite intersection of sets in $\mathcal{B}$ is a finitely supported union of sets in $\mathcal{B}$, then $\mathcal{B}$ is called a \emph{basis}. In this case, every open set of $X$ is a finitely supported union of elements of $\mathcal{B}$.  
\end{enumerate}
\end{defn}

\begin{example}\label{ex:nom-top}
  \begin{enumerate}
  \item A topological space may be viewed as a nominal
    topological space equipped with the trivial group
    action. Then every (open) subset has empty support and every union
    is finitely supported, so we recover the axioms of classical
    topology.
    
  \item Every nominal set \( X \) equipped with the \emph{discrete
      topology}, where all finitely supported subsets are open, is a
    nominal topological space.  It has a basis given by all singleton
    sets.
  \item\label{ex:nom-top:met} A \emph{nominal \mbox{(pseudo-)metric}
      space} is given by a nominal set \( X \) with a
    (pseudo-)metric\footnote{Recall that a pseudometric differs from a
      metric by not requiring $d(x,y)\neq 0$ for $x\neq y$.}
    \( d \colon X \times X \rightarrow \R \) which is equivariant as a
    function into the set \R, regarded as a nominal set with the
    trivial group action.  As usual, the open ball around
    \( x \in X \) with radius \( r > 0 \) is given by
    \( B_{r}x = \{y \in X \mid d(x, y) < r\} \).  Since
    \( \pi \cdot B_{r}(x) = B_{r}(\pi \cdot x) \) for all
    $\pi\in \perm\A$ and $x\in X$, every nominal (pseudo-)metric space
    carries a nominal topology whose basic opens are the open balls.
  \end{enumerate}
\end{example}

\begin{rem}\label{rem:func-ntop}
  Every nominal topological space induces two families of ordinary topological
  spaces, one by taking only opens with a certain support and the
  other by forming orbits. In more detail, let \( S \subseteq \A \) be
  a finite set of names and let \( X \) be a nominal topological space
  with topology \( \mathcal{O} \).
  \begin{enumerate}
    \item\label{rem:func-ntop:supp}
          The underlying set of the nominal space \( X \) carries a classical topology \( \mathcal{O}_{S} \) consisting of all \( S \)-supported open sets of \( \mathcal{O} \). We denote the resulting topological space by \( \forget[X]{S} \).
    \item\label{rem:func-ntop:orb}
          The set \( \orb_{S} X \) of \( S \)-orbits can  be equipped with the quotient topology \( \mathcal{O}_{\orb_{S} } \) induced by the projection \( X \epi \orb_{S} X \) mapping each \( x \in X \) to its \( S \)-orbit \( \orb_{S} x \).
          In this topology,
          a set \( O \subseteq \orb_{S} X \) of \( S \)-orbits is open iff its union \( \bigcup O\) is open in \( X \).
  \end{enumerate}
These constructions give rise to functors 
  $\forget{S}, \orb_{S}\colon \ntop \rightarrow \topo$. They allow us to switch between nominal and classical topology. 
\end{rem}
As noted in \autoref{ex:pro-completion}, the Pro-completion of the
category \setf is the category of profinite spaces. One may expect
that the Pro-completion of \nomof analogously consists of all
\emph{\pof} spaces, that is, nominal topological spaces that are
codirected limits of orbit-finite discrete spaces. However, this fails
due to a simple fact: while codirected limits of non-empty finite sets
are always non-empty (which is a consequence of Tychonoff's theorem,
thus the axiom of choice), codirected limits of non-empty
{orbit-finite} nominal sets may be empty.

\begin{rem} Similar to $\Top$, codirected limits in $\ntop$ are formed by taking the limit
  in $\nom$ equipping it with the initial topology.
\end{rem}

\begin{example}\label{ex:empty-limit}
  Consider the $\omega^\op$-chain
  $ 1 \leftarrow \mathbb{A} \leftarrow \mathbb{A}^{\#2} \leftarrow \mathbb{A}^{\#3} \leftarrow \cdots$
  in $\nomof$ with connecting maps omitting the last component.
  Its limit in \set (see \autoref{ex:cod-lim}) is given by $\A^{\#\omega}$, the set of all injective functions from $\omega$ to $\A$.
  Clearly no such function has finite support,
  thus the limit in \nom (and therefore also in \ntop) is empty.
\end{example}
This entails that it is in fact impossible to characterize \pronomof by any sort of spaces. By definition of the free completion $\pro(\nomof)$, the inclusion functor $I\colon \nomof \hookrightarrow \nom$ extends uniquely  to a functor \( \bar{I} \colon \pro(\nomof) \rightarrow \nom \) preserving codirected limits. The analogous functor $\bar I\colon \pro(\setf)\to \set$ is the forgetful functor of the category of profinite spaces. In contrast, we have
\begin{proposition}\label{prop:not-conc}
  The category \pronomof is \emph{not} concrete: the functor $\bar I$ is not faithful.
\end{proposition}
\begin{proof}
  Consider the chain
  $1 \leftarrow \A \leftarrow \A^{\#2} \leftarrow \cdots$ of
  \autoref{ex:empty-limit}. Let $D\colon \omega^\op\to \nomof$ denote
  the corresponding diagram, and let
  $E\colon \nomof\hookrightarrow \pro(\nomof)$ be the embedding. To
  prove that $\bar I$ is not faithful, we show that
  $|{\pro(\nomof)(\lim ED, E2)}|> |\nom(\bar I (\lim ED), \bar I E
  2)|$, where $2$ is the two-element nominal set. Indeed, we have
  \begin{align*}
    \pro(\nomof)(\lim_{n<\omega} ED_{n}, E2) &\cong \colim_{n<\omega} \pro(\nomof)(ED_{n}, E2) && \text{\( E2 \) finitely copresentable}\\
                                           &\cong \colim_{n<\omega} \nomof(D_{n}, 2) && E \text{ full embedding} \\
                                           &\cong 2 && 
\end{align*}
because $\nomof(D_0,2)\cong 2$ and the two elements are not merged by the colimit injection. However,
\begin{align*}
\nom(\bar I (\lim_{n<\omega} E D_n), \bar I E 2) & \cong \nom (\lim_{n<\omega} \bar I E D_n, \bar I E 2) &&  \text{ \( \bar I \) preserves codirected limits} \\
&  \cong \nom(\lim_{n<\omega} ID_n, 2) &&  I = \bar I E \\
&  \cong \nom(\emptyset, 2) &&  \text{\autoref{ex:empty-limit}} \\
& \cong 1. && \qedhere
  \end{align*}
\end{proof}

We thus restrict our focus to well-behaved subcategories of \nomof.
We choose these subcategories in such way that situations like in
\autoref{ex:empty-limit}, where unrestricted accumulation of supports
results in empty codirected limits, are avoided.

\begin{defn}
  A nominal set \( X \) is \emph{\kbnd}, for \( k \in \N \), if
  \( |{\supp x}| \le k \) for every \( x \in X \).
\end{defn}

For concrete categories $\cat C$ over \nom (or \nomof) we denote by
\( \cat{C}_{k} \) the full subcategory of \cat{C} whose underlying
objects are \kbnd. For instance, $\nom_k$ is the category of
$k$-bounded nominal sets, and $\ntop_k$ is the category of $k$-bounded
nominal topological spaces.

\begin{rem}\label{rem:nomk}
  \begin{enumerate}
  \item \label{rem:nomk:coref-subcat} The full subcategories
    \( \nomk \hookrightarrow \nom \) and
    \( \nomofk \hookrightarrow \nomof \) are
    coreflective~\cite[Section IV.3]{mac-71}: the coreflector (viz.\
    the right adjoint of the inclusion functor) sends a nominal
    set~\( X \) to its subset
    $X_{k} = \{ x \in X \mid |{\supp x}| \le k \}$.  Hence \nomk is
    complete: limits are formed by taking the limit in $\nom$ and
    applying the coreflector. Analogously, \nomofk is finitely
    complete.
    
  \item\label{rem:nomk:no-topos} In contrast to $\nom$,
    the category $\nom_k$ generally fails to be a topos because it is
    not cartesian closed. For instance, the functor
    $\A^{\#2}\times (-)$ on $\nom_2$ does not preserve coequalizers,
    hence it is not a left adjoint.
    
  \item\label{rem:nomk:presheaves}\sloppypar The category \nom is
    known to be equivalent to the category of pullback-preserving
    presheaves \(\mathbb{I} \rightarrow \set\), where \(\mathbb{I}\)
    is the category of finite sets and injective
    functions~\mbox{\cite[Theorem 6.8]{pit-13}}.  By inspecting the proof it
    is easy to see that this restricts to an equivalence between \nomk
    and the category of \emph{\( k \)-generated} pullback-preserving
    presheaves \( \mathbb{I} \rightarrow \set \).  Here a presheaf
    \( F \colon \mathbb{I} \rightarrow \set \) is 
    \( k \)-generated if for every finite set \( S \) and every
    \( x \in FS \) there exists a set \( S' \) of cardinality at most
    \( k \) and an injective map \( f \colon S' \rightarrow S \) such
    that \( x \in Ff[FS'] \).
  \end{enumerate}
\end{rem}

With regard to codirected limits, the restriction to bounded nominal sets fixes the issue arising in \autoref{ex:empty-limit}:

\begin{lemma}\label{lem:cod-lim-set}
  Codirected limits in $\nom_k$ are formed at the level of \set.
\end{lemma}

 We proceed to give a topological characterization of $\pro(\nomofk)$ in terms of nominal Stone spaces, generalizing the corresponding result \eqref{eq:stone-pro} for $\pro(\setf)$. To this end, we introduce suitable nominalizations of the three characteristic properties of Stone spaces: compactness, Hausdorffness, and existence of a basis of clopens. The nominal version of compactness comes natural and is compatible with the functors $\forget{S}$ and $\orb_{S}$ of \autoref{rem:func-ntop}.

\begin{defn}\label{def:compact}
  An \emph{open cover} of a nominal topological space
  \( (X,\mathcal{O}) \) is a finitely supported set
  \( \mathcal{C} \subseteq \mathcal{O} \) that covers \( X \),
  i.e.~\( \bigcup \mathcal{C} = X \).  A \emph{subcover} of
  \( \mathcal{C} \) is a finitely supported subset of
  \( \mathcal{C} \) that also covers \( X \).  A nominal topological
  space \( X \) is \emph{compact} if every open cover
  \( \mathcal{C} \) of \( X \) has an orbit-finite subcover: there
  exist \( U_{1}, \ldots, U_{n} \in \mathcal{C} \) such that
  \( X = \bigcup_{i=1}^{n}\bigcup \orb U_{i} \).
\end{defn}

\begin{lemma}\label{lem:comp}
  For every nominal topological space \( X \)  the following conditions are equivalent:
  \begin{enumerate}
    \item\label{lem:comp:def} The space \( X \)  is compact.
    \item\label{lem:comp:unif} Every uniformly finitely supported open
      cover of \( X \) has a finite subcover.
      
    \item\label{lem:comp:supp} For every finite set
      \( S \subseteq \A \) the topological space \( \forget[X]{S} \)
      is compact.

    \item\label{lem:comp:orb} For every finite set
      \( S \subseteq \A \) the topological space \( \orb_{S} X \) is
      compact.
  \end{enumerate}
\end{lemma}

The Hausdorff property is more subtle: rather than just separation of
points, we require separation of $S$-orbits (``thick points'') by disjoint $S$-supported
open neighbourhoods.
\begin{defn}\label{def:hausdorff}
  A nominal topological space \( X \) is \emph{(nominal) Hausdorff} if
  for every finite set \( S \subseteq \A \) and every pair
  \( x_{1}, x_{2} \in X \) of points lying in different
  \( S \)-orbits, there exist disjoint \( S \)-supported open sets
  \( U_{1}, U_{2} \subseteq X \) such that \( x_{i} \in U_i \) for
  \( i=1,2 \).
\end{defn}
Note that the nominal Hausdorff condition is clearly equivalent to being able to separate
disjoint \( S \)-\emph{orbits}: If
$\orb_{S} x_{1} \ne \orb_{S} x_{2}$, then any two disjoint open
\( S \)-supported neighbourhoods \( U_{1}, U_{2} \) of
\( x_{1}, x_{2} \) satisfy $\orb_S x_i\seq U_i$ for $i=1,2$. Note also that
$\orb_S x=\{x\}$ whenever $\supp x\seq S$, hence the nominal Hausdorff
condition implies the ordinary one. For bounded nominal compact
Hausdorff spaces, we have a codirected Tychonoff theorem:

\begin{proposition}\label{prop:cod-lim-ne}
 For every codirected diagram of non-empty \kbnd nominal compact Hausdorff spaces, the limit in $\ntop$ is a non-empty \kbnd nominal compact Hausdorff space.
\end{proposition}
Finally, having a basis of clopen sets is not sufficient in our setting. To see this, note that in an ordinary topological space $X$ every clopen subset $C\seq X$ can be represented as $C=f^{-1}[A]$ for some continuous map $f\colon X\to Y$ into a finite discrete space $Y$ and some subset $A\seq Y$. (In fact, one may always take $Y=2$ and $A=\{1\}$.) This is no longer true in the nominal setting, see \autoref{rem:rep} below. Therefore, in lieu of clopens we work with \emph{representable} subsets:

\begin{defn}
  A subset \( R \subseteq X \) of a nominal space \( X \)  is \emph{representable} if there exists a continuous map \( f \colon X \rightarrow Y \)  into an orbit-finite discrete space \( Y \) such that \( R = f^{-1}[A] \) for some \( A \in \powfs Y \).
\end{defn}

\begin{rem}\label{rem:rep}
  \begin{enumerate}
  \item\label{rem:rep:not-equiv} Every representable set is clopen,
    but the converse generally fails. To see this, consider the
    discrete space \( X = \coprod_{n < \omega} \A^{\#n} \). We show
    that for fixed \( a \in \A \) the (clopen) subset
    \( R = \{x \mid a \in \supp x\} \subseteq X \) is not
    representable.  Towards a contradiction suppose that \( R \) is
    represented by \( f \colon X \rightarrow Y \) as
    \( R = f^{-1}[A]\) for some $A\in \powfs Y$.  Since $Y$ is
    orbit-finite, we can choose~\( m \) large enough such that there exists some
    \( x \in \A^{\#m} \setminus R \subseteq X \) for which
    \( \supp f(x) \subsetneq \supp x \). Choose a name
    \( b \in \supp x \setminus \supp f(x) \).  Then
    \( a, b \not\in \supp f(x) \), and so we have
    \[
      f(\tr{a}{b} \cdot x) = \tr{a}{b} \cdot f(x) = f(x).
    \]
    Since \( \tr{a}{b} \cdot x \in R \), this shows $f(x)\in A$ and thus $x\in R$. This contradicts the above choice of $x$. 

  \item\label{rem:rep:basis} If a nominal space \( X \) has a basis of
    representable sets, then we may assume without loss of generality
    that the basic open sets are of the form \( f^{-1}[y] \) for some
    $f \colon X \rightarrow Y$ and $y \in Y$, where~\( Y \) is
    orbit-finite and discrete. Indeed, if $R = f^{-1}[A]$ for
    $A \in \powfs Y$, then \( R = \bigcup_{y \in A}f^{-1}[y] \).
    Moreover, given representable sets
    $R_{i} = f_{i}^{-1}[y_{i}]$, $i = 1,2$, the set
    \( R_{1} \cap R_{2} \) is equal to
    \( \langle f_{1}, f_{2} \rangle^{-1}(y_{1}, y_{2}) \) and
    therefore representable as well.  Hence, to show that
    representable subsets form a basis it suffices to check whether
    every open set is a finitely supported union of subsets of the form $f^{-1}[y]$.
  \end{enumerate}
\end{rem}
\begin{defn}\label{def:nom-stone}
  A \emph{nominal Stone space} is a nominal compact Hausdorff space
  with a basis of representables. We let \( \npro \) denote the full
  subcategory of $\ntop$ given by nominal Stone spaces.
\end{defn}
\begin{rem}\label{rem:nom-stone-prof}
  \newcommand{\rest}{\ensuremath{\mathrm{\mathbf{n}}}\xspace} Nominal
  Stone spaces as per \autoref{def:nom-stone} are conceptually very
  different from \emph{nominal Stone spaces with $\new$}, introduced
  by Gabbay et al.~\cite{gab-lit-pet-11} as the dual of \emph{nominal
    boolean algebras with $\new$}. The latter are equipped with a
  \emph{restriction operator} \rest tightly related to the freshness
  quantifier $\new$ of nominal sets, which enables a nominal version
  of the {ultrafilter theorem} and thus a represention of boolean
  algebras with $\new$ via spaces of ultrafilters. In {nominal Stone
    spaces with $\new$}, the Hausdorff property is implicit (but would
  be analogous to that in standard topology), the basis is given by
  clopen rather than representable sets, and the notion of compactness
  (called \emph{$\rest$-compactness}) considers open covers closed
  under the operator \rest, which are required to have a \emph{finite}
  subcover. By this definition, the orbit-finite discrete space \A
  fails to be compact (the \rest-cover \( \{\{a\} \mid a \in \A\} \cup \{ \emptyset\} \)
  has no finite subcover). Hence, given that algebraic recognition is
  based on orbit-finite sets, nominal Stone spaces with $\new$ are not
  suitable for a topological interpretion of data languages.
\end{rem}

\begin{example}\label{ex:of-prof}
  Every orbit-finite nominal set can be viewed as a nominal Stone
  space equipped with the discrete topology.  We thus regard
  \( \nomof \) as a full subcategory of \( \npro \).  Nontrivial
  examples of nominal Stone spaces are given by the spaces of
  pro-orbit-finite words introduced later.
\end{example}
Within the class of nominal Stone spaces, representable and clopen subsets coincide:
\begin{lemma}\label{lem:rep:lim}
If $X$ is a nominal Stone space, then every clopen set \( C \subseteq X \) is representable.
\end{lemma}
The following theorem is the key result leading to our topological
approach to data languages.
\begin{theorem}\label{thm:pro-comp}
  For each $k\in \N$, the category of \kbnd nominal Stone spaces is
  the Pro-completion of the category of \kbnd orbit-finite nominal
  sets:
  \[
    \pro(\nomofk) = \npro_k.
  \]
  Moreover, $k$-bounded nominal Stone spaces are precisely the nominal
  topological spaces arising as codirected limits of $k$-bounded
  orbit-finite discrete spaces.%
\end{theorem}
For $k=0$, we recover the corresponding characterization of classical Stone spaces.

\section{Nominal Stone Duality}
Next, we give a dual characterization of (bounded) nominal
Stone spaces.  It builds on the known duality between nominal sets and
complete atomic nominal boolean algebras due to Petri\c{s}an~\cite{pet-12}.

\begin{defn}\label{def:nom-ba}
  A \emph{nominal boolean algebra} is a nominal set equipped with the
  structure of a boolean algebra such that all operations are equivariant.
  It is \emph{(orbit-finitely) complete} if every (orbit-finite)
  finitely supported subset has a supremum.  A \emph{subalgebra} of an
  (orbit-finitely) complete nominal boolean algebra is an equivariant
  subset closed under boolean operations and the respective
  suprema. Let \ncofba and \ncba denote the categories of
  (orbit-finitely) complete nominal boolean algebras; their morphisms
  are\ equivariant homomorphisms preserving (orbit-finite) suprema.
\end{defn}

\begin{defn}\label{def:nom-ba-at}
  An element \( x \in B \) of a nominal boolean algebra is an
  \emph{atom} if $x\neq \bot$ and $y<x$ implies $y=\bot$. The
  (equivariant) set of atoms of \( B \) is denoted \( \at B \).  The
  algebra $B$ is \emph{atomic} if every element is the supremum of all
  atoms below it; if additionally \( \at B \in \nomofk \) we call it
  \emph{\( k \)-atomic}.  If \( A \subseteq B \) is a \( k \)-atomic
  subalgebra we write \( A \sa B \).  An algebra \( B \in \ncba \) is
  called \emph{locally \( k \)-atomic} if every element of \( B \) is
  contained in some \( A \sa B \).  We denote by
  \( \ncakba \subseteq \ncba \) the full subcategory of all
  \( k \)-atomic complete nominal boolean algebras, and $\ncofalkba\seq \ncofba$
  denotes the full subcategory of all locally \( k \)-atomic orbit-finitely complete
   nominal boolean algebras.
\end{defn}

\begin{rem}\label{rem:of-complete}
  \begin{enumerate}
  \item\label{rem:of-complete:equiv} Orbit-finite completeness is
    equivalent to the weaker condition that suprema of \( S \)-orbits
    exist for all finite subsets \( S \subseteq \A \).  In fact, every
    $S$-supported orbit-finite subset \( X \subseteq B \) is a finite
    union \( X = \bigcup_{i=1}^{n} \orb_{S} x_{i} \) of
    \( S \)-orbits, whence
    \( \bigvee X = \bigvee_{i=1}^{n} \bigvee \orb_{S} x_{i} \).
    
  \item\label{rem:of-complete:atomic} Every $k$-atomic orbit-finitely complete
    nominal boolean algebra is complete: For every \fs
    subset \( X \subseteq B \) we have
    $\bigvee X = \bigvee \{b \in \at (B) \mid \exists(x \in X).\ b \le
    x\}$, which is a supremum of an orbit-finite subset.
  \end{enumerate}
\end{rem}
\begin{theorem}\label{thm:ind-compl}
  For each \( k \in \N \), the category of locally \( k \)-atomic
  orbit-finitely complete nominal boolean algebras is the
  Ind-completion of the category of \( k \)-atomic complete nominal
  boolean algebras:
  \[
    \ncofalkba \simeq \ind(\ncakba).
  \]
\end{theorem}

\begin{theorem}[Nominal Stone Duality]\label{thm:duality}
  For each \( k \in \N \), the category of locally \( k \)-atomic
  orbit-finitely complete nominal boolean algebras is dual to the category of \kbnd
  nominal Stone spaces:
  \[\ncofalkba \simeq^{\mathrm{op}} \npro_{k}.\]
\end{theorem}
\begin{proof}
  The category \nom of nominal sets is dually equivalent to the
  category \ncaba of complete atomic nominal boolean
  algebras~\cite{pet-12}.  The duality sends a nominal set $X$ to the
  boolean algebra~$\powfs X$, equippped with the set-theoretic boolean
  structure.  Conversely, a complete atomic nominal boolean algebra
  \( B \) is mapped to the nominal set \( \at(B) \) of its atoms, and
  an $\ncaba$-morphism \(h\colon C \rightarrow B \) to the equivariant
  map \( \at(B) \rightarrow \at(C) \) sending \( b \in \at(B) \) to
  the unique \( c \in \at(C) \) such that \( c \le h(b) \). For every
  \( k \in \N \) the duality clearly restricts to one between \kbnd
  orbit-finite nominal sets and \( k \)-atomic complete nominal
  boolean algebras.  Thus \autoref{thm:ind-compl} and
  \autoref{thm:pro-comp} yield
  \begin{equation*}
  \ncofalkba \simeq \ind(\ncakba) \simeq^{\mathrm{op}} \pro(\ncakba^{\mathrm{op}}) \simeq \pro(\nomofk) \simeq \npro_{k}. \qedhere
\end{equation*}
\end{proof}

\begin{rem}\label{rem:concrete-duality}
  We give an explicit description of the dual equivalence of \autoref{thm:duality}.
  \begin{enumerate}
    \item In the direction \( \npro_{k} \rightarrow \ncofalkba \) it maps a \kbnd nominal Stone space \( X \) to the nominal boolean algebra \( \clo(X) \) of clopens (or representables, see \autoref{lem:rep:lim}).
          A continuous map \( f \colon X \rightarrow Y \) is mapped to the homomorphism \( f^{-1}\colon \clo(Y) \rightarrow \clo(X)\) taking preimages.
    \item
          The direction \( \ncofalkba \rightarrow \npro_{k} \) requires some terminology. A finitely supported subset $F\seq B$ of an algebra $B\in \ncofalkba$ is a \emph{nominal orbit-finitely complete prime filter} if (i) $F\neq \emptyset$, (ii) $F$ is upwards closed (\( x \in F \land x \le y \Rightarrow y \in F \)), (iii) $F$ is downwards directed (\( x, y \in F \Rightarrow x \land y \in F \)), and (iv) for every \fs \kbnd orbit-finite subset \( X \subseteq B \) such that \( \bigvee X \in F \), one has $X\cap F\neq \emptyset$. The equivalence now maps $B\in \ncofalkba$ to the space \( \fp(B) \) of nominal orbit-finitely complete prime filters of \( B \), whose topology is generated by the basic open sets \( \{F \in \fp(B) \mid b \in F \} \) for \( b \in B \).
          A morphism \( h \colon B \rightarrow C \) of $\ncofalkba$ is mapped to the continuous map  \( h^{-1} \colon \fp(C) \rightarrow \fp(B) \) taking preimages.
  \end{enumerate}
\end{rem}

In \autoref{thm:duality} we made the support bound \( k \) explicit,
but we can also leave it implicit. A nominal Stone space is
\emph{bounded} if it lies in $\npro_k$ for some natural number~$k$; similarly, a
\emph{locally bounded atomic orbit-finitely complete nominal boolean
  algebras} is an element of $\ncofalkba$ for some~$k$.
\begin{corollary}\label{cor:bnd-duality}
  The category of locally bounded atomic orbit-finitely complete
  nominal boolean algebras is dual to the category of bounded nominal
  Stone spaces.
\end{corollary}
\begin{rem}
  For \( k = 0 \) we recover the classical Stone duality between boolean algebras and Stone spaces. Indeed, \kbnd[0] nominal Stone spaces
  are precisely Stone spaces, and locally \( 0 \)-atomic
  orbit-finitely complete nominal boolean algebras are precisely
  boolean algebras
\end{rem}

\section[Pro-Orbit-Finite Words]{Pro-Orbit-Finite Words}
\label{sec:prof-nomin-mono}
In this section, we generalize the topological characterization of
regular languages to data languages recognizable by orbit-finite
nominal monoids~\cite{boj-13,col-15,boj-20}.

\begin{defn}\label{defn:nom-mon}
  A \emph{nominal monoid} \( M \) is a monoid object in \nom, that is, it is given by
  nominal set $M$ equipped with an equivariant associative
  multiplication \( M \times M \rightarrow M \) and an equivariant
  unit \( 1 \in M \). Nominal monoids and equivariant monoid homomorphisms form a category \( \nmon \).
\end{defn}
As for ordinary monoids, the \emph{free monoid} generated by
$\Sigma\in\nom$ is the nominal set $\Sigma^*$ of finite words (with
pointwise group action); its multipliation is concatenation and its
unit the empty word.

\begin{rem}
  We emphasize the difference between \kbnd nominal monoids -- nominal
  monoids whose carrier is \kbnd\/ -- and monoid objects in
  \( \nom_{k} \), which are partial nominal monoids where the product
  $x\cdot y$ is defined iff $|{\supp x} \cup \supp y|\leq k$.
\end{rem}

\begin{defn}
  A \emph{data language} over $\Sigma\in \nomof$ is a finitely
  supported subset $L\seq \Sigma^*$. It is \emph{recognizable} if
  there exists an equivariant monoid morphism
  \( h \colon \fm \rightarrow M \) with \( M \) orbit-finite and a \fs
  subset \( P \subseteq M \) such that \( L = h^{-1}[P] \). In this
  case, we say that the morphism \emph{$h$ recognizes $L$}.
\end{defn}
For example, the equivariant language $L_0$ from the Introduction is recognizable, while the language $L_1$ is not recognizable.
\begin{rem}
\begin{enumerate}
\item The morphism $h$ can be taken to be surjective; otherwise, take its coimage.

\item Via characteristic functions, data languages correspond
  precisely to \fs maps \( L \colon \fm \rightarrow 2 \), where $2$ is
  the two-element nominal set. Recognizablity then states that $L$
  factorizes through some equivariant monoid morphism with
  orbit-finite codomain.
\end{enumerate}
\end{rem}

Recall from \autoref{sec:preliminaries} that the Stone space
$\hatSigmas$ of profinite words over a finite alphabet $\Sigma$ is
constructed as the limit in \( \stone \simeq \pro(\setf) \) of all
finite quotient monoids of $\Sigma^*$. The obvious generalization to a
\emph{nominal} alphabet \( \Sigma \in \nomof \), which constructs the limit of all
orbit-finite quotient monoids in $\pro(\nomof)$, is unlikely to yield a
useful object since this category is not concrete
(\autoref{prop:not-conc}); in fact, it is futile from a language-theoretic perspective, cf.~\autoref{rem:big-limit}.
Instead, our results of
\autoref{sec:nom-prof-spaces} suggest to restrict the diagram scheme
to \( \epislice{\fm}{\nmonofk} \), the poset of $k$-bounded
orbit-finite quotient monoids (where $e\leq e'$ iff $e'$ factorizes
through $e$), and take the limit in \( \pro(\nomofk)=\npro_{k}
\). However, this diagram is not codirected
(\autoref{rem:canon-not-cod}), so its limit may not be a nominal Stone
space. We again focus on well-behaved (i.e., codirected),
subcategories by introducing \emph{support bounds}.

\begin{defn}\label{def:sbnd}
  A \emph{support bound} is a map
  \( s \colon \fm \rightarrow \pow \A \) such that
  $s[\fm]\seq \pow_k\A$ for some $k\in \N$, where
  $\pow_k\A = \{ S\seq \A \mid |S|\leq k \}$. We usually identify $s$
  with its codomain restrictions to $\pow_k\A$ for sufficiently large
  $k$. A morphism \( h \colon \fm \rightarrow M \) of nominal monoids
  is \emph{\sbnd} if $\supp h(w) \subseteq s(w)$ for all $w\in \fm$;
  we write \( h \colon \fm \rightarrow_{s} M \). We
  denote by \sepislice{\fm}{\nmonofk} the subposet of
  $\epislice{\fm}{\nmonofk}$ given by $s$-bounded quotient monoids.
\end{defn}

\begin{lemma}\label{lem:s-slice-cod}
  For every support bound \( s \), the poset \sepislice{\fm}{\nmonofk} is codirected.
\end{lemma}
\begin{proof}
  Let $h \colon \fm \rightarrow_s M_{h}$ and
  $h' \colon \fm \rightarrow_s M_{h'}$ be two \sbnd quotients in
  $\sepislice{\fm}{\nmonofk}$. Form the coimage $k \colon \fm \epi M$ of
  their pairing
  \( \langle h, h' \rangle\colon \Sigma^*\to M_{h}\times M_{h'} \). Then for all $w\in\fm$
    \[\supp k(w) = \supp(h(w), h'(w))
               = \supp h(w)  \cup \supp h'(w)
               \subseteq s(w).\]
  Hence, \( k \) is a lower bound for \( h, h' \) in the poset \sepislice{\fm}{\nmonofk}.
\end{proof}

\begin{defn}\label{def:pfm}
  For an orbit-finite nominal set \( \Sigma \) and a support bound
  \( s \colon \fm \rightarrow \pow_{k} \A \) we define the nominal
  Stone space \spfm to be the limit of the codirected diagram
  \[
    D\colon \sepislice{\fm}{\nmonofk} \to \npro_k,\qquad (e\colon \fm
    \epito_s M)\;\mapsto\; |M|,
  \]
  where $|M|$ is the nominal set underlying $M$, regarded as a
  discrete nominal topological space. The elements of $\spfm$ are
  called the \emph{($s$-bounded) pro-orbit-finite words over
    $\Sigma$}. We denote by \( \hat{e} \colon \spfm \to M \) the limit
  projection associated to $e\colon \fm \epito_s M$ in
  $\sepislice{\fm}{\nmonofk}$.
\end{defn}

\begin{rem}\label{rem:canon}
\begin{enumerate}
\item One may equivalently define $\spfm$ as the limit of the larger
  cofiltered diagram $D'$ given by
  \[
    D'\colon \sslice{\fm}{\nmonofk} \to \npro_k,\qquad (e\colon \fm
    \to_s M)\;\mapsto\; |M|,
  \]
  where $\sslice{\fm}{\nmonofk}$ is the category of all equivariant
  $s$-bounded monoid morphisms $h\colon \fm\to_s M$ with $k$-bounded orbit-finite
  codomain; a morphism from $h$ to $h'\colon \fm\to_s M'$ is an
  equivariant monoid morphism $k\colon M\to M'$ such that
  $h'=k\cdot h$.  In fact, the inclusion
  $\sepislice{\fm}{\nmonofk}\hookto\sslice{\fm}{\nmonofk}$ is an
  initial functor, hence the limits of $D$ and $D'$ coincide.  Since
  the limit of $D'$ is formed as in $\set$
  (\autoref{lem:cod-lim-set}), the space $\spfm$ is carried by the
  nominal set of compatible families $(x_h)_{h}$ of $D'$, and the
  limit projection $\hat h$ associated to $h\colon \Sigma^*\to_s M$ is
  given by $(x_h)_{h}\mapsto x_h$.
  
\item\label{rem:canon:dense} The forgetful functor
  \( V \colon \npro_{k} \rightarrow \nomk \) and the inclusion
  \( I \colon \nomk \rightarrow \nom \) both preserve codirected
        limits. The morphisms \( \sslice{\fm}{\nmonofk} \) viewed as equivariant functions form a cone for the
        diagram \( IVD' \), so there exists a unique equivariant map
  \( \eta \colon \fm \rightarrow IV \spfm \) such that
  \[
    h = \big(
    \begin{tikzcd}[cramped,column sep = 20]
      \fm
      \ar[dashed]{r}{\eta}
      &
      IV\spfm
      \ar{r}{IV\hat h}
      &
      IV\!M
    \end{tikzcd}
    \big) \qquad \text{for all \( h \in \sslice{\fm}{\nmonofk}\)}. 
  \]
  In more explicit terms, the map $\eta$ is given by
  $\eta(w)=(h(w))_{h}$ for $w\in \Sigma^*$.  For simplicity we omit
  $I$ and $V$ and write \( \eta\colon \fm \rightarrow \spfm \).  The
  image of $\eta$ forms a dense subset of $\spfm$. We note that~$\eta$
  is generally not injective since we restrict a subdiagram $\sslice{\fm}{\nmonofk}$ of
  the diagram \( \slice{\fm}{\nmonof} \),
  
\item\label{rem:canon:monoid-structure} The space \spfm is a nominal
  monoid with product
  \( \hat{h}(x \cdot y) = \hat{h}(x) \cdot \hat{h}(y) \) and unit
        $\eta(\varepsilon)$, with $\varepsilon$ the empty word.
  Since the multiplication is readily seen to be continuous,
  $\spfm$ can be regarded as an object of $\mon(\npro)$, the category
  of nominal Stone spaces equipped with a continuous monoid structure
  and continuous equivariant monoid morphisms.
\end{enumerate}
\end{rem}

Now recall from \autoref{sec:preliminaries} that the space
$\hatSigmas$ can be constructed as the metric completion of
$\Sigma^*$, where the metric measures the size of separating
monoids. We now investigate to what extent the metric approach applies
to the nominal setting, using nominal (pseudo-)metrics; see
\autoref{ex:nom-top}.

\begin{defn}\label{def:met-mon}
  Let \( s \) be a support bound on \( \fm \).  We say that a nominal
  monoid \( M \) \emph{$s$-separates} \( v, w\in \Sigma^* \) if there
  exists an $s$-bounded equivariant monoid morphism
  \( h\colon \Sigma^*\to_s M \) such that \( h(v) \ne h(w) \).  We
  define a nominal pseudometric $d_s$ on \( \fm \) by setting
  \[
    d_s(v,w) = \sup \{\, 2^{-|\orb M|} \mid \text{the orbit-finite
      nominal monoid $M$ $s$-separates $v,w$}\, \}.
  \]
  We let \( \fm / d_{s} \) denote the corresponding nominal metric
  space, obtained as a quotient space of the pseudometric space
  \( (\fm , d_{s}) \) by identifying \( v, w \) if
  \( d_{s}(v, w) = 0 \).
\end{defn}
\begin{rem}
  In contrast to the classical case, $d_s$ is generally not a metric:
  there may exist words \( v \ne w \) which are not $s$-separated by
  any orbit-finite nominal monoids. For example, if
  \( \Sigma = \A$ and
  $s(a_{1} \cdots a_{n}) = a_{1} \) for $a_1,\ldots, a_n\in \Sigma$,
  then for every \sbnd \( h \) and distinct names $a,b,c\in \A$ we have
  \( h(ab) = h(\tr b c \cdot ac) = \tr b c \cdot h(ac) = h(ac) \)
  since \( b, c \not \in s(ac) \supseteq \supp h(ac) \).  Therefore,
  the additional metrization process is required.
\end{rem}
For the next lemma we need some terminology. A nominal metric space is \emph{complete} if every finitely supported Cauchy sequence has a limit. A nominal topological space is \emph{completely metrizable} if its topology is induced by a complete metric. A subset $D\seq X$ of a nominal metric space is \emph{(topologically) dense} if every open neighbourhood of a point $x\in X$ contains an element of $D$.

\begin{rem}\label{R:dense0}
  In contrast to classical metric spaces, density is not equivalent to
  \emph{sequential density} (every point $x\in X$ is a limit of a
  finitely supported sequence in $D$). To see this, consider the
  space~\( \A^{\omega} \) of finitely supported infinite words with
  the prefix metric, that is, $d(v,w)=2^{-n}$ if $n$ is the length of
  the longest common prefix of $v,w$. Let $D\seq X$ be the equivariant
  subset given by
  \[
    D = \{\,x \in \A^{\omega} \mid |{\supp x}| \ge 2 \text{ and }
    |{\supp x}| \ge |\mathrm{initial block}(x)|\, \},
  \]
  where $\mathrm{initial block}(x)$ is the longest prefix of $x$ of
  the form $a^n$ ($a\in \A$).  The set $D$ is dense, but
  not sequentially dense: \( a^{\omega} \in \A^{\omega} \) is not the
  limit of any finitely supported sequence in \( D \).
\end{rem}

\begin{lemma}\label{lem:pfm-metr}
  \begin{enumerate}
  \item The space \spfm is completely metrizable via the complete nominal metric 
    \begin{equation}\label{eq:metric-hat-d}
      \hat d_s(x,y)
      =
      \sup \{\, 2^{-|{\orb M}|} \mid \exists (h \colon \fm \rightarrow_{s} M)\colon \hat{h}(x) \ne \hat{h}(y)\,\}.
    \end{equation}
    
  \item\label{lem:pfm-metr:dense} The canonical map $\eta$
    (\autoref{rem:canon}) yields a dense isometry
    \( \eta \colon (\fm,d_s) \to (\spfm, \hat d_s) \).
  \end{enumerate}
\end{lemma}
\begin{rem}\label{rem:dense}
  In classical topology, it would now be clear that \( \spfm \) is the
  {metric completion} of the metric space \( \fm / d_{s} \), i.e.~it
  satisfies the universal property that every uniformly continuous map
  from \( \fm / d_{s} \) to a complete metric space has a unique
  uniformly continuous extension to \( \spfm \).  However, this rests
  on the coincidence of topological and sequential density, which
  fails over nominal sets as seen in \autoref{R:dense0}. We therefore conjecture that
  \spfm is not the nominal metric completion of \( \fm / d_{s} \).
\end{rem}

By using support bounds, we obtain a topological perspective on
recognizable data languages. Let $\rec_s{\Sigma}$
denote the set of data languages recognized by \sbnd equivariant monoid morphisms.
\begin{theorem}\label{thm:rec-rep-equiv}
  For every support bound \( s \colon \fm \rightarrow \pow_{k} \A \),
  the $k$-bounded nominal Stone space $\spfm$ of $s$-bounded
  pro-orbit-finite words is dual to the 
  locally \( k \)-atomic orbit-finitely complete boolean algebra $\rec_s(\fm)$ of
  $s$-recognizable languages. In particular, we have the isomorphism
  \[\rec_{s}(\fm) \;\cong\; \clo(\spfm)\qquad \text{ in } \qquad \ncofalkba. \]
\end{theorem}
\begin{proof}[Proof (Sketch)] The isomorphism is illustrated by the two diagrams below:
\[
   \begin{tikzcd}[row sep=small, column sep=small]
            \mathllap{L\phantom{])}} = h^{-1}[P]
            \arrow[phantom]{r}{\subseteq}
            \dar[mapsto, shift right=13.5]{}
            &
            \fm
            \arrow{r}{h}
            \arrow{d}[swap]{\eta}
            &
            M
            \arrow[phantom]{r}{\supseteq}
            &
              P
            \\
            \mathllap{\overline{ \eta[L] }} = \hat{h}^{-1}[P]
            \arrow[phantom]{r}{\subseteq}
            &
            \spfm
            \arrow[dashed]{ru}[swap]{\hat{h}}
            &
            &
          \end{tikzcd}
\qquad\qquad
    \begin{tikzcd}[row sep=small, column sep=small]
            \mathllap{\eta^{-1}[C] }= h^{-1}[P]
            \arrow[phantom]{r}{\subseteq}
            &
            \fm
            \arrow{r}{h}
            \arrow{d}[swap]{\eta}
            &
            M
            \arrow[dashed]{d}{\exists p}
            \arrow[phantom]{r}{\supseteq}
            &
              P \mathrlap{ =  p^{-1}[U] }
            \\
            \mathllap{C\phantom{]}} = f^{-1}[U]
            \arrow[phantom]{r}{\subseteq}
            \uar[mapsto, shift left=12]{}
            &
            \spfm
            \arrow[dashed]{ru}{\hat{h}}
            \arrow{r}{f}
            &
            Y
            \arrow[phantom]{r}{\supseteq}
            &
            U
          \end{tikzcd}
\]
 In more detail, if $L\seq \Sigma^*$ is $s$-recognizable, say \( L = h^{-1}[P] \) for an $s$-bounded morphism \( h \),
  then its corresponding clopen is the topological closure \( \overline{\eta[L]} = \hat{h}^{-1}[P] \) represented by the continuous extension \( \hat{h} \).
  Conversely, every clopen \( C \subseteq \spfm \) restricts to an $s$-recognizable language \( \eta^{-1}[C] \subseteq \fm \).  
  We get \( s \)-recognizability of \( \eta^{-1}[C] \) by factorizing a representation $f\colon \spfm\to Y$ of $C$ through a limit projection \( \hat{h} \) as \( f = p \cdot \hat{h} \), using that $Y$ is finitely copresentable.
  Thus \( h \) recognizes \( \eta^{-1}[C] \).
\end{proof}

\begin{rem}\label{rem:big-limit}
  In the proof of \autoref{thm:rec-rep-equiv}, finite copresentability
  of orbit-finite sets is crucial to recover recognizable languages
  from representable subsets, highlighting the importance of working
  in the Pro-completion $\pro(\nomofk)=\npro_k$.  In a naive approach
  one might instead want to consider the limit of the diagram
  $D\colon\slice{\fm}{\nmonof} \to \ntop$ of \emph{all} equivariant
  morphisms from \( \fm \) to orbit-finite monoids.  The resulting
  space $\hatSigmas$ is still a nominal Hausdorff space with a basis
  of representables, but it generally fails to be compact, and its
  representable subsets do not correspond to recognizable data
  languages. To see this, consider the space $\widehat{\A^*}$ and the
  orbit-finite nominal monoids \( \A^{\le n} \) (words of length at
  most $n$) with multiplication cutting off after \( n \) letters.  We
  denote by \( h_n \colon \A^{*} \rightarrow \A^{\le n} \) and
  \( p_{k,n} \colon \A^{\le k} \epi \A^{\le n} \), $n \le k$, the
  equivariant monoid morphisms given by projection to the first $n$
  letters.  For every compatible family \( x=(x_{h}) \in \pfm[\A]{} \)
  its subfamily \( (x_{h_n})_{n \in \N} \) corresponds to a (possibly
  infinite) word over \( \A \) with finite support. Hence there exists
  a largest natural number $N=N(x)$ such that
  \( |\supp x_{h_N}| = N \). The subsets
  $C_{n} = \{x \in \pfm[\A]{} \mid N(x) = n \}$, $n\in \N$, are
  equivariant clopens since $C_n=\hat h_n^{-1}[\A^{\#n}] \cap \hat h_{n+1}^{-1}[\A^{\leq n+1}\setminus \A^{\#(n+1)}]$. Thus each $C_n$ is representable (by a continuous map into the
  two-element discrete space), non-empty (since $\eta(w)=(h(w))_h \in C_n$
  for every word $w \in \A^{\#n} \subseteq \fm[\A]$ of pairwise distinct
  letters), and pairwise disjoint. Hence they form a cover of $\widehat{\A^*}$
  that admits no orbit-finite (equivalently, finite) subcover, showing
  that $\widehat{\A^*}$ is not compact. Moreover, the
  sets $C_M=\bigcup_{m\in M} C_m$, where $M\seq \N$, are equivariant clopens (hence representable) and pairwise
  distinct. Thus $\widehat{\A^*}$ has uncountably many clopens. On
  the other hand, there exist only countably many recognizable
  languages over $\A$ (using that, up to isomorphism, there exist
  only countably many orbit-finite sets~\cite[Thm.~5.13]{pit-13} and
  thus countably many orbit-finite nominal monoids), showing that
  there is no bijective correspondence between representable sets in $\widehat{\A^*}$ and recognizable data languages over $\A$.
\end{rem}

\section[A Nominal Reiterman Theorem]{A Nominal Reiterman Theorem}
\label{sec:nom-reit-thm}

As an application of pro-orbit-finite methods, we present a nominal extension of Reiterman's classical pseudovariety theorem~\cite{rei-82}. The latter characterizes classes of finite algebras presentable by profinite equations as precisely those closed under finite products, subalgebras, and homomorphic images.
This result has been generalized to first-order structures~\cite{pin-wei-96} and, recently, to abstract categories~\cite{ada-che-mil-urb-21,mil-urb-19}. A key insight for the categorical perspective is that equations should be formed over projective objects. (Recall that an object $X$ in a category is projective w.r.t.\ a class \( \mathcal{E} \) of morphisms
if for all cospans \( \cramped[\textstyle]{X \xrightarrow{f} Y \overset{e}{\leftarrow} Z} \) with \(  e \in \mathcal{E} \) there exists a factorization of \( f \) through \( e \).)
In \nom, one takes strong nominal sets, which are projective with respect to support-reflecting quotients (see~\iref{def:epis}{supp-refl}).
For spaces of \pof words we have the support bound as an additional constraint, which makes the situation more complex: In a cospan \( \cramped{\spfm \xrightarrow{\hat{h}} N \overset{e}\twoheadleftarrow M} \) with $e$ support-reflecting, no $s$-bounded  factorization of \( \hat{h} \) through $e$ may exist (\autoref{ex:no-s-quot}). Surprisingly, there nonetheless exists a suitable type of quotients for nominal monoids, called \emph{\msr quotients}, which is \emph{independent} of the support bound \( s \).

\begin{defn}\label{def:epis}
  A surjective equivariant morphism \( e \colon M \epito N \) of nominal monoids is
  \begin{enumerate}
    \item\label{def:epis:supp-pres-n} \emph{support-preserving} if $\supp e(x) = \supp x$ for every $x\in X$;
    \item\label{def:epis:supp-refl} \emph{support-reflecting} if for every \( y \in Y \) there exists $x\in e^{-1}[y]$ such that $\supp x = \supp y$; 
    \item\label{def:epis:msr} \emph{multiplicatively support-reflecting} (\emph{\msr} for short) if there exists a nominal submonoid \( M' \subseteq M \) such that the domain restriction \( e|_{M'} \colon M'\to N \) of $e$ is surjective and support-preserving.
  \end{enumerate}
\end{defn}

\begin{rem}
  Note that a surjective morphism \( e \) is support-reflecting iff it
  restricts to a support-preserving surjection $e|_{M'}$ for some
  equivariant subset $M'\seq M$. For MSR morphisms one additionally
  requires that $M'$ may be chosen to form a submonoid. We thus have
  \[
    \text{support-preserving} \quad\To\quad \text{multiplicatively
      support-reflecting} \quad\To\quad \text{support-reflecting}.
  \]
  None of the two converses holds in general; for the first one consider the morphism $\A^*\epito 1$ into the trivial monoid, and for the second one see \autoref{ex:compare}.
\end{rem}

\begin{proposition}\label{prop:s-quot}
A surjective equivariant morphism $e\colon M\epito N$ between orbit-finite nominal monoids is \msr iff all the monoids $\spfm$ (where $\Sigma\in \nomof$ is strong and $s\colon \Sigma^*\to \pow \A$ is a support bound) are projective with respect to $e$ in $\mon(\npro)$, with $M$ and $N$ regarded as discrete spaces.
\end{proposition}
\begin{defn}\label{def:pse-var}
  An \emph{MSR-pseudovariety of nominal monoids} is a class
  \( \mathcal{V} \subseteq \nmonof \) of orbit-finite nominal monoids closed under
  \begin{enumerate}
  \item finite products: if $M_1,\ldots,M_n\in \V$, $n\in \N$, then
    $M_1\times \cdots\times M_n\in \V$;

  \item submonoids: if $M\in \V$ and $N\seq M$ is a nominal submonoid,
    then \( N \in \mathcal{V} \):

  \item \msr quotients: if \( M \in \mathcal{V} \) and
    \( e\colon M \epito N \) is an \msr quotient, then
    \( N \in \mathcal{V} \).
  \end{enumerate}
\end{defn}
\begin{defn}\label{def:pro-eq}
  Let \( s \colon \Sigma^{*} \rightarrow \pow \A \) be a support
  bound.  A \emph{morphic \pof equation}, or \emph{morphic
    proequation} for short, is a surjective $\npro$-morphism
  \( \varphi \colon \spfm \epi E \).  An orbit-finite monoid~$M$
  \emph{satisfies} \( \varphi \) if for every $s$-bounded morphism
  $h\colon \Sigma^*\to M$, the limit projection
  \( \hat h\colon \spfm \rightarrow M \) factorizes through
  \( \varphi \) in $\npro_k$, for some $k\in \N$ such that
  $M\in \nomofk$ and $s$ corestricts to $\pow_k\A$:
 \[
   \hat h = \big(
   \begin{tikzcd}[cramped, column sep=20]
     \spfm
     \ar[->>]{r}{\varphi}
     &
     E
     \ar[dashed]{r}{\exists}
     &
     M
   \end{tikzcd}
   \big).
    \]
    For a set \( \T \) of morphic proequations, taken over possibly
    different $\spfm$, we denote by $\V(\T)$ the class of orbit-finite
    monoids satisfying all proequations in \T. A class $\V$ of
    orbit-finite monoids is \emph{presentable by morphic proequations}
    if $\V=\V(\T)$ for some set $\T$ of morphic proequations.
\end{defn}

\noindent Note that proequations use support bounds, while the definition of an MSR-pseudovariety does not.

\begin{theorem}[Nominal Reiterman]\label{thm:nominal-reiterman}
  A class of orbit-finite nominal monoids is an MSR-pseudo\-variety
  iff it is presentable by morphic proequations.
\end{theorem}
The main technical observations for the proof are that (i) every
orbit-finite set is $k$-bounded for some~$k$, hence finitely
copresentable in $\npro_k$, and (ii) there are ``enough'' proequations
in the sense that every orbit-finite nominal monoid is a quotient of
some $\spfm$. The quotient is not necessarily MSR, which entails that
abstract pseudovariety theorems~\cite{mil-urb-19,ada-che-mil-urb-21}
do not apply to our present setting.

We also give a syntactic version of our nominal Reiterman theorem, which uses explicit proequations in lieu of morphic proequations.

\begin{defn}\label{def:expl-pro-eq}
  An \emph{explicit proequation} is a pair
  \( (x, y) \in \spfm \times \spfm \) for some strong
  $\Sigma\in\nomof$ and some support bound $s$, denoted by
  \( x = y \).  An orbit-finite monoid $M$ \emph{satisfies} the explicit proequation
  \( x = y \) if
  \[
    \hat{h} (x) = \hat{h} (y) \qquad \text{for every $s$-bounded equivariant
      monoid morphism $h\colon \Sigma^*\to M$}.
  \]
  (Here choose a common support size bound $k$ for $M$ and $s$, so that $\hat h$ lies in $\npro_k$.)
\end{defn}

\begin{theorem}[Explicit Nominal Reiterman]\label{thm:exp-nominal-reiterman}
  A class of orbit-finite nominal monoids is an MSR-pseudovariety iff
  it is presentable by explicit proequations.
\end{theorem}

\begin{example}\label{ex:proequation}
  Recall that in a finite monoid $M$ every element \( m \) has a
  unique idempotent power, denoted by \( m^{\omega} \). This holds
  analogously for orbit-finite nominal monoids $M$~\cite[Theorem
  5.1]{boj-13}: one has $m^\omega = m^{(n\cdot k!)!}$ where $n$ is the
  number of orbits $M$ and $k$ is the maximum support size. (The
  number $n\cdot k!$ is an upper bound on the number of elements
  of $M$ with any given finite support~\cite[Thm.~5.13]{pit-13}, hence
  on the cardinality of the set $\{m^i\colon i\in \N\}$.) The nominal monoid
  $M$ is \emph{aperiodic} if $m^\omega \cdot m = m^\omega$ for all
  $m\in M$. Languages recognizable by aperiodic orbit-finite monoids
  are captured precisely by first-order logic on data
  words~\cite{boj-13,col-15}. One readily verifies that the class of
  aperiodic orbit-finite monoids forms an MSR-pseudovariety; in fact,
  it is closed under all quotients. To present it by pro-orbit-finite
  equations, note that for every \( x\in \spfm \) the family
  \( x^\omega=(\hat h(x)^{\omega})_h \) is again compatible, hence
  $x^\omega\in \spfm$.  If $s\colon \Sigma^*\to \pow_k \A$ and
  \( h\colon \Sigma^*\to_s M\) is an $s$-bounded equivariant monoid
  morphism such that $M$ has at most \( n \) orbits, then
  $\hat{h}(x^{\omega}) = \hat{h}(x)^{\omega} = \hat{h}(x)^{(n\cdot
    k!)!} = \hat{h}(x^{(n\cdot k!)!})$, hence
  $\hat d_s(x^\omega,x^{(n\cdot k!)!})<2^{-n}$ in the metric
  \eqref{eq:metric-hat-d} on $\spfm$. This shows that \( x^{\omega} \)
  is the limit of the sequence \( (x^{(n\cdot k!)!})_{n\in \N} \) in
  \spfm, and moreover that the pseudovariety of aperiodic orbit-finite
  monoids is presented by the explicit proequations
  $x^{\omega} \cdot x = x^{\omega}$,
  where \( x \in \spfm \) and
  \( s \colon \fm \rightarrow \pow_{k} \A \) ranges over all support
  bounds on strong orbit-finite alphabets. Restricting to $k=0$, we
  recover the well-known description of aperiodic finite monoids by
  the (single) profinite equation $x^\omega\cdot x = x^\omega$.
\end{example}

\begin{rem}
\begin{enumerate}
\item Pseudovarieties of finite monoids admit an alternative
  equational characterization based on sequences of word equations
  rather than profinite equations. A \emph{word equation} is a pair
  $(v,w)\in \Sigma^*\times \Sigma^*$ of words over some finite
  alphabet $\Sigma$, denoted $v=w$; it is 
  \emph{satisfied} by a monoid $M$ if $h(v)=h(w)$ for every monoid
  morphism $h\colon \Sigma^*\to M$. More generally, a sequence
  $(v_0=w_0, v_1=w_1,\ldots)$ of word equations, taken over
  possibly different finite alphabets, is \emph{eventually satisfied}
  by $M$ if it satisfies all but finitely many of the
  equations. As shown by Eilenberg and Schützenberger~\cite{es76}, a
  class of finite monoids forms a pseudovariety iff it is presentable
  by a (single) sequence of word equations.
  
\item Urbat and Milius~\cite{mil-urb-19-06} recently established a
  nominal version of the Eilenberg-Schützenberger theorem. They
  consider nominal word equations (defined as above, where $\Sigma$ is
  now a strong orbit-finite nominal set) and show that sequences of
  nominal word equations present precisely \emph{weak
    pseudovarieties}, i.e.\ classes of orbit-finite nominal monoids
  closed under finite products, submonoids, and support-reflecting
  quotients. Clearly every MSR-pseudovariety is weak, but the converse
  does not hold; hence over nominal sets, sequences of word equations
  and pro-orbit-finite equations are of different expressivity. The
  example below illustrates one source of additional expressivity of
  pro-orbit-finite equations: The support bound \( s \) can control
  how the support changes during multiplication, which is not
  expressible by sequences of word equations.
\end{enumerate}
\end{rem}
\begin{example}\label{ex:compare}
  An example of an MSR-pseudovariety that is not a weak pseudovariety
  is given by the class \V of all orbit-finite nominal monoids \( M \)
  such that
  \begin{equation}
    \label{eq:ex-msr}
    \forall (m, n \in M) \colon \quad \supp (mn) = \emptyset \quad\iff\quad \supp (m,n) = \emptyset.
  \end{equation}
  (Note that $\supp(m,n)=\supp m\cup \supp n$ and that
  ``$\Leftarrow$'' always holds by equivariance of the monoid
  multiplication.) It is not difficult to prove that $\V$ is an
  MSR-pseudovariety.  To show that \V is not a weak pseudovariety, we
  construct a support-reflecting quotient under which \V is not
  closed.  The nominal set
  $\overline{1} + \overline{\A} = \{\overline{1}\} + \{\overline{a}
  \mid a \in \A \}$ forms a nominal monoid with multiplication given
  by projection on the first component and unit \( \overline{1} \).
  We extend the multiplication to the nominal set
  \( M = 1 + \A + \overline{1} + \overline{\A} \) by letting \( 1 \)
  be the unit and setting
  \( x \cdot y = \overline{x} \cdot \overline{y} \) whenever
  \( x, y \ne 1 \); here overlining is idempotent
  (\( \overline{\overline{x}} := \overline{x} \)).  This makes the
  multiplication associative and equivariant. Thus, \( M \) is a
  nominal monoid.  Now let \( N = 1 + \A + 0 = \{1\} + \A + \{0\}\) be
  the nominal monoid with multiplication \( x \cdot y = 0 \) for
  \(x, y \ne 1\). Thus $0$ is an absorbing element.  Letting
  $\mathrm{const}_0\colon \overline{1}+\overline{\A}\to 0$ denote the
  constant map, we have the equivariant surjective map
  \[
    e = \id_{1+\A} + \mathrm{const}_0 \colon M=(1 + \A) +
    (\overline{1} + \overline{\A}) \epi (1 + \A) + 0=N.
  \]
  Note that $e$ is a monoid morphism: it maps $1$ to $1$ and if
  \( x, y \ne 1 \) then \( e(x), e(y) \ne 1 \) and hence
  $e(x \cdot y) = e(\overline{x} \cdot \overline{y}) = 0 = e(x) \cdot
  e(y)$.  The quotient \( e \) is support-reflecting, but it is not
  \msr: the subset \( 1 + \A + \overline{1} \subseteq M \) of
  support-preserving elements does not form a submonoid of \( M \).
  Finally, clearly \( M \) satisfies \eqref{eq:ex-msr} while \( N \) does not.
\end{example}

\section[Conclusion and Future Work]{Conclusion and Future Work}
\label{sec:concl-future-work}

We have introduced topological methods to the theory of data
languages, and also explored some of their subtleties and
limitations. Following the spirit of Marshall Stone's slogan
\emph{``always topologize''}, the core insight of our paper may be
summarized as:
\begin{center}
\emph{Data languages topologize for bounded supports.}
\end{center}
In fact, by restricting to support-bounded orbit-finite nominal sets and analyzing their Pro-com\-ple\-tion, we have shown that fundamental results from profinite topology (notably Stone duality and the equivalence between profinite spaces and Stone spaces) generalize to the pro-orbit-finite world. These results are of independent interest; in particular, they are potentially applicable to data languages recognizable by all kinds of orbit-finite structures. For the case of monoids, we derived a topological interpretation of recognizable data languages via clopen sets of pro-orbit-finite words, as well as a nominal version of Reiterman's pseudovariety theorem characterizing the expressive power of pro-orbit-finite equations.

The foundations laid in the present paper open up a number of promising directions for future research. One first goal is to develop a fully fledged duality theory for data languages along the lines of the work of Gehrke et al.~\cite{ggp08} on classical regular languages, based on an \emph{extended nominal Stone duality} between pro-orbit-finite monoids and nominal boolean algebras with operators. 

Regarding specific applications, we aim to analyze further classes of orbit-finite monoids in terms of pro-orbit-finite equations, following the lines of \autoref{ex:proequation}, in order to classify the corresponding data languages. One natural candidate is the class of \( \mathcal{J} \)-trivial monoids, with the vision of a nominal version of Simon's theorem~\cite{simon75} relating \( \mathcal{J} \)-triviality to existential first-order logic on data words.

Finally, we aim to extend our topological theory of recognizable data languages, and the corresponding nominal Reiterman theorem, to algebraic structures beyond orbit-finite monoids. Potential instances include algebras for a signature $\Sigma$, which serve as recognizers for data tree languages, infinitary structures such as nominal $\omega$-semigroups~\cite{wilke91}, modeling languages of infinite data words, and algebraic structures with binders, which we expect to bear interesting connections to data languages with binders and their automata models~\cite{skmw17,uhms21}.

\bibliographystyle{plainurl}
\bibliography{bibliography}

\setcounter{section}{0}
\renewcommand\thesection{\Alph{section}}

\clearpage
\section[Appendix]{Appendix}
\label{sec:appendix}

This appendix provides full proofs of all results and technical statements omitted for space reasons.

\subsection*{Details for Example~\ref{ex:nom-top}.\ref{ex:nom-top:met}}

  We show that every nominal metric space carries a natural nominal topology with basic open sets given by the open balls.
  First we show that taking open balls is equivariant:
\begin{align*}
            \pi \cdot B_{r}(x) &= \{\pi \cdot y \mid d(x, y) < r, y \in X \} \\
                       &= \{\pi \cdot (\pi^{-1} \cdot z) \mid d(x, \pi^{-1} \cdot z) < r, z = \pi \cdot y \in \pi  \cdot X = X\} \\
                       &= \{z \mid \pi \cdot d(x, \pi^{-1} \cdot z) < r, z \in X \} \\
                       &= \{ z \mid d(\pi x, z) < r, z \in X \} \\
            &= B_{r}(\pi x).
\end{align*}
  Note that this implies that \( B_{r}x \) is \fs by \( \supp x \).
  It remains to show that the open balls form a basis.
  For two balls \( B_{r}x, B_{r'}x' \) the union \( \bigcup \{ B_{s}y \mid  B_{s}y \seq B_{r}x \cap B_{r'}x'\} \) is \fs by $\supp x\cup \supp x'$ and equal to \( B_{r}x \cap B_{r'}x' \).
\subsection*{Details for \autoref{rem:nomk}.\ref{rem:nomk:no-topos}}
  We prove that the category \nomk is not cartesian closed, and consequently not a topos.
  Recall that finite products in \( \nom_{k} \) are induced by the coreflector \( (-)_{k} \colon \nom \rightarrow \nom_{k} \) sending a nominal set to its nominal subset of \( k \)-bounded elements,
  so the product of \( X, Y \in \nom_{k} \) is given by
  \[ X \times_{k} Y := (X \times Y)_{k} = \{\,(x, y) \mid |\supp x \cup \supp y| \le k\,\}. \]
  It suffices to show that the functor \( X \times_{k} (-)  \) generally does not preserve coequalizers; this implies that it is not a left adjoint, whence $\nom_k$ is not cartesian closed.

  We set \( k=2 \) and $X=\A^{\#2}$. Consider the parallel pair
  \[ i_{0}, i_{1} \colon \A^{\#2} \rightrightarrows \A + \A
    \qquad\text{with}\qquad i_{l} = \iota_{l} \cdot \pi_{l},\] where
  $\pi_l\colon \A^{\#2}\to \A$ is the $l$-th projection and
  $\iota_l\colon \A\to \A+\A$ is the $l$-th coproduct injection; that
  is, \( i_{0}(a, b) = \iota_{0}(a)$ and  $i_{1}(a, b) = \iota_{1}(b) \).
  Since colimits in \( \nom_{k} \) are formed in \nom,
  the coequalizer of \( i_{0}, i_{1} \) is \( 1 \), i.e., the equivalence
  relation on \( \A + \A \) generated by
  \( i_{0}(a, b) \sim i_{1}(a, b) \) identifies everything. To see this,
  consider two elements \( \iota_{l}(a), \iota_{m}(b) \) of $\A+\A$.  They are
  clearly identified whenever \( l \ne m \) and \( a \ne b \).  If
  \( l = m \) and \( a = b \) they are identified by
  reflexivity. If \( l = m \) and \( a \ne b \) we choose a fresh
  \( c \in \A \) and \( n \ne l \) to get
  \[ \iota_{l}(a) \sim \iota_{n}(c) \sim \iota_{l}(b) = \iota_{m}(b). \]
  Finally, if \( l \ne m \) and \( a = b \) we choose fresh names \( c, d \in \A \) to get
  \[
    \iota_{l}(a) \sim \iota_{m}(c) \sim \iota_{l}(d) \sim \iota_{m}(b).
  \]
  The parallel pair $i_0,i_1$ is mapped under the functor \( \A^{\#2} \times_{2} (-) \) to
  \[
    j_{0}, j_{1} \colon \A^{\#2} \times_{2} \A^{\#2} \rightrightarrows
    \A^{\#2} \times_{2} (\A + \A) \qquad\text{where}\qquad j_{l} = \id_{\A^{\#2}} \times_{2} i_{l}.\]
  We observe that
  for every element
  \( (a,b, \iota_{l}(c)) \in \A^{\#2} \times_{2} (\A + \A) \) either
  \( c = a \) or \( c = b \) by \( 2 \)-boundedness.  Moreover, every
  such element has a unique preimage under \( j_{l} \) and no preimage under $j_m$ for \( m \ne l \). For example, the element \( (a, b, \iota_{0}(a)) \) has the
  preimage \( (a, b, a, b) \) under $j_0$ and no preimage under $j_1$.
  Therefore, the equivalence class of that element is the two-element set
  \( \{ (a, b, \iota_{0}(a)), (a, b, \iota_{1}(b))\} \). Similarly, the
  preimage of $(a,b,\iota_0(b))$ under $j_0$ is $(a,b,b,a)$ and its
  equivalence class is $\{(a,b,\iota_0(b)), (a,b,\iota_1(a))\}$.  In
  particular, \( (a, b, \iota_{0}(a)) \) and \( (a, b, \iota_{0}(b)) \)
  are not identified.  This shows that the
  coequalizer of $j_0,j_1$ is given by
  \[
    c \colon \A^{\#2} \times_{2} (\A + \A) \epi \A^{\#2} \times 2,
  \]
  defined by $(a, b, \iota_{0}(x)) \mapsto (a, b, \delta_{x = a})$
  and
  $(a, b, \iota_{1}(x)) \mapsto (a, b, \delta_{x = b})$,
  where $\delta_{x = a}(a) = 1$ and $\delta_{x = a}(x) = 0$ for $x \neq a$.
  This map is clearly not the image of \( \A + \A \epi 1 \) under \( \A^{\#2} \times_{2} (-) \), which proves that this functor
  does not preserve coequalizers.

\subsection*{Proof of \autoref{lem:cod-lim-set}}
  Let \( D \colon I \rightarrow \nomk \) be a codirected diagram,
  and let \( p_{i} \colon L \rightarrow |D_{i}| \) be the limit of the underlying diagram of sets,
  that is, \( L \) consists of all compatible tuples \( (x_{i} \in D_{i})_{i \in I} \).
  Suppose that for any tuple \((x_{i})_{i \in I}\) the union \( S = \bigcup_{i \in I} \supp x_{i} \)
  contains \( k+1 \) pairwise distinct names \( \{ a_{0},\ldots, a_{k}  \} \subseteq S\).
  Then every \( a_{i} \) lies in the support of some \( x_{t_{i}} \).
  By codirectedness there exists \( t \le t_{0}, \ldots, t_{k} \) in \( I \),
  then \( D_{t, t_{j}}(x_{t}) = x_{t_{j}} \) and therefore, for all $j = 0,\ldots,k$,
  \[ a_{j} \in \supp x_{t_{j} } = \supp D_{t, t_{j}}(x_{t_{j}} ) \subseteq \supp x_t. \]
  But this is a contradiction to our assumption that \( D_{t} \) is \kbnd.
  So \( S \) is a finite support of \( (x_{i})_{i \in I} \),
  and the limit of \( D \) in \nom hence consists of \emph{all} compatible families \( (x_{i} \in D_{i})_{i \in I}  \).
  Moreover, for any such family its least support is of size less or equal to \( k \).

\subsection*{Proof of \autoref{lem:comp}}

  \noindent\( \ref{lem:comp:def} \Rightarrow \ref{lem:comp:unif} \):
  Let \( \mathcal{C} \) be an open cover of \( X \) that is uniformly supported by some finite subset \( S \subseteq \A \).
  Then there exists an orbit-finite finitely supported subcover \( \mathcal{C}_{0} \subseteq \mathcal{C} \).
  All sets in \( \mathcal{C}_{0} \) are \( S \)-supported, but the orbit-finite set \( \mathcal{C}_{0} \) contains finitely many \( S \)-supported elements,
  proving that \( \mathcal{C}_{0} \) is finite.

  \noindent\( \ref{lem:comp:unif} \Rightarrow \ref{lem:comp:supp} \):
  By definition of \( \forget[X]{S} \).

  \noindent\( \ref{lem:comp:supp} \Rightarrow \ref{lem:comp:orb} \):
  The space \( \orb_{S} X \) is defined as a quotient of \( \forget[X]{S} \) and compactness of topological spaces is preserved under quotients.

  \noindent\( \ref{lem:comp:orb} \Rightarrow \ref{lem:comp:unif} \):
  Let \( \mathcal{C} = \{U_{i}\}_{i \in I} \) be an open cover of \( X \)  \ufs by \( S \subseteq \A \).
  Let \( V_{i} \) be the set of \( S \)-orbits contained in \( U_{i} \),
  then the sets \( \{V_{i}\}_{i \in I} \) cover \( \orb_{S} X \):
  If \( x \in U_{i} \) then the  \( S \)-orbit \( \orb_{S} x \) is a subset of \( U_{i} \) since the latter is $S$-supported,
  so \( \orb_{S} x \in V_{i} \).
  By compactness finitely many such \( V_{i} \) cover \( \orb_{S} X \),
  so the sets \( U_{i} = \bigcup V_{i} \) give a finite cover of \( X \).

  \noindent\( \ref{lem:comp:unif} \Rightarrow \ref{lem:comp:def} \):
  Let \( \mathcal{C} \) be a cover of \( X \) that is \fs by \( S \subseteq \A \).
  The set \( \mathcal{C}_{S} = \{\hull_{S} U \mid U \in \mathcal{C}\} \) then is a \ufs cover of \( X \),
  so finitely many \( \hull_{S}U_{1}, \dots , \hull_{S}U_{n} \) suffice to cover \( X \).
  The subset \( \mathcal{C}' = \{\pi \cdot U_{i} \mid \pi \in \perm_{S} \A, 1 \le i \le n\} \) therefore is an orbit-finite  subcover \( X \) finitely supported by \( S \).

\subsection*{Proof of \autoref{prop:cod-lim-ne}}

The proof of \autoref{prop:cod-lim-ne} requires some auxiliary statements.

\begin{lemma}\label{lem:cod-lim}
  Let \( D \colon I \rightarrow \ntop \) be a codirected diagram of nominal spaces with limiting cone \( p_{i} \colon L \rightarrow D_{i} \) for \(  i \in I \).
  \begin{enumerate}
    \item\label{lem:cod-lim:basis}
      The  open sets \( \{p_{i}^{-1}[U] \mid i \in I , U \subseteq D_{i} \text{ open } \} \) form a basis of $L$.
    \item\label{lem:cod-lim:forget}
      For every finite subset \( S \subseteq \A \) the set \( \{p_{i}^{-1}[U] \mid i \in I , U \subseteq \forget[ D_{i} ]{S} \text{ open  } \} \) forms a basis of $\forget[ L ]{S}$.
  \end{enumerate}
\end{lemma}

\begin{proof}
  \begin{enumerate}
    \item[\ref{lem:cod-lim:basis}.]
          The set is closed under intersection.
          Given \( i, j \in I \)  with lower bound \( k \le i, j \) and open sets \( U \subseteq D_{i}, V \subseteq D_{j} \),
          the intersection \( p_{i}^{-1}[U] \cap p_{j}^{-1}[V] \) is equal to \( p_{k}^{-1}[D_{ki}^{-1}[U] \cap D_{kj}^{-1}[V]] \) since the \( p_{i} \) form a cone.
    \item[\ref{lem:cod-lim:forget}.]
          We show that every \( U \subseteq L \) supported by \( S \subseteq \A \) is a union of \( S \)-supported basic open sets.
          By definition \( U \) is the union of all basic open sets \( p_{i}^{-1}[U_{i}] \) contained in \( U \).
          The \( S \)-hull \( \hull_{S}(p_{i}^{-1}[U_{i}]) \) of each of these sets is open, supported by \( S \) and satisfies
          \[p_{i}^{-1}[U_{i}] \subseteq \hull_{S}(p_{i}^{-1}[U_{i}]) \subseteq \hull_{S} U = U, \]
          hence \( U \) is also a union of the \( S \)-supported basic open sets \( \hull_{S}(p_{i}^{-1}[U_{i}]) = p_{i}^{-1}[\hull_{S}U_{i}]\). \qedhere
  \end{enumerate}
\end{proof}

\begin{lemma}\label{lem:cod-lim-comp}
  Limits of codirected bounded diagrams of compact spaces are compact.
\end{lemma}
\begin{proof}
  Let \( D \colon I \rightarrow \ntop \) be a codirected diagram of \kbnd compact spaces, and let $\left|-\right|\colon \ntop\to\set$ denote the forgetful functor.
  From \autoref{lem:cod-lim-set} we get \( |\! \lim D| \cong \lim |D| \),
  and thus by  \iref{lem:cod-lim}{forget}  also \( | \! \lim \forget[D]{S} | \cong \forget[\lim D]{S} \) as topological spaces.
  All spaces \( \forget[D]{S} \) are compact by \iref{lem:comp}{supp},
  and limits of compact spaces are compact, so every \( \forget[\lim D]{S} \cong \lim \forget[D]{S} \) is compact.
  Using \autoref{lem:comp} again, we get that \( \lim D \) is nominally compact.
\end{proof}

\begin{lemma}\label{lem:haus}
  For every nominal topological space \( X \) the following statements are equivalent:
  \begin{enumerate}
    \item\label{lem:haus:haus} The space \( X \) is Hausdorff.
    \item\label{lem:haus:orb} For every finite \( S \subseteq \A \) the topological space \( \orb_{S} X \) is Hausdorff.
  \end{enumerate}
\end{lemma}
\begin{proof}
  By definition of the topology of \( \orb_{S}X \).
\end{proof}

\begin{lemma}\label{lem:cod-lim-haus}
  Codirected limits of nominal Hausdorff spaces are nominal Hausdorff.
\end{lemma}
\begin{proof}
  Let \( D \colon I \rightarrow \ntop \) be a codirected diagram of nominal Hausdorff spaces and take its limit \( p_{i} \colon L \rightarrow D_{i} \).
  We fix a finite subset \( S \subseteq \A \) and prove that every pair of distinct orbits \( \orb_{S} x, \orb_{S} y \subseteq L \) with \( x = (x_{i})_{i \in I}, y = (y_{i})_{i \in I} \) has disjoint \( S \)-supported open neighbourhoods.
  To this end, it suffices to show that there exists an \(i \in I\) such that \( x_{i} \) and \( y_{i} \) lie in different \( S \)-orbits:
  then there exist, because \( D_{i} \) is Hausdorff,
  disjoint open \( S \)-supported neighbourhoods \( U_{1}, U_{2} \subseteq D_{i} \) of \( \orb_{S} x_{i}, \orb_{S} y_{i} \) whose preimages under $p_i$ are the desired disjoint \( S \)-supported open neighbourhoods of \( \orb_{S} x, \orb_{S} y \).
  Suppose the contrary, i.e., that \( x_{i}, y_{i} \) lie in the same \( S \)-orbit for all \( i \in I \).
  Choose \( T \subseteq \A \) finite and disjoint from \( S \) such that \( S \cup T \) supports both \( x, y \).
  We enumerate all permutations~\( \pi_{0}, \ldots, \pi_{n} \) supported by \( T \) and consider for each \( \pi_{k} \) the set
  \[ I_{k} = \{ i \in I \mid x_{i} = \pi_{k} \cdot y_{i} \} \subseteq I. \]
  Then \( I = \bigcup_{k=0}^{n} I_{k} \):
  For each \( i \in I \) there exists some \( \pi \in \perm_{S} \A \) such that \( x_{i} = \pi \cdot y_{i} \),
  and since \( x_{i}, y_{i}  \) are supported by \( S \cup T \) we may choose \( \pi \) such that it fixes all elements in \( \A \setminus T \).
  We now show that some $I_k$ forms a cofinal subset of the partial order $I$, that is, each element of $I$ has a lower bound in $I_k$.
Suppose the contrary. Then, for each $k\in \{0,\ldots,n\}$ there exists $i_k\in I$ such that no $i\leq i_k$ lies in $I_k$. Since $I$ is codirected, there exists a lower bound $i\leq i_0,\ldots, i_n$. Thus $i\not\in I_k$ for any $k$, a contradiction.

Thus let $I_k$ be cofinal. Then for each \( i \in I \) there exists \( j \in I_{k} \) such that \( j \le i \),
  and so we compute
  \[ x_{i} = D_{ji}(x_{j}) = D_{ji}(\pi_{k} \cdot y_{j}) = \pi_{k} \cdot D_{ji}(y_{j}) = \pi_{k}\cdot y_{i}\]
  for all \( i \in I \).
  This contradicts our assumption of \( x, y \) lying in different \( S \)-orbits of \( L \).
\end{proof}

We are now ready to prove \autoref{prop:cod-lim-ne}.

  Let \( D \colon I \rightarrow \ntop \) be a \kbnd non-empty diagram of nominal compact Hausdorff spaces.
  Combining \autoref{lem:haus}~(\(\ref{lem:haus:haus} \Rightarrow \ref{lem:haus:orb}\)) and \iref{lem:comp}{orb}
  we get that the codirected diagram \( \orb_{\emptyset} \cdot D \colon I \rightarrow \ntop \rightarrow \topo \) of non-empty compact Hausdorff spaces,
  hence its limit \( L \) is non-empty~\cite[Prop.~1.1.4]{rz10}.
  This limit contains a family of orbits \( (O_{i} \subseteq D_{i })_{i \in I} \) that is compatible, i.e.,
  \( D_{ij}[O_{i}] = O_{j} \) for all \( i \le j \) in \(I \).
  Now fix an arbitrary set \( S \subseteq \A \) of size \( k \). Then the set
  \[ O_{i}^{S} = \{ x \in O_{i} \mid x \text{ is supported by } S\} \]
  is non-empty:
  If \( x \in O_{i} \) has support \( T \) then \( |T| \le k  \) by boundedness of \( D_{i} \), and  \( \pi \cdot x \) has support \( \pi \cdot T = S \) for any permutation \(\pi\) mapping \( T \) to \( S \).
  The sets \( O_{i}^{S} \) thus give a codirected diagram of non-empty finite sets,
  so their limit in \set is itself non-empty.
  Any element of this limit is supported by \( S \) and hence an element in \( \lim D \).
  Moreover, the limit is a nominal compact Hausdorff space by \autoref{lem:cod-lim-haus} and \autoref{lem:cod-lim-comp}.

\subsection*{Proof of \autoref{lem:rep:lim}}
Let $C\seq X$ be clopen and let $S=\supp C$. Then \( C \) is equal to the \ufs union of all \( S \)-supported basic representables \( p_{i}^{-1}[U_{i}] \) it contains (\iref{lem:cod-lim}{forget}).
          Since \( C \) is compact, being a closed subspace of the compact space \( X \),  finitely many representables \( p_{i}^{-1}[U_{i}] \) suffice to cover \( C \).
          Take a lower bound \( j \) of all these \( i \); then $C=p_j^{-1}[U]$ where $U=\bigcup_i D_{ji}^{-1}[U_{i}]$.

\subsection*{Proof of \autoref{thm:pro-comp}}

To prove that \( \npro_{k} \) is the Pro-completion of \( \nomofk \), we make use of the following characterization of Pro-completions of small categories.
\begin{lemma}[\!\!{\cite[Corollary~A.5]{ada-che-mil-urb-21}}]\label{lem:pro-comp-char}
  If \cat{C} is small, then \( \pro(\cat C) \) is characterized, up to equivalence, as a category \( \cat{P} \) containing \( \cat C \) as a full subcategory such that
  \begin{enumerate}
    \item the category \label{lem:pro-comp-char:has-lim} \cat{P} has codirected limits,
    \item\label{lem:pro-comp-char:is-lim} every object in \cat P is a codirected limit of objects in \cat C, and
    \item\label{lem:pro-comp-char:fin-copres} every object in \cat C is finitely copresentable in \cat P.
  \end{enumerate}
\end{lemma}

Our task is thus to show that $\npro_k$ satisfies the above conditions. Towards the proof of the second condition, we introduce the canonical diagram of a space:

\begin{construction}\label{con:canon-diag}
  To every \kbnd nominal Stone space \( X \) we can construct a \emph{canonical diagram}:
  The small {category} \slice{X}{\nomofk} has as objects continuous equivariant functions \( f \colon X \rightarrow Y_{f} \) with \( Y_{f} \in \nomofk \) (regarded as a discrete space),
  and \( h \colon Y_{f} \longrightarrow Y_{f'} \) is a morphism from \( f \colon X \rightarrow Y_{f} \) to \( f' \colon X \rightarrow Y_{f'} \) if \( hf = f' \).
  The projection $f\mapsto Y_f$ gives a diagram in $\npro_k$, which by abuse of notation we also denote \slice{X}{\nomofk}.
\end{construction}

\begin{lemma}\label{lem:prof-canon-diag}
Every \kbnd nominal Stone space is the limit of its canonical diagram, with limit cone $f\colon X\to Y_f$ ($f\in \slice{X}{\nomofk}$).
\end{lemma}
\begin{proof}
  Let \( X \) be a \kbnd nominal Stone space.
  To prove that  \( X \) is isomorphic to the limit of its canonical diagram \( \slice{X}{\nomofk} \) we verify three conditions, namely that
  (1) every pair of distinct points of \( X \) can be separated by some morphism in \( \slice{X}{\nomofk} \);
  (2) for every finitely supported compatible family \( (y_{f})\) where \( f \) ranges over \(\slice{X}{\nomofk} \) there exists some \( x \in X \) with \( f(x) = y_{f} \) for all \( f \).
  Conditions (1) and (2) ensure that the induced continuous equivariant map
  \[i\colon X \rightarrow \lim (\slice{X}{\nomofk}) \qquad x \mapsto (x_{f})_{f \in (\slice{X}{\nomofk})} \]
  is a bijection (they are responsible for injectivity and surjectivity, respectively). Moreover, since the representables $f^{-1}[y]$ ($f\in \slice{X}{\nomofk}$, $y\in Y_f$) form a subbase of $X$, the inverse $i^{-1}$ is also continuous, hence $i$ is a homeomorphism.
  \begin{enumerate}
    \item\label{pr:prof:1}
         Let $x,x'\in X$ be distinct points. 
         The space \( X \) is Hausdorff, whence for \( x \) we get a basic clopen neighbourhood not containing \( x' \);
          this neighbourhood is representable by \( f \colon X \rightarrow Y_{f} \),
          and this function clearly separates \( x, x' \).
    \item\label{pr:prof:2}
          We need to prove that the \fs intersection \( \bigcap_{f} f^{-1}[y_{f}] \) is non-empty.
          Suppose the contrary. Then the open sets \( U_{f} = X \setminus f^{-1}[y_{f}] \) form a cover \( \mathcal{C} \) of \( X \), uniformly finitely supported by any finite support $S$ of the family $(y_f)$.
          By compactness of \( X \) finitely many such sets \( U_{f_{1}}, \ldots, U_{f_{n}} \subseteq \mathcal{C} \) suffice to cover \( X \);
          hence \( \bigcap_{i=1}^{n} f^{-1}[y_{f_{i}}] \) is empty.
          Let \( f \colon X \epi Y_{f} \) be the subdirect product of the \( f_{i}, i=1,\dots,n \), i.e.,
          the coimage of the continuous equivariant map
          \[ \langle f_{1}, \ldots, f_{n} \rangle \colon X \rightarrow Y_{f_{1}} \times \cdots \times Y_{f_{n}}. \]
          Note that \( Y_{f} \) is \kbnd since \( X \) is \kbnd and \( f \) is surjective.
          By surjectivity of \( f \) there exists an \( x \in X \)  such that \( f(x) = y_{f} \);
          but then \( x \in \bigcap_{i=1}^{n} f^{-1}[y_{f_{i}}] \), a contradiction.\qedhere
  \end{enumerate}
\end{proof}

\begin{construction}\label{con:of-diag}
  For every codirected diagram \( D \colon I \rightarrow \npro_{k} \) one can construct a diagram \( D' \colon I' \rightarrow \nomofk \) of orbit-finite discrete spaces:
  Let \( I' \) be the category with objects
  \[ (i, f) \text{ with } i \in I, f \in (\slice{D_{i}}{\nomofk}) \]
  with arrows \( (i, f) \rightarrow (j, g) \) those \(  h \colon Y_{f} \rightarrow Y_{g} \) making the following square commute:
\[
    \begin{tikzcd}
      D_i
      \arrow[r, "D_{ij}"]
      \arrow[d, "f", ]
      &
      D_j
      \arrow[d, "g", ]
      \\
      Y_f
      \arrow[r, "h", dashed]
      &
      Y_g
    \end{tikzcd}
  \]
\end{construction}

\begin{lemma}\label{lem:of-diag}
  Let \( D \colon I \rightarrow \npro_{k} \) be a codirected diagram.
  Then the diagram from \autoref{con:of-diag} is cofiltered,
  and if \( (p_{i} \colon \lim D \rightarrow D_{i})_{i \in I} \) is a limiting cone for \( D \),
  then \( (f \cdot p_{i} \colon \lim D \rightarrow D_{i} \rightarrow Y_{f})_{(i, f) \in I'} \) is a limiting cone for \( D' \).
\end{lemma}
\begin{proof}[Proof of \autoref{lem:of-diag}.]
  We prove that \( I' \) is cofiltered.
  Firstly, given objects \( (i, f), (j, g) \in I' \) there exists by codirectedness of \( I \) a lower bound \( k \le i, j \) in \( I \).
  The coimage \( h \colon  D_{k} \epi Y_{h} \) of \( \langle f \cdot D_{ki}, g \cdot  D_{kj} \rangle \) has a \kbnd codomain
  and the projections \( p_{f} \colon Y_{h} \rightarrow Y_{f}, p_{g} \colon Y_{h} \rightarrow Y_{g} \) are morphisms in \( I' \).
  This makes \( (k, h) \) a span over \( (i, f) \) and \( (j, g) \) in \( I' \).
  Secondly, assume that \( h, h' \colon (i, f) \rightarrow (j, g) \) are parallel arrows in \( I' \).
  This in particular implies \( h \cdot f = h' \cdot f \),
  so by forming their equalizer \( e \colon E \mono Y_{f} \) in \nomofk its universal property yields a morphism \( \bar{f} \colon D_{i} \rightarrow E \) with \( e \cdot \bar{f} = f \). Note that $\bar f$ is continuous, as it is just a corestriction of the continuous map $f$.
  This is precisely the statement that \( (i, \bar{f}) \) equalizes the morphisms \( h, h' \) in \( I' \). 

  Now we prove the rest of the statement. If \( (p_{i} \colon \lim D \rightarrow D_{i})_{i \in I} \) is a limiting cone for \( D \)
  then the family \( (f \cdot p_{i} \colon \lim D \rightarrow D_{i} \rightarrow Y_{f})_{(i, f) \in D'} \) is clearly a cone for \( D' \), we just have to show it is universal.
  Suppose \( (a_{(i, f)} \colon A \rightarrow Y_{f})_{(i, f) \in I'} \) is another cone for \( D' \).
  If we restrict this cone to a fixed \( i \in I \) we obtain a cone over the canonical diagram of \( D_{i} \).
  By \autoref{lem:prof-canon-diag} every \( D_{i} \) is the limit of its canonical diagram,
  so we get a unique \( a_{i} \colon A \rightarrow D_{i} \) satisfying \( a_{(i, f)} = f \cdot a_{i} \) for \( (i, f) \in I' \).
  The family of all \( (a_{i})_{i \in I} \) is a cone over \( D \):
  For \( i \le j \) and every \( (j, g) \in I' \) we get the following diagram:
  \begin{equation}\label{diag:cone}
    \begin{tikzcd}[column sep = large]
      D_i
      \arrow[rr, "D_{ij}", ]
      \arrow[dd, "g \cdot D_{ij}", ]
      &
      &
      D_j
      \arrow[dd, "g", ]
      \\
      &
      A
      \arrow[lu, "a_i", ]
      \arrow[ru, "a_j", ]
      \arrow[ld, "a_{(i, g \cdot D_{ij})}", ]
      \arrow[rd, "a_{(j, g)}", ]
      &
      \\
      Y_{g \cdot D_{ij}}
      \arrow[rr, "\id"', ]
      &
      &
      Y_g
    \end{tikzcd}
  \end{equation}
The outer square and all but the upper triangle commute.
The morphisms \( (j, g) \) for fixed \( j \) form the limiting cone of \( D_{j} \) and thus are jointly monic,
so the upper triangle also commutes for all \( i \le j \).
Since the family \( (a_{i})_{i \in i} \) is a cone for \( D \)
we get a unique \( \alpha \colon A \rightarrow \lim D \) with \( a_{i} = p_{i} \cdot \alpha \),
postcomposition yields \( f \cdot a_{i} = (f \cdot p_{i}) \cdot \alpha \) for all \( (i, f) \in I' \).
Regarding uniqueness, let \( f \cdot p_{i} \cdot \alpha = f \cdot p_{i} \cdot \beta \) for all \( (i, f) \in I' \).
First for every \( i \in I \) we get \( p_{i} \cdot \alpha = p_{i} \cdot \beta \) as the \( (i, f) \in I' \) are jointly monic,
and second also \( \alpha = \beta \) since the \( p_{i} \) are jointly monic.
\end{proof}

\begin{theorem}\label{thm:pro-nom-equiv}
  For any \kbnd nominal space \( X \) the following statements are equivalent:
  \begin{enumerate}
    \item\label{thm:pro-nom-equiv:pro} \( X \) is a nominal Stone space.
    \item\label{thm:pro-nom-equiv:cof-canon} \( X \) is the codirected limit of its canonical diagram.
    \item\label{thm:pro-nom-equiv:cod-diag} \( X \) is a codirected limit of \kbnd orbit-finite discrete spaces.
  \end{enumerate}
\end{theorem}
\begin{proof}[Proof of \autoref{thm:pro-nom-equiv}.]
  \( \ref{thm:pro-nom-equiv:pro} \Rightarrow \ref{thm:pro-nom-equiv:cof-canon} \): \autoref{lem:prof-canon-diag}.

  \( \ref{thm:pro-nom-equiv:cof-canon} \Rightarrow \ref{thm:pro-nom-equiv:cod-diag} \):
  This holds trivially.

  \( \ref{thm:pro-nom-equiv:cod-diag} \Rightarrow \ref{thm:pro-nom-equiv:pro} \):
  If \( X \) is a codirected limit of \kbnd orbit-finite discrete spaces then it is nominal compact and Hausdorff by Lemma~\ref{lem:cod-lim-comp} and~\ref{lem:cod-lim-haus}.
  It has a basis of representable sets by \iref{lem:cod-lim}{basis}. \qedhere
\end{proof}

The following result critically depends on non-emptiness of codirected limits of \kbnd compact Hausdorff codirected diagrams.

\begin{lemma}\label{lem:surj-proj}
  Let \( D \colon I \rightarrow \nomofk \) non-empty codirected diagram with limit cone \( p_{i} \colon L \rightarrow D_{i} \).
  Then there exists for every \( i \in I \) some \( j \le i \) with
  \[ D_{ji}[D_{j}] = p_{i}[\lim D].\]
\end{lemma}
\begin{proof}[Proof of \autoref{lem:surj-proj}.]
  We fix \( i \in I \).
  Note that for all \( j \le i \) already \( p_{i}[\lim D] \subseteq D_{ji}[D_{j}] \) since the \( (p_{i})_{i \in I} \) form a cone,
  hence we only need to prove the reverse inclusion.
  Suppose no \( j \in I \) satisfies \( D_{ji}[D_{j}] \subseteq p_{i}[\lim D] \),
  then for all \( j \in I^{\le i} \) the sets \( D_{j}' = D_{ji}^{-1}[D_{i} \setminus p_{i}[\lim D]] \subseteq D_{j} \) are non-empty
  (here \( I^{\le i} \) denotes the the downset \( \{j \in I \mid j \le i\} \) of \( i \) in \( I \)).
  The restriction \( D^{\le i} \colon I^{\le i} \mono I \rightarrow \nomofk \) of \( D \) to \( I^{\le i} \) is a non-empty codirected diagram in \nomofk,
  of which \( D' \) is a non-empty subdiagram.
  Its limit \( \lim D' \) is a non-empty nominal space by~\autoref{prop:cod-lim-ne},
  containing a compatible family \( (x_{j})_{j \le i} \in \lim D' \).
  By initiality of the inclusion functor \( I^{\le i} \mono I\) the limits of \( D^{\le i} \) and \( D \) are isomorphic,
  hence the family \( (x_{j})_{j \le i} \) extends by \( \lim D' \subseteq \lim D^{\le i} \cong \lim D \) to an element \( (x_{i})_{i \in I} \in \lim D \).
  But this element satisfies both \( x_{i} = p_{i}(x) \in p_{i}[\lim D] \) and \( x_{i} \in D_{i}' = X \setminus p_{i}[\lim D] \), a contradiction.
\end{proof}

\begin{lemma}\label{lem:of-copres}
  Every \kbnd orbit-finite nominal set is finitely copresentable in \( \npro_{k} \).
\end{lemma}
\begin{proof}
  Let \( Y \) be a $k$-bounded orbit-finite set, w.l.o.g.\ $Y\neq \emptyset$, and let \( D \colon I \rightarrow \npro_{k} \) be a codirected diagram with limiting cone \( (p_{i} \colon \lim D \rightarrow D_{i})_{ i \in I } \).
  We show that every continuous, equivariant function \( f \colon \lim D \rightarrow Y \) factors essentially uniquely through some \( p_{i} \).
  \begin{enumerate}

    \item[(1)]\label{lem:of-copres:(1)}
          We first assume that all \( D_{i} \) are orbit-finite discrete spaces.

    \item[(1.i)]\label{lem:of-copres:(1.i)}
          We prove that there exists some \( i \in I \) and a continuous equivariant function \( g \) with \( f = g \cdot p_{i} \).
          We may assume that \( f \) is surjective;
          otherwise we replace \( f \) by its coimage.
          First, we choose representatives \( y_{i}, \ldots, y_{m} \) of the orbits of \( Y \) with (representable!) preimages \( Y_{i} = f^{-1}[y_{i}] \).
          Note that the sets \( \pi \cdot Y_{i}, \pi \in \perm \A, \) form a partition of \( L = \lim D \).
          Let \( S \subseteq \A \) be a finite set supporting all \( y_{i}, i = 1, \ldots, m \),
          then \( S \) also supports all sets \( Y_{i} \) due to equivariance of \( f \).

          We first show that there exists some \( j \in I \) such  that
          \begin{equation}
            \label{eq:Y_i}
            Y_{i} = p_{j}^{-1}[p_{j}[Y_i]] \quad \text{ for all } i = 1, \ldots, m.
          \end{equation}
          Fix \( i \in \{1, \ldots, m\} \).
          The set \( Y_{i} \) is a clopen (and hence compact) subset of \( \forget[L]{S} \).
          Recall that by \iref{lem:cod-lim}{forget} the space \( \forget[L]{S}  \) has a basis given by the sets \( p_{l}^{-1}[U] \) where \( l \in I \) and \( U \subseteq D_{l} \) is an \( S \)-supported set.
          So we can present \( Y_{i} \) as a finite union of basic open sets
          \begin{equation}
            Y_{i} = \bigcup_{\alpha=1}^{n}p_{j_{\alpha}}^{-1}[U_{\alpha}],\label{eq:Y_i:union}
          \end{equation}
          with \( j_{\alpha} \in I \) and such that all \( U_{\alpha} \subseteq D_{j_{\alpha}} \) are \( S \)-supported.
          The set \( I \) is codirected so there exists a lower bound \( j \in I \) of all \( j_{\alpha}, \alpha = 1, \ldots, n \).
          By putting
          \[ Y_{i}'  = \bigcup_{\alpha=1}^{n}D_{j, j_{\alpha}}^{-1}[U_{\alpha}] \]
          we compute
          \[ Y_{i} = \bigcup_{\alpha=1}^{n} p_{j_{\alpha}}^{-1}[U_{\alpha}] = \bigcup_{\alpha=1}^{n} p_{j}^{-1}[D_{j, j_{\alpha}}^{-1}[U_{\alpha}]] = p_{j}^{-1}[Y_{i}']. \]
          Using codirectedness of \( I \) again allows us to choose \( j \) independently of \( i = 1, \ldots, m \) to obtain
          \[ Y_{i} = p_{j}^{-1}[Y_{i}'] \quad \text{ for all } i = 1, \ldots, m.\]
          This proves~\eqref{eq:Y_i}.

          This enables us to define a function \( g' \colon D_{j} \rightarrow Y \) with \( f = g' \cdot p_{j} \) as follows:
          for \( x \in D_{j} \),
          if there exists \( \pi \in \perm \A \) such that \( x \in p_{j}[\pi \cdot Y_{i}] = \pi \cdot p_{j}[Y_{i}] \) then we define
          \[ g'(x) := \pi \cdot y_{i};\]
          otherwise we set \( g'(x) := y_{1} \).
          To show that \( g' \) is well-defined let us assume some \( x \) lies in \( p_{j}[\pi \cdot Y_{i}] \cap p_{j}[\pi' \cdot Y_{i'}] \).
          By \eqref{eq:Y_i} and equivariance, we have $p^{-1}[x]\seq \pi\cdot Y_i \cap \pi'\cdot Y_{i'}$, hence the latter intersection is non-empty.
          But the sets \( \sigma \cdot Y_{l} \) form a partition, so \( \pi \cdot Y_{i} = \pi' \cdot Y_{i'} \) and thus $\pi\cdot y_i = \pi'\cdot y_{i'}$.
          This proves that the choice of \( g'(x) \) does not depend on neither \( \pi \) nor \( i \).
          We verify the factorization \( f = g'p_{j} \) pointwise:
          Take \( z \in L \); it is mapped by \( f \) to \( f(z) = \pi \cdot y_{i} \) for some \( \pi \in \perm \A\) and some $i\in \{1,\ldots, m\}$.
          This implies
          \[ p_{j}(z) \in p_{j}(f^{-1}[f(z)]) = p_{j}(f^{-1}[\pi \cdot y_{i}]) = p_{j}(\pi \cdot Y_{i}), \]
          now the definition of \( g' \) yields \( g'(p_{j}(z)) = \pi \cdot y_{i} = f(z) \).

          However,
          the function \( g' \) just defined need not be equivariant.
          But by~\autoref{lem:surj-proj} there exists \( i \le j \) with \( p_{j}[L] = D_{ij}[D_i] \);
          set \( g := g' \cdot D_{ij} \colon D_i \rightarrow Y \).
          Then \( f \) also factorizes as \( f = g' \cdot p_{j} = g' \cdot D_{ij} \cdot p_{i} = g \cdot p_{j} \),
          and moreover, for \( g \) we are able to prove equivariance: let \( x \in D_{i} \). Since \( p_{j}[L] = D_{ij}[D_i] \), the element $D_{ij}(x)$ has a preimage \( y \in L \) under \( p_{j} \). Then, for every \( \pi \in \perm \A \),
          \begin{align*}
            g(\pi \cdot x) &= g'(D_{ij}(\pi \cdot x)) && \text{def. $g$} \\
                     &= g'(\pi \cdot D_{ij}(x)) && \text{$D_{ij}$ equiv.} \\
                     &= g'(\pi \cdot p_{j}(y)) && \text{def. $y$} \\
                     &= g'(p_{j}( \pi \cdot y)) && \text{$p_{j}$ equiv.} \\
                     &= f(\pi \cdot y) && \text{$f = g' \cdot p_{j}$} \\
                     &= \pi \cdot f(y) && \text{$f$ equivariant} \\
                     &= \pi \cdot g'(p_{j}(y)) = \pi \cdot g'(D_{ij}(x)) = \pi \cdot g(x). && \text{(steps backwards)}
          \end{align*}
          Hence, \( g \) is equivariant.

    \item[(1.ii)]\label{lem:of-copres:(1.ii)}
          We show that this factorization is essentially unique, that is,
          if for some \( i \in I \) there are two equivariant maps \( g, h \colon D_{i} \rightarrow Y \) with \( g \cdot p_{i} = h \cdot p_{i} \),
          then they are equalised by some \( D_{ji} \) for \( j \le i \).
          We choose \( j \) as provided by~\autoref{lem:surj-proj},
          then the maps \( g, h \) coincide on \( p_{i}[L] = D_{ji}[D_{j}] \),
          hence \( g \cdot D_{ji} = h \cdot D_{ji} \).

    \item[(2)]\label{lem:of-copres:(2)}
          Now let \( D \) be an arbitrary codirected diagram in \( \npro_{k} \).
          We construct the corresponding diagram \( D' \) with scheme \( I' \) (see \autoref{con:canon-diag});
          it has the limiting cone
          \[ (h \cdot p_{i} \colon L \rightarrow D_{i} \rightarrow Y_{h})_{(i, h) \in I'} \]
          by \autoref{lem:of-diag}.
    \item[(2.i)]\label{lem:of-copres:(2.i)}
          Since \( D' \) is a diagram of orbit-finite spaces,
          the map \( f \) factorizes by Item~(1.i) as \( f = g \cdot h \cdot p_{i} \), where
          \[ h \cdot p_{i } \colon \lim D' \cong \lim D = L \rightarrow D_{i} \rightarrow Y_{h}\]
          for some \( (i, h) \in I' \).
          This immediately yields the desired factorization of \( f \) through \( p_{i} \) via \( g \cdot h \).

    \item[(2.ii)]
          Let \( g, h \colon D_{i} \rightarrow Y  \) be two equivariant, continuous maps satisfying \( g \cdot p_{i} = h \cdot p_{i} \).
          The space \( D_{i} \) lies in \( \npro_{k} \), hence it is by \iref{thm:pro-nom-equiv}{cof-canon} the codirected limit of its canonical diagram \( \slice{D_{i}}{\nomofk} \).
          We thus can use Item~(1.i) above to get factorizations of both maps \( g, h \) through \( g', h' \colon Y_{q} \rightarrow Y \) for some \( q \colon D_{i} \rightarrow Y_{q} \),
          i.e., \( g = g' \cdot q \) and \( h = h' \cdot q \). Note that we may choose the index to be the same for \( g, h \) by codirectedness of \( D \),
          so in the bottom right triangle of diagram \eqref{diag:pres} both the lower and the upper paths commute.
          Precomposition with \( p_{i} \) yields \( g' \cdot q \cdot p_{i} = h' \cdot q \cdot p_{i} \).
          The morphism \( q \cdot p_{i} \) is a limit projection of \( L = \lim D' \),
          hence there exists, by Item~(1.ii),
          some index \( (j, r \colon D_{j}\rightarrow Y_{r}) \in I' \) and a morphism \( k \colon Y_{r} \rightarrow Y_{q} \) in \( I' \), i.e.,
          \( k \cdot r = q \cdot D_{ji} \), that equalizes \( g' \) and \( h' \): \( g' \cdot k = h' \cdot k \).
          A diagram chase now shows that the connecting morphism \( D_{ji} \) equalizes \( g \) and \( h \), concluding the proof. \qedhere
          \begin{equation}\label{diag:pres}
            \begin{tikzcd}[column sep=large, row sep=large]
              &
              L
              \arrow[ld, "p_j"', ]
              \arrow[d, "p_i", ]
              \arrow[rd, "g \cdot p_i = h \cdot p_i", ]
              &
              \\
              D_j
              \arrow[d, "r", ]
              \arrow[r, "D_{ji}", dashed]
              &
              D_i
              \arrow[r, "g", shift left=1]
              \arrow[r, "h"', shift left=-1]
              \arrow[d, "q", ]
              &
              Y
              \\
              Y_r
              \arrow[r, "k", ]
              &
              Y_q
              \arrow[ru, "g'", shift left=1]
              \arrow[ru, "h'"', shift left=-1]
              &
            \end{tikzcd}
          \end{equation}
\end{enumerate}
\end{proof}

Collecting all the above results, we obtain the

\begin{proof}[Proof of \autoref{thm:ind-compl}]
We verify the three conditions of \autoref{lem:ind-comp-char} for $\npro_k$.
\begin{enumerate}
\item The category $\npro_k$ has codirected limits, as it is closed under codirected limits in $\ntop$ (where bounded codirected limits exist and are formed at the level of $\nom$ and $\set$). In fact, given a codirected diagram in $\npro_k$, w.l.o.g.~of non-empty spaces, its limit in $\ntop$ is compact Hausdorff by \autoref{prop:cod-lim-ne}, and has a base of representables by \autoref{lem:cod-lim}.
\item Every space $X\in \npro_k$ is a codirected limit of $k$-bounded orbit-finite discrete spaces by \autoref{lem:prof-canon-diag}.
\item Every $k$-bounded orbit-finite discrete space is finitely copresentable in $\npro_k$ by \autoref{lem:of-copres}
\end{enumerate}

For the second statement of the theorem, asserting that $k$-bounded nominal Stone spaces are precisely codirected limits of $k$-bounded orbit-finite discrete spaces, see \autoref{thm:pro-nom-equiv}.
\end{proof}

\subsection*{Proof of \autoref{thm:ind-compl}}

We classify  the \emph{Ind-completion} of \ncakba.
Dual to the Pro-completion, the Ind-completion \( \ind(\cat{C}) \) of a category \cat{C} is the free completion of \cat{C} under directed colimits.
For a thorough presentation of the construction of directed colimits of nominal sets we refer the reader to~\cite[Section 5.3]{pit-13}.
We recall some important facts.

\begin{rem}\label{rem:nom-lfp}
  \begin{enumerate}
    \item The category \nom is locally finitely presentable and its finitely presentable objects are precisely the orbit-finite sets~\cite[Theorem 5.16, Remark 5.17]{pit-13}.
    \item Directed colimits in  \nom  are created by the forgetful functor from  \nom  to \set.
  \end{enumerate}
\end{rem}

Dually to  \autoref{lem:pro-comp-char} the Ind-completion of a small category is characterized by the following lemma.
\begin{lemma}[\!\!{\cite[Theorem~A.4]{ada-che-mil-urb-21}}]\label{lem:ind-comp-char}
  If \cat{C} is small, then \( \ind(\cat C) \) is characterized, up to equivalence, as a category \( \cat{I} \) containing \( \cat C \) as full subcategory such that
  \begin{enumerate}
    \item the category \label{lem:ind-comp-char:has-colim} \cat{I} has directed colimits,
    \item\label{lem:ind-comp-char:is-colim} every object in \cat I is a directed colimit of objects in \cat C, and
    \item\label{lem:ind-comp-char:fin-pres} every object of \cat C is finitely presentable in \cat I.
  \end{enumerate}
\end{lemma}

Before we show that \ncofalkba satisfies the conditions of \autoref{lem:ind-comp-char}  we first prove an auxiliary lemma.

\begin{lemma} \label{lem:fil-colim}
Directed colimits of locally \( k \)-atomic orbit-finitely complete nominal boolean algebras are formed at the level of \nom.
\end{lemma}
\begin{proof}
  Let \( D \colon I \rightarrow \ncofalkba, \) be a directed diagram whose objects we denote \( D(i) = B_{i} \) with colimiting cocone \( (b_{i} \colon B_{i} \rightarrow B)_{i \in I} \) in \nom.
  It carries a unique structure of a nominal boolean algebra such that all \( b_{i} \) are homomorphisms:
  For \( x, y \in B \) there exists by directedness some \( i \in I \) and \( x_{i}, y_{i} \in B_{i} \) with \( b_{i}(x_{i}) = x, b_{i}(y_{i}) = y \).
  We define
  \[x \lor y = b_{i}(x_{i} \lor y_{i}), \quad x \land y = b_{i}(x_{i} \land y_{i}), \neg x = b_{i}(\neg x_{i}),\]
  this definition by directedness does not depend on the choice of \( I \),
  and it makes all colimit injections equivariant boolean algebra homomorphisms.

  We prove that that \( B \) is orbit-finitely complete with respect to this structure,
  and that all morphisms \( b_{i} \) preserve suprema of orbit-finite subsets.
  By \iref{rem:of-complete}{equiv} it suffices to show that suprema of \( S \)-orbits exist.
  Thus, for \( x \in B \), we choose \( i \) with \( x = b_{i}(x_{i}) \) and show that
  \[\bigvee \orb_{S} x := b_{i}(\bigvee \orb_{S} x_{i})\]
  is the supremum of $\orb_S$. Note that this definition does not depend on the choice of \( i \) by directedness of \( D \).
  It is indeed an upper bound for all \( \pi \cdot x \in  \orb_{S} x \)
  \[\pi \cdot x = \pi \cdot b_{i}(x_{i}) = b_{i}(\pi \cdot x_{i}) \le b_{i}(\bigvee \orb_{S} x_{i}) = \bigvee \orb_{S} x. \]
  Suppose \( y = b_{i}(y_{i}) \) (choice of the index by directedness does not matter, hence we may also choose it to be \( i \)) is also an upper bound for \( \orb_{S} x \):
  \[ \forall(\pi \in \perm_{S}\A).\ \pi \cdot b_{i}(x_{i}) \le b_{i}(y_{i}). \]
  We show that there exists some $k\geq i$ in $I$ such that
  \begin{equation}
   \forall(\pi \in \perm_{S}\A).\ \pi \cdot b_{ik}(x_{i}) \le b_{ik}(y_{i}) \label{eq:k}
 \end{equation}
 For this, note that for every \( \pi \in \perm_{S}\A \) there exists a subset \( T_{\pi} \subseteq \A \)  of names fresh for \( x_{i}, y_{i}, S \) of size \( |T_{\pi}| = |\supp x_{i} \setminus S| \),
 such that
 \[ \pi [\supp x_{i} \setminus S] \subseteq (\supp x_{i} \cup \supp y_{i} \cup T_{\pi}) \setminus S. \]
 We fix such a subset \( T \), then there exist only finitely many injective functions
 \[ \supp (x_{i}) \setminus S \mono (\supp x_{i} \cup \supp x_{i} \cup T) \setminus S; \]
 each of these can be extended to a permutation
 \( \sigma_{l} \in \perm_{S}\A, l=1, \ldots, m \).
 We can now recover any \( \pi \in \perm_{S}\A \) on \( \supp x_{i} \setminus S \) ``up to \( T \)'' from one of these permutations \( \sigma_{l} \), i.e.,
 there exists an \( 1 \le l(\pi) \le m \) and a permutation \( \hat{\pi} \) fixing \( S \cup \supp x_{i} \cup \supp y_{i} \) such that \[ \forall(a \in \supp x_{i} \setminus S)\colon (\hat{\pi} \cdot \pi)(a) = \sigma_{l(\pi)}(a). \]
 Every \( \sigma_{l} \) lies in \( \perm_{S}\A \), hence \( b_{i}(\sigma_{l} \cdot x_{i}) \le b_{i}(y_{i}) \).
 Thus we get for every \( 1 \le l \le m \) some \( k\geq i \) with \( b_{ik}(\sigma_{l} \cdot x_{i}) = b_{ik}(y_{i}) \);
 since \( D \) is directed this \( k \) may be chosen independently of \( l \).
 This \( k \) now satisfies~\eqref{eq:k}: For all \( \pi \in \perm_{S}\A \) we compute
 \begin{align*}
   \pi \cdot b_{ik}(x_{i}) &= b_{ik}(\pi \cdot x_{i}) \\
                     &= b_{ik}(\hat{\pi}^{-1} \cdot \hat{\pi} \cdot \pi \cdot x_{i}) \\
                     &= \hat{\pi}^{-1} \cdot b_{ik}(\hat{\pi} \cdot \pi \cdot x_{i}) \\
                     &= \hat{\pi}^{-1} \cdot b_{ik}(\sigma_{l(\pi)} \cdot x_i) \\
                     &\le \hat{\pi}^{-1} \cdot b_{ik}(y_{i}) \\
   &= b_{ik}(y_{i}).
 \end{align*}
 This proves \( \bigvee \orb_{S} b_{ik}(x_{i}) \le b_{ik}(y_{i}) \) and hence also
 \[ \bigvee \orb_{S} x  = b_{i}(\bigvee \orb_{S} x_{i}) = b_{k}(b_{ik}(\bigvee \orb_{S} x_{i})) =  b_{k}(\bigvee \orb b_{ik}(x_{i})) \le b_{k}(b_{ik}(y_{i})) = y. \]

 Next we prove local \( k \)-atomicity.
 Let \( x \in B, x = b_{i}(x_{i}) \),
 then by \( k \)-atomicity of \( B_{i} \) there exists a \( k \)-atomic complete subalgebra \( B_{i}' \subseteq B_{i} \) containing \( x_{i} \).
 Its image \( B' = b_{i}[B_{i}'] \subseteq B \) is a complete subalgebra with \( x \in B' \),
 we must show it \( k \)-atomic.
 First, every atom \( y_{i} \in \at(B_{i}') \) with \( b_{i}(y_{i}) \ne 0 \), is mapped to an atom \( b_{i}(y_{i}) \in \at (B') \):
 If \( y' < b_{i}(y_{i}) \) there exists some \( y_{i}' \in B_{i}' \) with \( b_{i}(y_{i}') = y' \).
 Then \( y_{i} \not \le y_{i}' \), but \( y_{i} \) is an atom and hence \( y_{i} \land y_{i}' = 0 \).
 We get
 \[ y' = y' \land b_{i}(y_{i}) = b_{i}(y_{i}' \land y_{i}) = 0, \]
 this proves \( b_{i}(y_{i}) \) an atom.
 Conversely, every atom \( x \in \at(B') \) of \( B' \) arises this way:
 we choose \( x_{i} \in B_{i}' \) with \( b_{i}(x_{i}) = x \),
 then \( x_{i} = \bigvee X_{i} \) for some \( X_{i} \subseteq \at(B_{i}') \).
 Since \( X_{i} \) is orbit-finite and \( b_{i} \) preserves this supremum we get
 \( x = b_{i}(x_{i}) = \bigvee b_{i}[X_{i}] \).
 But \( x \) is an atom and hence \( x = b_{i}(y_{i}) \) for some atom \( y_{i} \in X_{i} \) of \( B_{i} \).
 This shows that \( \at(B') \) is orbit-finite and \kbnd,
 yet we still have to prove atomicity of \( B' \).
 If \( x \in B' \) then \( x = b_{i}(x_{i}) \) with \( x_{i} = \bigvee X_{i} \) for \( X_{i} \subseteq \at (B_{i}) \).
 Therefore \( x = \bigvee_{y_{i} \in X_{i}} b_{i}(y_{i}) \), and we have shown that every element of this join is either an atom or zero,
 which proves \( B' \) atomic.

 Finally, it is a routine verification that \( B \) is indeed the colimit of the diagram \( D \), i.e., it satisfies the universal mapping property.
\end{proof}

\begin{lemma}\label{lem:fin-pres}
  Every \( B \in \ncakba \) is finitely presentable in \( \ncofalkba \).
\end{lemma}

\begin{proof}
  Let \( D \colon I \rightarrow \ncofalkba, i \mapsto B_{i} \) be a directed diagram with colimiting cocone \[(b_{i} \colon B_{i} \rightarrow B)_{i \in I}.\]

  We first prove that for every \( B' \sa B \) there exists \( j \in I \) and a subalgebra \( B_{j}' \sa B_{j} \) such that \( b_{j} \) restricts to an isomorphism \( B_{j}' \cong B' \).
  We first show this on atoms, that is, consider the orbit-finite, \kbnd equivariant subset \( \at(B') \) of atoms of \( B' \).
  It is finitely presentable in \( \nom \),
  so the inclusion \( \at(B') \hookrightarrow B \) factors through some colimit injection \( b_{i} \) as \( b_{i} \cdot k \colon \at(B') \mono B_{i} \rightarrow B \).
  If we denote the image of \( k \) by \( X_{i} =  k[\at(B')] \) then injectivity of \( k \) shows \( \at(B') \cong X_{i} \).
  \[
    \begin{tikzcd}
      X_{i}
      \arrow[hook]{d}
      &
      \at(B')
      \arrow[hook]{d}
      \arrow[phantom]{l}{\cong}
      \arrow[swap]{ld}{k}
      \\
      B_{i}
      \arrow{r}{b_{i}}
      &
      B
    \end{tikzcd}
  \]
  To extend this isomorphism  to subalgebras we choose a subset \( S \subseteq \A \) of size \( |S| = 2k \) and consider all (finitely many) equations
  \[x \land y = 0 \qquad x \ne y \in X_{i}, \supp x \cup \supp y \subseteq S \]
  in \( B_{i} \) together with the equation \( \bigvee X_{i} = 1 \).
  These equations hold in \( B \) when \( b_{i} \) is applied.
  By directedness there now exists some \( j \in I \) such that all these equations with \( b_{ij}  \) applied hold in \( B_{j} \), i.e.,
  if we denote \( X_{j} = b_{ij}[X_{i}] \) then \( \bigvee X_{j} = 1 \) and
  \[ x \land y = 0 \qquad x \ne y \in X_{j}, \supp x \cup \supp y \subseteq S. \]
  But we find for arbitrary \( x, y \in X_{j} \) a permutation \( \pi \) mapping \( \supp x \cup \supp y \) to \( S \), so
  \[\supp (\pi \cdot x) \cup \supp (\pi \cdot y) \subseteq \pi \cdot (\supp x \cup \supp y) \subseteq S, \]
  and whence
  \[ x \land y = \pi^{-1} \cdot \pi \cdot x \land \pi^{-1} \cdot \pi \cdot y = \pi^{-1} \cdot (\pi \cdot x \land \pi \cdot y) = \pi^{-1} \cdot 0 = 0. \]
  So the set \( X_{j} = b_{ij}[X_{i}] = b_{ij}[k[\at(B')]] \) generates a subalgebra \( B_{j}' \sa B_{j} \) with atoms \( \at(B_{j}') = X_{j} \),
  and the isomorphism \( X_{j} \cong \at(B') \) induces an isomorphism \( B_{j}' \cong B' \).
  The result is the following extended diagram
  
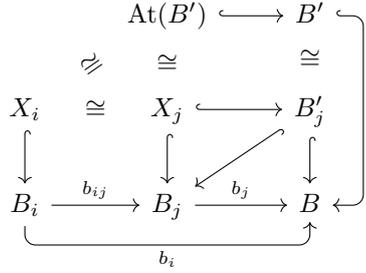
\[
    \begin{tikzcd}
      &
      \at(B')
      \arrow[phantom]{ld}{\rotatebox[origin=c]{35}{\(\cong\)}}
      \arrow[phantom]{d}{\cong}
      \arrow[hook]{r}
      &
      B'
      \arrow[phantom]{d}{\cong}
      \arrow[shiftarr={xshift=20}, hook]{dd}{}
      \\
      X_{i}
      \arrow[phantom]{r}{\cong}
      \arrow[hook]{d}
      &
      X_{j}
      \arrow[hook]{d}
      \arrow[hook]{r}
      &
      B_{j}'
      \arrow[hook]{ld}
      \arrow[hook]{d}
      \\
      B_{i}
      \arrow[shiftarr={yshift=-15}, swap]{rr}{b_{i}}
      \arrow{r}{b_{ij}}
      &
      B_{j}
      \arrow{r}{b_{j}}
      &
      B
    \end{tikzcd}
  \]

  We are ready to prove finite presentability of \( C \in \ncakba \). Take any morphism \(f \colon C \rightarrow B \).
  The image \( f[C] \subseteq B \) is a subalgebra \( f[C] \sa B\).
  Hence, there exists some \( j \in I \) and a subalgebra \( C_{j} \sa B_{j} \) such that \( b_{j} \) restricts to an isomorphism \( C_{j} \cong f[C] \).
  This induces a factorization of \( f \) via
  
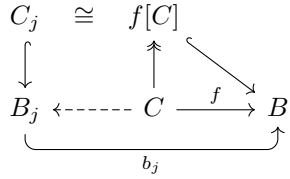
\[
    \begin{tikzcd}
      C_{j}
      \arrow[phantom]{r}{\cong}
      \arrow[hook]{d}
      &
      f[C]
      \arrow[hook]{rd}
      &
      \\
      B_{j}
      \arrow[shiftarr={yshift=-15}, swap]{rr}{b_{j}}
      &
      C
      \arrow[dashed]{l}
      \arrow{r}{f}
      \arrow[two heads]{u}
      &
      B
    \end{tikzcd}
  \]

  Finally, to prove its essential uniqueness, assume there exist two factorizations \( g, h \) of some morphism \( C \rightarrow B \) through the same injection \( b_{j} \).
  The colimit  \( B = \colim D \) is taken in \nom, wherein the orbit-finite set of atoms \( \at(C) \subseteq C \) is finitely presentable,
  hence there exists some \( i \in I \) such that \( b_{ji} \) joins the respective restrictions of \( g, h \) to \( \at(C) \).
  Both \( g \) and \( h \) are morphisms of orbit-finitely complete atomic boolean algebras and hence determined by their values on atoms, so \( b_{ji} \) also joins \( g \) and \( h \). \qedhere
  \[
    \begin{tikzcd}
      &
      C
      \arrow[xshift=2]{d}{g}
      \arrow[xshift=-2, swap]{d}{h}
      \arrow{rd}{}
      &
      \\
      \at(C)
      \arrow[yshift=-2, swap]{r}{}
      \arrow[yshift=2]{r}{}
      \arrow[hook]{ru}
      \arrow{rd}
      &
      B_{j}
      \arrow[dashed, swap]{d}{b_{ji}}
      \arrow{r}{b_{j}}
      &
      B
      \\
      &
      B_{i}
      \arrow{ru}{b_{i}}
      &
    \end{tikzcd}
  \]
\end{proof}

\begin{lemma}\label{lem:dir-union}
  Every \( B \in \ncofalkba \) is the directed union of its complete \( k \)-atomic subalgebras.
\end{lemma}
\begin{proof}
  It suffices to show that the poset of all complete \( k \)-atomic subalgebras is directed, then its union is equal to \( B \) by definition.
  Given orbit-finitely complete \( k \)-atomic subalgebras \( C_{0}, C_{1} \subseteq B \) we construct a orbit-finitely complete \( k \)-atomic subalgebra \( C\subseteq B\) containing both \( C_{0}, C_{1} \) as subalgebras.
  The sets \( \at(C_{0}), \at(C_{1}) \) are orbit-finite and thus the set
  \[ X = \{ c_{0} \land c_{1} \mid c_{i} \in \at(C_{i})\} \]
  is orbit-finite and forms a partition of \( 1 \) in \( B \).
  The subalgebra \( C' \subseteq B \) generated by \( X \) is hence orbit-finitely complete and atomic,
  but not necessarily \( k \)-atomic since elements of \( X \) might have support larger than \( k \).
  We embed \( C' \) into a larger \( k \)-atomic subalgebra \( C \).
  For every orbit-representative \( y \in X \) we find an orbit-finitely complete \( k \)-atomic boolean subalgebra \( B_{y} \subseteq B \) containing \( y \) with atoms \( X_{y} \subseteq B_{y} \).
  The hull \( Y = \hull (\bigcup_{\orb y \subseteq X} X_{y}) \) is an orbit-finite partition of \( 1 \):
  clearly \( \bigvee Y \ge \bigvee X = 1 \),
  and if \( x_{y} \in X_{y}, x_{y'} \in X_{y'} \) for \( y, y' \in X \)
  then \( x_{y} \land x_{y'} \le y \land y' = 0 \).
  The subalgebra \( C \subseteq B\) generated by \( Y \)  then is orbit-finitely complete, \( k \)-atomic and it contains both \( C_{0}, C_{1} \) as subalgebras.
\end{proof}

We are now ready to prove that \ncofalkba is the Ind-completion of \ncakba by verifying the three conditions from \autoref{lem:ind-comp-char}:\\
(1) The category \ncofalkba has all directed colimits by \autoref{lem:fil-colim},
(2) every object arises as a directed colimit of its \( k \)-atomic subalgebras by \autoref{lem:dir-union},
and (3) every \( B \in \ncakba \) is finitely presentable in \( \ncofalkba \) by \autoref{lem:fin-pres}.

\subsection*{Details for \autoref{rem:concrete-duality}}

We give a detailed proof that the equivalence \( \npro_{k} \simeq^{\mathrm{op}} \ncofalkba \) indeed works as described in \autoref{rem:concrete-duality}.

\begin{rem}\label{rem:prime-filter-equiv}
In the following, ``prime filter'' always means nominal orbit-finitely complete prime filter.
  To verify that a finitely supported subset \( F \subseteq B \in \ncofalkba \)  is a prime filter, the condition \( \bigvee X \in F \Rightarrow X \cap F \ne \emptyset \) needs only to be checked for \kbnd orbit-finite sets \( X \) whose elements are pairwise disjoint ($x\wedge y=0$ for $x\neq y$ in $X$).
  To see this,  let \( X \subseteq B \) be orbit-finite and supported by \( S \).
  Then it is included in a subalgebra \( X \subseteq A \sa B \).
  The set \( X' = \{a \in \at(A) \mid \exists x \in X.\ a \le x\} \)
  is a \kbnd orbit-finite set of pairwise disjoint elements, and $\bigvee X = \bigvee X'$.  Moreover \( X' \cap F \ne \emptyset \) implies
 \( X \cap F \ne \emptyset \) because $F$ is upwards closed. Thus the primality condition for $X'$ implies that for $X$.
\end{rem}

\noindent\emph{From \kbnd nominal Stone spaces to locally \( k \)-atomic orbit-finitely complete nominal boolean algebras.}\\
           A \kbnd nominal Stone space \( X \) is mapped via duality to the directed colimit of the diagram \( \colim (\powfs \circ (\slice{X}{\nomofk})) \).
          We show that this colimit is given by boolean algebra \( \clo(X) \) of clopens of \( X \) with colimiting cocone
          \[ f^{-1} \colon \powfs(X_{f}) \rightarrow \clo(X) \qquad  (f \colon X \rightarrow X_{f} \text{ with } X_{f} \in \nomofk). \]
          Filtered colimits in \ncofalkba are formed in \set by \autoref{lem:fil-colim},
          so it suffices that to prove that
          (i) this cocone is jointly epimorphic,
          and (ii) any two finitely supported subsets \( U, V \in \powfs(X_{f}) \) merged by \( f^{-1} \) are already merged by some connecting map.
          For (i) it suffices to note that every clopen of \( X \) is representable by \autoref{lem:rep:lim}.
          For (ii) let \( f^{-1}[U] = f^{-1}[V] \).
          The subset \( f[X] \hookrightarrow X_{f} \) is a \kbnd orbit-finite nominal set,
          hence the corestriction \( X \epi f[X] \) lies in the canonical diagram for \( X \).
          The inclusion \( \iota \colon f[X] \hookrightarrow X_{f} \) is a connecting morphism merging \( U \) and \( V \):
          \[\iota^{-1}[U] = \{f(x) \mid f(x) \in U\} = \{f(x) \mid x \in f^{-1}[U]\} = \{f(x) \mid x \in f^{-1}[V]\} = \iota^{-1}[V].\]
\noindent\emph{From locally \( k \)-atomic orbit-finitely complete nominal boolean algebras to \kbnd nominal Stone spaces.}\\
          In the other direction,
          the duality maps \( B \in \ncofalkba \) to the codirected limit
          \[ L = \lim (\at \circ (\ncakba \dhook B)) \] in \( \npro_{k} \) with limit projections \( p_{A} \colon L \rightarrow \at(A) \).
          We prove that this limit is isomorphic to the space \( \fp(B) \) of prime filters of \( B \) whose topology has basic open sets \( \hat{b} = \{F \in \fp(B) \mid b \in F\} \). The limiting cone of \( \fp(B) \) is given by
          \[x(-)_{A} \colon \fp(B) \rightarrow \at(A) \qquad A \sa B, \]
          where the morphism \(x(-)\) maps a prime filter \( F \subseteq B \) to the unique atom of \( A \) that lies in \( F \).

          (1) We first show that the family \( x(-)_{A} \) is a well-defined cone.
          Then, since the limit \( L \) is formed in \set by \autoref{lem:cod-lim-set}, the universal property of \( L \) induces a morphism
          \[x(-) \colon \fp(B) \rightarrow L \qquad F \mapsto (x(F)_{A})_{A \sa B}. \]

          (a) We prove that every \(x(-)_{A}\) is a function, i.e.,
          that for every nominal orbit-finitely  complete prime filter \( F \subseteq B \) and \( A \sa B \) the set \( x(-)_{A} = \at(A) \cap F \) is a singleton.
          The set \( \at(A) \cap F \) is non-empty as \( 1 = \bigvee \at(A) \in F \) is the join over an orbit-finite set,
          so some \( x \in \at (A) \) lies in  \( F \).
          If \( \at (A) \cap F \) were to contain two elements \( a \ne b \),
          then \( 0 = a \wedge b \in F \) since \( F \) is downwards directed,
          which would be a contradiction to \( F \) being a proper subset of \( B \).
          By the equation \( x(-)_{A} = \at(A) \cap (-) \) equivariance is clear.

          (b) The family \( (x(-)_{A})_{A \sa B}\) indeed forms a cone.
          Let \( C, D \le_{\mathrm{of}, k} B \) with \( C \hookrightarrow D \),
          and recall that the connecting morphism \( \at(D) \epi \at(C) \) maps \( d \) to the unique \( c \in \at(C) \) with \( c \ge d \).
          Let \( F \in \fp(B) \) be a prime filter.
          If \( c \in \at(C) \) is the unique atom with \( c \ge x(F)_{A} \) then \( c \in F \) as \( F \) is upwards closed,
          hence \( c \in \at(C) \cap F  \),
          so \( x(-)_{A} \) is a cone.

          (2) Consider the map
          \[ F(-) \colon L \rightarrow \powfs B, \qquad x \mapsto F(x) = \{y \mid y \ge x_{A} \text{ for some } A \sa B, \} \subseteq B. \]
          This assignment is equivariant and \( F(x) \) is finitely supported by \( \supp x \),
          so \( F(x) \) is \kbnd.
          We show that \( F(-) \) corestricts to \( \fp(B) \), i.e.,
          that every \( F(x) \) is a prime filter.

          (a) The set \( F(x) \) is upwards closed by definition.

          (b) The set \( F(x) \) is downwards directed.
          If \( y, y' \in F(x) \) with \( y \ge x_{A}, y' \ge x_{A'} \) take the upper bound \( C \sa B \) of \( A, A' \).
          By compatibility, the component \( x_{C} \in F(x) \) then satisfies \( x_{C} \le x_{A} , x_{A'} \), whence \( x_{C} \le x_{A} \wedge x_{A'} \).
          By (a) the set \( F(x) \) is upwards closed,
          so \( y \land y' \ge x_{A} \wedge x_{A'} \ge x_{C} \in F(x) \).

          (c) We show that \( F(x) \) is nominally prime by using \autoref{rem:prime-filter-equiv}.
          Let \( X \) be a \kbnd orbit-finite set of disjoint elements with \( \bigvee X \in F(x) \).
          By definition of \( F(x) \) there exists some \( A \sa B \) with \( \bigvee X = y \ge x_{A} \in \at(A) \).
          The set \( X' = \at(A) \setminus \{a \in \at(A) \mid a \le y\} \cup X \) forms a \kbnd orbit-finite partition of \( 1 \),
          so it generates a subalgebra \( A' \sa B \).
          The element \( x_{A'} \in F(x) \) however cannot lie in \( \at(A) \setminus  \{a \in \at(A) \mid a \le y\} \),
          since then \( x_{A'} \land x_{A} = 0 \in F(x) \).
          This shows that \( x_{A'} \in X \cap F(x) \subseteq F(x) \).

          (3) One readily verifies that the assignments \( x(-), F(-) \) are mutually inverse:
          In one direction \( x(F(x))_{A} = \at A \cap F(x) = x_{A} \).
          In the other direction \( F(x(F)) = F \) let for ``\( \subseteq \)'' \( y \in F(x(F)) \),
          then there exists some subalgebra \( A \) such that \( y \ge x(F)_{A} = \at(A) \cap F \).
          Since \( x(F)_{A} \in F \) also \( y \in F \), as \( F \) is upwards closed.
          For ``\( \supseteq \)'', if \( y \in F \) we take some subalgebra \( A \sa B \) such that \( y \in A \).
          It cannot occur that \( y \not \ge x(F)_{A} \in F \),
          as this would imply \( 0 = y \land x(F)_{A} \in F \).
          So \( y \ge x(F)_{A} \), which proves \( y \in F(x(F)) \).

          Finally, one easily sees that the isomorphisms \( x(-), F(-) \) are compatible with the limit projections \( x(-)_{A}, p_{A} \).
          The basis of \( \fp(B) \) is induced by the family \(x(-)_{A}\) proving that also \( L \cong \fp(B) \) as topological spaces. \qed
          \[
            \begin{tikzcd}[column sep=small]
              L
              \arrow[rr, "F(-)", bend left=13]
              \arrow[swap]{rd}{p_A}
              \arrow[phantom]{rr}{\cong}
              &
              &
              \fp\mathrlap{B}
              \arrow[ll, "x(-)", bend left=13]
              \arrow{ld}{x(-)_{A }}
              \\
              &
              \at(A)
              &
            \end{tikzcd}
          \]

\begin{rem}\label{rem:canon-not-cod}
  The diagram \( \slice{\fm}{\nmonofk} \) generally fails to be codirected.
  For example, take \( \Sigma = \A \) and \( k = 1 \), and consider the \kbnd[1] monoids $P_1$ and $P_2$ with carrier $\A + \{1\}$ and multiplication projecting on the first and second component, respectively.
  Define the equivariant monoid morphisms \( p_{i} \colon \fm[\A] \rightarrow P_{i} \) projecting a word on its first, respectively last, component.
  For the sake of contradiction,
  we assume that there exists some lower bound \( p \colon \fm[\A] \rightarrow P \) for \( p_{i}\), $i=1,2$ in \( \slice{\fm[\A]}{\nmonof{_{,1}}} \) with connecting morphisms \( k_{i} \colon P \rightarrow P_{i}\).
  Let $a,b\in \A$ be distinct names. The element \( p(ab) \in P \) is then mapped by the connecting morphisms \( k_{i} \) to the names \( a, b \), respectively.
  But then
  \[ \{a,b\} = \supp{k_{1}(p(ab))} \cup \supp{k_{2}(p(ab))} \subseteq \supp{p(ab)}, \]
  so \( |\supp{p(ab)}| > 1 \), hence $P$ ist not \kbnd[1], a contradiction.

  The missing codirecteness essentially refutes the nominal topological space \( \lim(\slice{\fm}{\nomofk}) \) from a language-theoretic viewpoint:
  The languages recognizable by $k$-bounded orbit-finite monoids (and thus also the corresponding representables of \( \lim(\slice{\fm}{\nomofk}) \)) are not closed under intersection.
  To see this, consider the above monoids \( P_{1}, P_{2} \).
  They respectively recognize the data languages \( a\A^{*} \) and \( \A^{*}a \) (for fixed $a\in \A$),
  but their intersection \( L_{a} = a\A^{*}a \) is not recognizable by any \( 1 \)-bounded orbit-finite nominal monoid.
  Indeed, for the sake of contradiction,
  assume that \( L_{a} \) is recognized by a \( 1 \)-bounded nominal monoid \( M \) as \( L_{a} = h^{-1}[P_{a}] \) for \( h \colon \fm[\A] \rightarrow M \supseteq P_{a} \).
  Then \( h\) also recognizes every language \( L_{b} = \tr a b \cdot L_{a} \) and hence also the equivariant language \( L = \bigcup_{a \in \A} L_{a} = h^{-1}[\bigcup_{a \in \A}P_{a}] \).
  But for equivariant languages we can compute the syntactic nominal monoid~\cite{boj-13},
  and that of \( L \) is easily seen to be isomorphic to \( P_{1} \times P_{2} \).
  This monoid is not \kbnd[1],
  so  no factorization of the syntactic morphism through \( h \) can exist, a contradiction.
\end{rem}

\subsection*{Details for \iref{rem:canon}{dense}}
(a)  We prove  \( \eta[\fm] \subseteq \spfm \) is dense.
  If \( R = \hat{h}^{-1}[P] \subseteq \spfm \) is a non-empty representable basic open set,
  then \( P \) is non-empty and there exists some word \( w \in h^{-1}[P]  \).
  Then $\eta(w)\in R$, i.e.\ $R\cap \eta[\fm]\neq\emptyset$.

\medskip\noindent 
(b)  We show that the unit \( \eta \colon \fm \rightarrow \spfm \) is not injective provided that $\Sigma$ is nontrivial, that is, there exists an element \( x \in \Sigma \) with non-empty least support.
  Let \( s \colon \fm \rightarrow \pow_{k} \A \) be a support bound and choose any \( w \in \fm  \) with \( \supp w \not \subseteq s(w) \);
  such \( w \) must exist since the size of the sets \( s(w) \) is bounded by \( k \) while that of the supports \( \supp w \) is not (\( \Sigma \) is not a trivial nominal set).
  Let \( a \in \supp w \setminus s(w) \) and choose \( b \in \A\) fresh. There exists a permutation \(  \pi \in \perm_{s(w)} \A \) mapping \( a \) to \( b \).
  We know \( w \ne \pi \cdot w \) since \( b = \pi \cdot a \in \pi \cdot \supp w =  \supp (\pi \cdot w) \), and \( b \not \in \supp w \).
  But every \sbnd \( h \) identifies \( h(w) = \pi \cdot h(w) = h(\pi \cdot w): \) the image \( h(w) \) is supported by \( s(w) \) which is fixed by \( \pi \). Thus $\eta(w)=\eta(\pi\o w)$.

\subsection*{Details for \iref{rem:canon}{monoid-structure}}
We prove that the monoid multiplication of $\spfm$ is continuous.
        Let \( \hat{h}^{-1}(m) \) be a basic open environment of \( x \cdot y \).
        The set \( U = \hat{h}^{-1}(\hat{h}(x)) \times \hat{h}^{-1}(\hat{h}(y)) \) is an open environment of \( (x, y) \),
        and under multiplication it is mapped to \( \hat{h}^{-1}(m) \), i.e.,
        for all \( (x', y') \in  U \) we get
        \[\hat{h}(x' \cdot y') = \hat{h}(x') \cdot \hat{h}(y') = \hat{h}(x) \cdot \hat{h}(y) = \hat h(x\cdot y) = m. \]

\subsection*{Proof of \autoref{lem:pfm-metr}}
  \begin{enumerate}
    \item
(a)  We show that the respective bases are included in each other.
  Let \( \hat{h}^{-1}[m] \subseteq \spfm, h \colon \fm \rightarrow_{s} M \) be a basic open neighbourhood of \( x \in \spfm \) in the limit topology of \spfm.
  If \( n >  |\orb(M)| \), then \( \hat{d}(x, y) < 2^{-n} \) implies that \( \hat{h}(x) = \hat{h}(y) = m \),
  and hence \( B_{2^{-n}}x \subseteq \hat{h}^{-1}[m] \).

  In the other direction, let \( B_{2^{-n}}x \) be a basic open neighbourhood of \( y \) in the metric topology.
  Every alphabet \( \Sigma \) only has finitely many non-isomorphic quotients, hence there exist finitely many \sbnd homomorphisms \( h_{i} \colon \fm \rightarrow_{s} M_{i} \) whose codomain has no more than \( n \) orbits.
  Let \( h = \bigvee h_{i} \) be the join of those morphisms with connecting morphisms \( f_{i} \).
  If now \( z \in \hat{h}^{-1}[\hat{h}(y)] \) then for all \( h_{i} \) we get  \[ \hat{h}_{i}(z) = f_{i}(\hat{h}(z)) = f_{i}(\hat{h}(y)) = \hat{h}_{i}(y) = \hat{h}_{i}(x), \]
  which proves \( d(x, z) < 2^{-n} \),
  whence \( \hat{h}^{-1}(y) \subseteq B_{2^{-n}}x \).

\medskip\noindent (b) To prove completeness, we first simplify the description of \spfm.
  Note that for each \( n \in \N \) there only exist, up to isomorphism, finitely many \emph{\kbnd} orbit-finite nominal monoids with \( n \) orbits.
          We can thus impose a monotone enumeration \( I \colon (\can) \cong \N \) on the diagram \can by first enumerating homorphisms whose codomain has one orbit,
          then those with two orbits, etc., such that
  \begin{equation}
  (\forall e, e' \in \can)\colon e \le e' \Rightarrow I(e) \le I(e'). \label{eq:indexing}
  \end{equation}
  We denote \( \cod(I(e)) = M_{i} \), so by~\eqref{eq:indexing} \( i \le j \Rightarrow  |\orb(M_{i})| \le | \orb(M_{j}) | \).
  Elements of \spfm are thus equivalent to compatible families  \( (x_{i} \in M_{i})_{i \in \N} \).
  Now assume that \( ((x_{i})_{n}) \) is a \fs cauchy sequence of compatible families in \spfm supported by \( S \subseteq \A \).
  Spelling out the definition of \( \hat{d} \) being Cauchy reads
  \[
    \forall(k \in \N)\colon \exists(N(k) \in \N)\colon \forall(n, n' \ge N(k))\colon  |\orb(M_{i})| \le k \Rightarrow x_{i, n} = x_{i, n'}.
  \]
  By \eqref{eq:indexing} this condition simplifies to
  \begin{equation}
    \label{eq:cauchy:2}
    \forall(k \in \N)\colon \exists(N(k) \in \N)\colon \forall(n, n' \ge N(k))\colon x_{i, n} = x_{i, n'}.
  \end{equation}
  We now construct the limit \( (y_{i})_{i \in \N} \in \spfm \) of
  this sequence by setting \( y_{i} = x_{i, n_{i}} \), where
  \( n_{i} \) is chosen such that it satisfies (i)
  \( n_{i} \ge N(i) \) and (ii)
  \( i \le j \Rightarrow n_{i} \le n_{j} \).  The family
  \( (y_{i})_{i \in N} \) is supported by \( S \); we prove that
  \( (y_{i})_{i \in \N} \) is indeed a compatible family: Let
  \( y_{j} \in M_{j} \) and \( h_{ji} \colon M_{j} \epi M_{i} \) then
 \[h_{ji}(y_{j}) = h_{ji}(x_{j, n_{j}}) \stackrel{\text{a}}{=} x_{i, n_{j}} \stackrel{\text{b}}{=} x_{i, n_{i}} = y_{i}.\]
 Equality (a) holds since \( (x_{i})_{n_{i}} \) is a compatible family.
 For equality (b) we know that \( i = I(e), j = I(e') \) for some quotients \( e, e' \in \can \).
 That \( h_{ij} \) is a connecting morphism implies \( e \le e' \), and thus
 \[ N(i) \stackrel{\text{(i)}}{\le} n_{i} = n_{I(e)} \stackrel{\mathclap{\text{\eqref{eq:indexing} + (ii)}}}{\le} n_{I(e')} = n_{j},  \] which now by~\eqref{eq:cauchy:2} implies \( x_{i, n_{j}} = x_{i, n_{i}} \).

 It is easy to see that \( (y_{i}) \) is indeed the limit of the sequence \( ((x_{i})_{n}) \).
    \item The relation identifying, \( v, w \) iff \( d_{s}(v, w) = 0 \) is precisely the kernel of \( \eta \colon \fm \rightarrow \spfm \),
  and \( \hat{d}(\iota([v]), \iota([w])) = d_{s}([v], [w]) \) holds by definition;
  therefore \( \iota \) is an dense isometry. \qed
    \end{enumerate}

\subsection*{Proof of \autoref{thm:rec-rep-equiv}}
  \emph{From \( s \)-recognizable languages to representable subsets.}
  For a language \( L \subseteq \fm \) recognized as \( L = h^{-1}[P] \) by an \sbnd homomorphism \( h \colon \fm \rightarrow_{s} M, P \subseteq_{\text{fs}} M \) we define its corresponding representable clopen as \( \hat{L} = \hat{h}^{-1}[P] \).
  \[
    \begin{tikzcd}[row sep=small, column sep=small]
      L = h^{-1}[P]
      \arrow[phantom]{r}{\subseteq}
      &
      \fm
      \arrow{r}{h}
      \arrow{d}{}
      &
      M
      \arrow[phantom]{r}{\supseteq}
      &
      P
      \\
      \hat{L} = \hat{h}^{-1}[P]
      \arrow[phantom]{r}{\subseteq}
      &
      \spfm
      \arrow[dashed]{ru}{\hat{h}}
      &
      &
    \end{tikzcd}
    \]
  The definition of \( \hat{L} \) is independent of \( h \) and \( P \):
  If \( L = h'^{-1}[P'] \) for \( h' \colon \fm \rightarrow_{s} M', P' \subseteq_{\text{fs}} M'  \),
  then we take the coimage \( e \colon \fm \epi_{s} E \) of $\langle h, h' \rangle$ with connecting morphisms \( f \colon E \rightarrow M, f' \colon E \rightarrow M' \).
  In \( E \) we now have \( f^{-1}[P] = f'^{-1}[P'] \) and whence
  \( \hat{h}^{-1}[P] = \hat{e}^{-1}[f^{-1}[P]] = \hat{e}^{-1}[f'^{-1}[P']] = \hat{h}'^{-1}[P'] \).
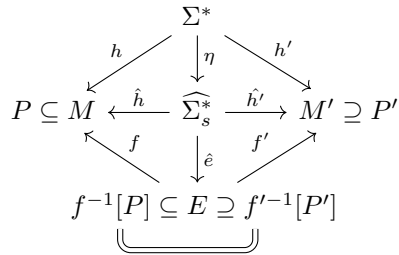
  \[
    \begin{tikzcd}[column sep=small]
      &
      \fm
      \dar{\eta}
      \dlar[swap]{h}
      \drar{h'}
      &
      \\
      P \subseteq M
      &
      \spfm
      \lar[swap]{\hat{h}}
      \rar{\hat{h'}}
      \dar{\hat{e}}
      &
      M' \supseteq P'
      \\
      \phantom{X}
      \arrow[shiftarr={yshift=-8}, equal,
      start anchor={[xshift=25]},
      end anchor={[xshift=-35]},
      ]{rr}
      &
      \mathllap{f^{-1}[P]} \subseteq E \supseteq \mathrlap{f'^{-1}[P']}
      \ular[swap]{f}
      \urar{f'}
      &
      \phantom{X}
    \end{tikzcd}
    \]

  The clopen \( \hat{L} \) is in fact is equal to the topological closure \( \overline{ \eta[L] } \) of \( \eta[L] \subseteq \spfm \).
  Since \( \eta[L] \subseteq \hat{h}^{-1}[P] = \hat{L} \) and \( \hat{L} \) is clopen we get the direction \( \overline{ \eta[L] } \subseteq \hat{L} \).
  Conversely, \( \eta[\fm] \) is dense in \( \spfm \),
  so \(\eta[L] = \eta[\fm] \cap \hat{L}\) is dense in \( \spfm \cap \hat{L} = \hat{L} \).

\medskip\noindent
[Proof of \(\eta[L] = \eta[\fm] \cap \hat{L}\): ($\seq$) If $w\in L$ then $\hat h\cdot \eta(w)=h(w)\in P$, hence $\eta(w)\in \hat h^{-1}[P] = \hat L$. \\
($\supseteq$) If $w\in \fm$ and $\eta(w)\in \hat L=\hat h^{-1}[P]$, then $w\in \eta^{-1}\hat h^{-1}[P]=h^{-1}[P]=L$, whence $\eta(w)\in \eta[L]$.]

\medskip\noindent \emph{From representable subsets to recognizable languages.}
  If \( \hat{L} \subseteq \spfm \) is representable,
  we simply define the corresponding language as \( L = \eta^{-1}[\hat{L}] \).
  We show that \( L \) is indeed \( s \)-recognizable.
  The set \( \hat{L} \) is representable,
  hence there exists a continuous equivariant function \( f \colon \spfm \rightarrow X\) into a (w.l.o.g.\ \kbnd) orbit-finite set such that \( \hat{L} = f^{-1}[P] \) for some finitely supported \( P \subseteq X \).
  Since \( X \) is finitely copresentable in \( \npro_{k} \) the map \( f \) factors through some \( \hat{h} \) as \( f = p \cdot \hat{h} \) for an \sbnd morphism  \( h \colon \fm \rightarrow M \).
  We compute
  \[L = \eta^{-1}[\hat{L}] = \eta^{-1}[f^{-1}[P]] = h^{-1}[p^{-1}[P]],\]
  this shows that \( L \) is \( s \)-recognizable.

  \hspace{3pt}
  \[
    \begin{tikzcd}[row sep=small, column sep=small]
      &
      &
      &
      &
      \mathllap{L = \eta^{-1}[\hat{L}] }= h^{-1}[P]
      \arrow[phantom]{r}{\subseteq}
      &
      \fm
      \arrow{r}{h}
      \arrow{d}{}
      &
      M
      \arrow[dashed]{d}{\exists p}
      \arrow[phantom]{r}{\supseteq}
      &
      P \mathrlap{ =  p^{-1}[U] }
      \\
      &
      &
      &
      &
      \mathllap{\hat{L}} = f^{-1}[U]
      \arrow[phantom]{r}{\subseteq}
      &
      \spfm
      \arrow[dashed]{ru}{\hat{h}}
      \arrow{r}{f}
      &
      X
      \arrow[phantom]{r}{\supseteq}
      &
      U
    \end{tikzcd}
    \]

  The assignments \( L \mapsto \hat{L} \) and \( \hat{L} \mapsto L \) are mutually inverse.
  Finally, it is obvious from the definition that the mapping \( \hat{L} \mapsto L = \eta^{-1}[\hat{L}] \) is a homomorphism,
  and thus also an isomorphism,
  of orbit-finitely complete boolean algebras (where the boolean operations on $\rec_s(\fm)$ and $\clo(\spfm)$ are union, intersection, and complement).
\begin{expl}\label{ex:no-s-quot}
  Let \( \Sigma = \A, s \colon \fm[\A] \rightarrow \pow_{2} \A \) with \( s(a_{1} \cdots a_{n}) = \{a_{1}\} \).
  Let \( N = \{1\} + \A + 0 \) with monoid structure given by \( x \cdot y = 0 \)
  and let
  \[M = 1 + \A + \A * \A + 0\]
  have monoid structure given by \( a \cdot b = ab \) for \( a \ne b \) and \( a \cdot (b, c) = (b, c) \cdot a = (b, c) \cdot (d, e) = 0 \), with \( 0 \) absorbing.
  The monoid quotient \( e \colon M \epi N \) mapping \( e(a) = a,  e(a, b) = e(0) = 0 \) is support-reflecting (see \iref{def:epis}{supp-refl}).
  Now suppose \fm[\A] is projective with respect to \( e \) and let \( h \colon \fm[\A] \rightarrow N \) be the extension of \( a \mapsto a \).
  If \( h = e \cdot h' \) is a factorization of \( h \) through \( e \) it satisfies \( h'(a) = a \).
  But then
  \[ \supp (h'(ab)) = \supp (h'(a)\cdot h'(b)) = \{a, b\}. \]
  But this means that \( h \) cannot factorize through any \sbnd morphism \( \fm[\A] \rightarrow M \).
  \[
    \begin{tikzcd}[column sep=tiny]
    & \mathbb{A} \arrow[ld, "\eta^*"] \arrow[rdd, "a \mapsto a", bend left=60] \arrow[ldd, "a \mapsto a"', bend right=60] \arrow[rd, "", ] & \\
    \mathbb{A}^* \arrow[rr, "\eta", ] \arrow[d, "h'", ] & &
    \widehat{\mathbb{A}_s^*} \arrow[d, "\hat{h}", ] \arrow[lld, "\times" marking, ""', dashed] \\
    M \arrow[rr, "e", two heads] & & N
  \end{tikzcd}
  \]
\end{expl}

\subsection*{Proof of \autoref{prop:s-quot}}

\begin{rem}
For every support bound $s$ and every orbit-finite monoid $M$, with common upper bound $k$ for the support size, every continuous monoid morphism $h\colon \spfm\to M$ is of the form $\hat g$ for some $g\colon \fm\to_s M$. Indeed, since $M$ is finitely copresentable in $\npro_k$, the map $h$ factors as $h=l\cdot \hat e$ for some $e\in\sepislice{\fm}{\nmonofk}$. Since $e$ is surjective and $e$ and $h$ are monoid morphisms, so is $l$. Hence $h=l\cdot \hat e = \widehat{l\cdot e}$. 
\end{rem}

  We first show that for every strong orbit-finite nominal set $\Sigma$ and every support bound $s\colon \fm\to \pow \A$, the monoid \spfm is projective w.r.t.~every \msr quotient \( e \colon M \epi N \) between orbit-finite monoids. Choose a nominal submonoid $M_e\seq M$ such that the restriction $e|_{M_e}$ is surjective and support-preserving.
  Let \( \hat{h} \) be continuous extension (i.e.\ the associated limit projection) of \( h \colon \fm \rightarrow_{s} N \), which is itself the extension of some equivariant map \( h_{0} \colon \Sigma \rightarrow N \). Since the strong nominal set $\Sigma$ is projective w.r.t.~the support-preserving quotient $e|_{M_e}$~\cite[Lem.~B.28]{mil-urb-19}, there exists an the equivariant map \( k_{0}\colon \Sigma \rightarrow M_{e} \subseteq M \) such that $h_0 = e|_{M_e}\cdot k_0$. Let
   \( k \colon \fm \rightarrow M_{e}  \) be its extension to an equivariant monoid morphism.
  Then we compute for all words \( w = x_{1} \cdots x_{n} \in \fm \):
  \begin{align*}
    \supp k(w) &= \supp(k_{0}(x_{1}) \cdots k_{0}(x_{n})) && \text{def.~of \( k \)} \\
               &= \supp(e(k_{0}(x_{1}) \cdots k_{0}(x_{n}))) &&  \text{\( e|_{M_{e}} \) support-preserving} \\
               &= \supp(e(k(w))) && \text{def.~of \( k \)} \\
               &= \supp h(w) && e \cdot k = h \\
    &\subseteq s(w).
  \end{align*}
  Hence \( k \) is \sbnd and thus extends to a continuous homomorphism \( \hat{k} \colon \spfm \rightarrow M \) satisfying \( e \cdot \hat{k} = \hat{h} \).

  For the converse direction, we first prove an auxiliary lemma.

\begin{rem}\label{rem:supp-refl-quot}
For every orbit-finite nominal set $M$, there exists an orbit-finite strong nominal set $\Sigma$ and a support-reflecting quotient $e\colon \Sigma\epito M$~\cite[Cor.~B.27]{mil-urb-19}. If $M$ is a nominal monoid, the extension of $e$ to a morphism $\overline{e}\colon \Sigma^*\epito M$ is also support-reflecting. 
\end{rem}

\begin{lemma}\label{lem:fact-msr}
  Let \( q\colon \Sigma^*\epito N \) be a support-reflecting equivariant monoid morphism with $N$ orbit-finite, and \( s = \supp \cdot\, q \).
  If \( q \) factors as
  \[  q = e \cdot k \colon \fm \epi_{s} M \epi N  \]
  with \sbnd \( k \), then \( e \) is \msr.
\end{lemma}
\begin{proof}
  Consider the equivariant subset $R_q\seq \Sigma^*$ given by 
 \[R_q = \{ w\in \fm\mid \supp q(w) = \supp w \}.\]
Since the map $q$ is support-reflecting, it restricts to a surjection $q|_{R_q}$. Let \( M_{e} = \langle k[R_{q}] \rangle \) be the submonoid of \( M \) generated by the image \( k[R_{q}] \subseteq M \), i.e.\ the equivariant set of all finite products $k(w_1)\cdots k(w_n)$ where $w_1,\ldots,w_n\in R_q$. 
Clearly the restriction \( e|_{M_{e}}\colon M_e\to N \) is surjective. To prove that it is support-preserving, let $k(w_1)\cdots k(w_n)\in M_e$.
  The inclusion \( \supp e(k(w_{{1}}) \cdots k(w_{n})) \subseteq \supp k(w_{{1}}) \cdots k(w_{n}) \) holds by equivariance of $e$, and for the reverse inclusion we compute
\begin{align*}
  \supp(k(w_{{1}}) \cdots k(w_{n})) &= \supp k(w_{{1}} \cdots w_{n}) \\
                                     &\subseteq s(w_{{1}} \cdots w_{n}) \\
                                     &= \supp q(w_{{1}}\cdots w_{n}) \\
                                     &= \supp (e(k(w_{{1}}) \cdots k(w_{n}))).\qedhere
\end{align*}
\end{proof}
  Now assume a quotient \( e \colon M \epi N \) such that every \spfm is projective w.r.t.\ \( e \);
  we prove that \( e \) is \msr.
Choose an orbit-finite strong nominal set $\Sigma$ and a support-reflecting morphism
   \( q \colon \fm \epi N \), see \autoref{rem:supp-refl-quot}, and  
  put \( s = \supp \cdot\, q \).
  This makes \( q \) an \sbnd quotient, and by projectivity 
   its extension \( \hat{q} \) thus factorizes through \( e \) as \( \hat{q} = e \cdot \hat{k} \colon \spfm \rightarrow M \epi N \) for some \sbnd homomorphism \( k \colon \fm \rightarrow_{s} M \);
  precomposition with \( \eta \) yields \( q = e \cdot k \).
  Now apply \autoref{lem:fact-msr}.

\subsection*{Proof of \autoref{thm:nominal-reiterman}}

\begin{construction}\label{con:eq-var}
 Let \( \Sigma \) be an orbit-finite strong nominal set and let \( s\colon \fm\to \pow_k\A \) be a support bound on \fm.
  For an \msr-pseudovariety \(\V\) of orbit-finite monoids  we denote  the set of all \( s \)-quotients with domain \fm and codomain in \( \V \) by
  \[\sepislice{\fm}{\V} = \{e \colon \fm \epi_{s} M \mid M \in \V \}.\]
Let $D_\V\colon \sepislice{\fm}{\V}\to \npro_k$ be the restriction of the canonical diagram $D\colon \sepislice{\fm}{\nomof,k}\to \npro_k$ (\autoref{def:pfm}). This diagram is codirected because $\V$ is closed under finite products and submonoids.
  The induced map into the limit is denoted by \( \varphi_{s} \colon \spfm \epi \lim (D_\V) \); we show below that it is surjective, hence a proequation. The limit projection associated to $e\in \sepislice{\fm}{\V}$ is denoted by $p_e\colon \lim (D_\V)\to M$. Thus $\hat e =p_e\cdot \varphi_s$ for all $e$.
  The set of all proequations \( \varphi_{s} \), with $\Sigma$ ranging over all orbit-finite strong nominal sets and \( s \) ranging over all support-bounds on $\fm$, is denoted by \( \T(\V) \).
\end{construction}

That $\varphi_s$ is surjective following from the next lemma.
\begin{lemma}\label{lem:med-mor-surj}
  If \( f_{i} \colon X \epito D_{i}, i \in I \) is a surjective cone for a codirected diagram in \( \npro_{k} \) then the mediating morphism \( f \colon X \rightarrow \lim D \) is surjective.
\end{lemma}
\begin{proof}
  Let \( (y_{i})_{i \in I} \in \lim D \) be a compatible family.
  The subspaces \( D'_{i} = f^{-1}((y_{i})_{i \in I}) \subseteq X \) are non-empty by surjectivity of the \( f_{i} \),
  nominally Hausdoff,
  and also nominally compact as closed subspaces of a nominally compact space.
  Observe that since the \( f_{i} \) form a cone,
  the subsets \( D'_{i} \subseteq X \) form a subdiagram \( D' \) of the constant diagram $C_X$ with value \( X \).
  This diagram \( D' \) has a non-empty limit \( \lim D' \) by \autoref{prop:cod-lim-ne};
  choose any element \( x \in \lim D' \subseteq \lim C_X = X \). Then $x$ is
  mapped by $f$ to \( (y_{i})_{i \in I} \) by construction.
\end{proof}

\autoref{thm:nominal-reiterman} is immediate from

\begin{lemma}\label{lem:presentation}
  For every set \T of proequations the class \( \V(\T) \) is an \msr-pseudovariety.
  Conversely, every \msr-pseudovariety \V is presented by the family \( \T(\V) \) of proequations, that is,
  \( \V = \V(\T(\V)) \).
\end{lemma}

\begin{proof}
  We first prove that the class of orbit-finite monoids presented by a set of proequations is an MSR-pseudovariety.
  Since the intersection of MSR-pseudovarieties is again an MSR-pseudovariety,
  it suffices to show that \( \V=\V(\varphi) \) is an MSR-pseudovariety for every proequation \(\varphi \colon \spfm \epi T\).
  The proof is a routine verification.
  \begin{description}
    \item[Closure under finite products.]
          Let \( (M_{i})_{i \in I} \) be a finite family of nominal monoids in $\V$, and let
          \[ {h}  \colon \fm \rightarrow \prod_{i \in I} M_{i} \]
          be an \sbnd equivariant monoid morphism. Let $\pr_i\colon \prod M_i\to M_i$ denote the projection, and put $h_i=\pr_i \cdot h\colon \fm\to M_i$. Note that $h_i$ is $s$-bounded and that \( \hat{h}_{i} = \pr_{i} \cdot \hat h \). Since $M_i\in \V$,
          the map $\hat{h}_{i}$  factors through \(\varphi\) as \( \hat{h}_{i} = k_{i} \cdot \varphi \), for some $k_i$.
          We get \( \hat{h} = \langle \hat h_{i} \rangle_{i \in I} = \langle k_{i} \rangle_{i\in I} \cdot \varphi \). The map $\langle k_i\rangle $ is continuous because $\prod M_i$ carries the product topology. This proves $\prod_i M_i\in \V$.
    \item[Closure under submonoids.]
          Let \( M \in \V \), and let \( m \colon N \mono M \) be a nominal submonoid of \( M \). For every $s$-bounded morphism $h\colon \fm\to N$, the composite $m\cdot h$ is $s$-bounded, hence \( \widehat{m\cdot h}=m \cdot \hat{h} \) factorizes as \( k \cdot \varphi \colon \spfm \epi T \rightarrow M \) because $M\in \V$.
          The corestriction \( k|^{N} \) of \( k \) to \( N \) then yields a factorization of \( \hat{h} \) as \( \hat{h} = k|^{N} \cdot \varphi \).
          Note that \( k|^{N} \) is continuous since \( N \) carries the subspace topology. This proves $N\in \V$.
    \item[Closure under \msr quotients.]
          Let \( q \colon M \epi N \) be a \msr quotient with \( M \in \V \), and let $h\colon \fm\to N$ be an $s$-bounded morphism. 
          By \autoref{prop:s-quot}, we have \( \hat h = q \cdot \hat{g} \) for some $s$-bounded morphism \( {g} \colon \fm \rightarrow M \).
          Since \( M\in \V \), the map
          \( \hat{g} \) factors as \( \hat{g} = p \cdot \varphi \).
          This yields a factorization of \( \hat{h} \) as \( \hat{h} = q \cdot \hat{g}  = q \cdot p \cdot \varphi \).
          Therefore \( N \in \V \).
  \end{description}

This proves that every class $\V(\T)$ is an \msr-pseudovariety. Conversely, 
 we now show that every MSR-pseudovariety \( \mathcal{V} \) is presented by the family \( \mathcal{T}(\mathcal{V}) \) of proequations, that is, $\V=\V(\T(\V))$. 
  For the left-to-right inclusion we have to show that every monoid $M\in \V$ satisfies all proequations \(\varphi_{s}\) in \( \T(\V) \).
  Given an $s$-bounded morphism \( {h} \colon \fm \rightarrow M \), take the image factorization
  \[ h = m \cdot e \colon \Sigma^{*} \epi N \mono M. \]
Note that the coimage $e$ is also $s$-bounded, and that $N\in \V$ because $M\in \V$ and $\V$ is closed under submonoids.
  Thus \( e \in (\sepislice{\fm}{\V}) \), and so \( M \) satisfies \( \varphi_{s} \) because \( \hat{h} = m \cdot \hat{e} = m \cdot p_{e} \cdot \varphi_{s} \).

  To prove the right-to-left inclusion, let \( M \) be an orbit-finite monoid satisfying all proequations in \( \T(\V) \). We show that $M$ lies in \V.
  Choose an orbit-finite strong nominal set $\Sigma$ and a support-reflecting morphism \( q \colon \Sigma^{*} \epi M \), see \autoref{rem:supp-refl-quot}. Put \( s = \supp \cdot \, q\colon \Sigma^*\to \pow_k \A \), where the number $k$ is an upper bound to the support size of elements in $\Sigma$ and $M$.
  The monoid \( M \) satisfies \(\varphi_{s} \in \T(\V)\),
  so the map \( \hat{q} \) factorizes as
  \[ \hat{q} = (\spfm \overset{\varphi_{s}} \epi \lim (\sepislice{\fm}{\V})  \overset{h}\epi M). \]
  Since the orbit-finite set $M$ is finitely copresentable in $\npro_k$
  the morphism \( h \) itself factors through the limiting cone as \( h = h' \cdot p_{e} \colon \lim (\sepislice{\fm}{\V}) \epi M' \epi M \) for some $e\in \sepislice{\fm}{\V}$; thus \( M' \in \V  \).
  By  \autoref{lem:fact-msr} the quotient \( h' \) is \msr,
  and since \( \V \) is closed under \msr quotients,
  we conclude \( M \in \V \). 
\end{proof}

\subsection*{Proof of \autoref{thm:exp-nominal-reiterman}}
The theorem is immediate from following two lemmas.
\begin{lemma}\label{lem:exp-pro-eq}
  An orbit-finite nominal monoid satisfies a proequation \( \varphi \colon \spfm \epi T \) iff it satisfies all explicit proequations
  \[ \{ x = y \mid (x, y) \in \ker \varphi \}, \]
where $\ker \varphi = \{ (x,y)\in \spfm \mid \varphi(x)=\varphi(y) \}$.
\end{lemma}
\begin{proof}
  For the ``only if''-direction, let the proequation \( \varphi \) be satisfied by an orbit-finite monoid \( M \).
  Every continuous extension \( \hat{h} \) factors through \(\varphi\) as \( \hat{h} = k \cdot \varphi \).
  Then \( \ker \varphi \subseteq \ker (k \cdot \varphi) = \ker \hat{h} \),  so \( M \) satisfies all equations \( (x = y) \in \ker \varphi \).

  The ``if''-direction follows from the homomorphism theorem:
  If \( M \) satisfies all explicit equations,
  then by the homomorphism theorem for monoids there exists for every surjective \( h \colon \fm \epi_{s} M  \) a surjective monoid homomorphism \( p_{k} \colon T \epi M \) with \( \hat{h} = p_{k} \cdot \varphi \).
  The morphism \( p_{k} \) is equivariant since \( \varphi \) is surjective and \( \hat h \) is equivariant,
  and it is continuous since \( E \) carries the quotient topology.
\end{proof}

\begin{lemma}\label{lem:pro-ex-var}
  The class of orbit-finite nominal monoids satisfying a set of explicit proequations  forms an MSR-pseudovariety.
\end{lemma}
\begin{proof}
  The proof is analogous to its counterpart in \autoref{lem:presentation}.
  The intersection of MSR-pseudovarieties is again an MSR-pseudovariety, so we only show the statement for a single equation \( x = y \) over \spfm.
\begin{description}
    \item[Closure under finite products.]
          Let \( (M_{i})_{i \in I} \) be a finite family of nominal monoids in $\V$, and let
          \[ {h}  \colon \fm \rightarrow \prod_{i \in I} M_{i} \]
          be an \sbnd equivariant monoid morphism. Let $\pr_i\colon \prod M_i\to M_i$ denote the projection, and put $h_i=\pr_i \cdot h\colon \fm\to M_i$. Note that $h_i$ is $s$-bounded and that \( \hat{h}_{i} = \pr_{i} \cdot \hat h \). Since $M_i\in \V$,
        we get $\hat{h}_{i}(x) = \hat{h}_{i}(y)$ for all \( i \in I \),
        and hence \( \hat{h}(x) = \hat{h}(y) \).
           This proves $\prod_i M_i\in \V$.
    \item[Closure under submonoids.]
        Let \( M \in \V \), and let \( m \colon N \mono M \) be a nominal submonoid of \( M \). For every $s$-bounded morphism $h\colon \fm\to N$, the composite $m\cdot h$ is $s$-bounded.
        Since \( M \in \V \) and the \( \hat{h} \) form a cone by \autoref{rem:canon}.1 we get
        \[ (m \cdot \hat{h})(x) =  (\widehat{m \cdot h})(x) = (\widehat{m \cdot h})(y) = (m \cdot \hat{h})(y). \]
        Since \( m \) is mono and therefore injective, this yields \( \hat h (x) = \hat h (y) \),
        proving \( N \in \V \).
    \item[Closure under \msr quotients.]
          Let \( q \colon M \epi N \) be a \msr quotient with \( M \in \V \), and let $h\colon \fm\to N$ be an $s$-bounded morphism.
          By \autoref{prop:s-quot}, we have \( \hat h = q \cdot \hat{g} \) for some $s$-bounded morphism \( {g} \colon \fm \rightarrow M \).
        Since \( M\in \V \) it satisfies \( \hat{g}(x) = \hat{g}(y) \),
        postcomposition with \( q \) yields \( \hat h (x) = (q \cdot \hat g)(x) = (q \cdot \hat g) (y) = \hat h (y)  \).
          This shows \( N \in \V \).\qedhere
  \end{description}

\end{proof}

\subsection*{Details for \autoref{ex:compare}}

  We prove that \V is an MSR-pseudovariety.
  Clearly $\V$ is closed under submonoids.  To prove
  closure under finite products, suppose that \( M, N \in \V \), and
  let \( (m, n), (m', n') \in M \times N\) such that
  \( \supp ((m, n)(m', n')) = \supp (mm', nn') = \emptyset \).  Then
  \( \supp (mm') = \emptyset \) and \( \supp (nn') = \emptyset \).
  Since \( M, N \) both lie in \V this implies
  \( \supp (m, m') = \supp (n, n') = \emptyset \), whence
  \( \supp ((m, n), (m', n')) = \emptyset \). This proves that
  $M\times N\in \V$.  To prove closure under \msr quotients, let
  \( M \in \V \) and let \( e \colon M \epi N \) be \msr; thus there
  exists a nominal submonoid $M'\seq M$ such that
  \( e|_{M'} \colon M' \to N \) is surjective and support-preserving.
  Given \( n, n' \in N \) with \( \supp (nn') = \emptyset \) there
  exist respective preimages \( m, m' \in M' \). Then
  \[ \supp (mm')=\supp(e(mm'))= \supp(e(m)e(m')) = \supp(nn') =
    \emptyset,\] where the first step uses that $mm'\in M'$ because
  $M'$ is submonoid of $M$, and that $e|_{M'}$ is support-preserving.
  Since $M\in \V$, we get \( \supp (m, m') = \emptyset \), whence
  \[
    \supp (n, n') = \supp (e(m), e(m')) = \supp e(m) \cup \supp e(m')
    = \supp m \cup \supp m' = \emptyset.
  \]
  Thus \( N \) satisfies \eqref{eq:ex-msr}, so $N\in \V$.
  This concludes the proof that $\V$ is an
  MSR-pseudovariety.

\end{document}